\documentclass{article}

\usepackage[utf8]{inputenc}
\usepackage{microtype}
\usepackage{calligra}
\usepackage{amsfonts}
\usepackage{bbm}
\usepackage{yhmath}
\usepackage{dsfont}
\usepackage{amsmath,amsthm, amssymb}
\usepackage{mathrsfs}
\usepackage{mathtools}
\usepackage{titlesec}
  \titleformat{\paragraph}
    {\normalfont\normalsize\bfseries}{\theparagraph}{1em}{}
  \titlespacing*{\paragraph}
    {0pt}{3.25ex plus 1ex minus .2ex}{1.5ex plus .2ex}
\usepackage{float}
\usepackage{enumerate}
\usepackage{paralist}
\usepackage{fullpage}
\usepackage[ruled,vlined]{algorithm2e}
\usepackage{xcolor}
\usepackage{graphicx}
\usepackage{svg}
\usepackage{tikz}
  \usetikzlibrary{positioning}
  \usetikzlibrary{calc}
  \usetikzlibrary{backgrounds}
  \usetikzlibrary{3d}
  \usetikzlibrary{fit}
  \usetikzlibrary{fadings}
  \usetikzlibrary{matrix}
  \usetikzlibrary{spy}
  \usetikzlibrary{decorations.text}
  \usetikzlibrary{decorations.pathreplacing}
  \usetikzlibrary{shapes.arrows}
  \usetikzlibrary{plotmarks}
  \tikzset{%
    circle dotted/.style={dash pattern=on 0.05mm off 2mm,line cap=round}
    pics/fake box/.style args={#1 with dimensions #2 and #3 and #4}
    {code={\draw[gray,ultra thin,fill=#1]  
      (0,0,0) coordinate(-front-bottom-left) to
       ++ (0,#3,0) coordinate(-front-top-right) --
       ++ (#2,0,0) coordinate(-front-top-right) --
       ++ (0,-#3,0) coordinate(-front-bottom-right) -- cycle;
      \draw[gray,ultra thin,fill=#1] (0,#3,0)  --
       ++ (0,0,#4) coordinate(-back-top-left) --
       ++ (#2,0,0) coordinate(-back-top-right) --
       ++ (0,0,-#4)  -- cycle;
      \draw[gray,ultra thin,fill=#1!80!black] (#2,0,0) --
       ++ (0,0,#4) coordinate(-back-bottom-right) --
       ++ (0,#3,0) -- ++ (0,0,-#4) -- cycle;
      \path[gray,decorate,decoration={text effects along path,text={CONV}}] 
        (#2/2,{2+(#3-2)/2},0) -- (#2/2,0,0);
      }
    }
  } 
\usepackage{tikz-3dplot}
\usepackage{xspace}
\usepackage{specialmacros}
  \newcommand{\esssuppp}{\mathrm{ess\, supp}}
  \newcommand{\operator}[1]{\mathcal{#1}}
  \newcommand{\vecoperator}[1]{\mathsf{#1}} 
  \DeclareMathOperator{\vecForwardOp}{\vecoperator{A}}
  \DeclareMathOperator{\Radon}{\operator{A}}
  \DeclareMathOperator{\vecRadon}{\vecoperator{A}}
  \DeclareMathOperator{\Loss}{\operator{L}} 
  \DeclareMathOperator{\loss}{\ell}
  \DeclareMathOperator{\DataDiscrep}{\operator{L}}
  \DeclareMathOperator{\Regulariser}{\operator{S}}
  \DeclareMathOperator{\RecOp}{\operator{R}}
  \DeclareMathOperator{\vecRecOp}{\vecoperator{R}}
  
  \DeclareMathOperator{\vecConvAffine}{\vecoperator{W}}
  \newcommand{\datadomain}{\Xi}
  
  \newcommand{\random}[1]{\mathbbm{#1}}
  \newcommand{\signal}{f}
  \newcommand{\signaltrue}{\signal_{\mathrm{true}}}
  \newcommand{\stsignal}{\random{\signal}}
  \newcommand{\vecsignal}{\boldsymbol{\signal}}
  \newcommand{\data}{g}
  \newcommand{\stdata}{\random{\data}}
  \newcommand{\vecdata}{\boldsymbol{\data}}
  \newcommand{\datanoise}{e}
  \newcommand{\stdatanoise}{\random{\datanoise}}
  \newcommand{\dualvar}{h}
  \newcommand{\vecdualvar}{\boldsymbol{h}}
  \newcommand{\vecWF}{\vecoperator{DWF}}
  \newcommand{\visible}{\mathrm{vis}}
  \newcommand{\WFvis}{\WF^{\visible}}
  \DeclareMathOperator{\ReLU}{\vecoperator{ReLU}}
  \DeclareMathOperator{\ReLUOp}{ReLU}
  \DeclareMathOperator{\ResNet}{ResNet}
  \DeclareMathOperator{\vecResNet}{\vecoperator{ResNet}}
  \DeclareMathOperator{\Heav}{\vecoperator{H}}
  \DeclareMathOperator{\HeavOp}{\operator{H}}

\usepackage[textsize=small]{todonotes}
\usepackage{comment}

\usepackage[pdfencoding=auto,psdextra,bookmarksopen]{hyperref}
\title{Deep Microlocal Reconstruction for Limited-Angle Tomography}
\author{%
  H\'ector Andrade-Loarca$^1$
  \and Gitta Kutyniok$^{1,2}$
  \and Ozan \"Oktem$^{3,4}$
  \and Philipp Petersen$^5$
}
\date{}

\begin{document}
\maketitle

\begin{abstract}
We present a deep learning-based algorithm to jointly solve a reconstruction problem and a wavefront set extraction problem in tomographic imaging. The algorithm is based on a recently developed digital wavefront set extractor as well as the well-known microlocal canonical relation for the Radon transform. We use the wavefront set information about x-ray data to improve the reconstruction by requiring that the underlying neural networks simultaneously extract the correct ground truth wavefront set and ground truth image. As a necessary theoretical step, we identify the digital microlocal canonical relations for deep convolutional residual neural networks. We find strong numerical evidence for the effectiveness of this approach.
\end{abstract}

\par\medskip\noindent
\textbf{Keywords:} Inverse problems, deep learning, tomography, microlocal analysis, wavefront set.
\par\smallskip\noindent
\textbf{Mathematics Subject Classification:} 35A18, 65T60, 68T10.


\footnotetext[1]{Department of Mathematics, LMU Munich, 80333 Munich, Germany, 
\texttt{$\{$kutyniok,andrade$\}$@math.lmu.de}}
\footnotetext[2]{Department of Physics and Technology, University of Troms\o, 9019 Troms\o, Norway}
\footnotetext[3]{Department of Mathematics, KTH - Royal Institute of Technology, SE-100 44 Stockholm, Sweden, \texttt{ozan@kth.se}}
\footnotetext[4]{Department of Information Technology, Division of Scientific Computing, Uppsala University, SE-751 05 Uppsala, Sweden, \texttt{ozan.oktem@it.uu.se}}
\footnotetext[5]{Faculty of Mathematics and Research Network Data Science, University of Vienna, 1090 Vienna, Austria \texttt{philipp.petersen@univie.ac.at}}

\section{Introduction}\label{sec:Intro}
Tomographic imaging aims at uncovering the interior 2D/3D structure of an object from a \emph{sinogram}, which is data obtained by repeatedly exposing the object to a particle or wave from different directions.  
This is a key example of an \emph{inverse problem} where one computationally attempts recover an unknown signal from data given as indirect observations.
A reconstruction method refers to an algorithm performing the recovery.

Inverse problems, like those that arise in tomographic imaging, are often \emph{ill-posed} which means that there can be multiple solutions consistent with the data or solution procedures that maximise consistency against measured data are sensitive to variations in data.
Such high sensitivity is referred to as instability, and it appears when the \emph{forward operator}, which models how a signal gives rise to corresponding noise-free data, is not continuously invertible.
Stability properties are further degraded for sparse-view data, which is when data is under-sampled, and for limited-angle data, which refers to unevenly sampled data.
\emph{Regularisation} refers to mathematical theory and methods for stabilising the solution procedure of an ill-posed inverse problem.
Many regularisers enforce stability by requiring consistency against a \emph{prior model}.
This prior should ideally encode known properties of the unknown signal one seeks to recover, and choosing an appropriate prior is an essential part of regularisation. 

Most reconstruction methods in tomography assume that measurements are collected from views that are evenly distributed around the object.
\emph{Limited-angle tomography} refers to a case when this is not fulfilled. 
Such problems arise naturally in many applications, like digital breast tomosynthesis \cite{Mall:2017aa,Baker:2011aa}, dental tomography \cite{Hyvonen:2010aa,Kalke:2014aa}, electron tomography \cite{quinto2008local,Oktem:2015aa}, transmission x-ray microscopy \cite{Huang:2020aa}, nondestructive testing \cite{Quinto:1998aa,Riis:2018aa}, geophysical prospecting \cite{hoop2009genRadon}, etc. 
This missing data significantly amplifies the instability in the corresponding reconstruction problem \cite{Davison:1983aa,Louis:1986aa}.
Hence, traditional reconstruction methods, like filtered back-projection (FBP) \cite{Bracewell:1967aa,Shepp:1974aa,Gullberg:1979aa}, that implicitly assume missing data is zero do not perform well in such situations.
Overall, it has been very challenging to develop regularisation methods that handle this instability as these methods somehow need to fill in the missing data without imposing too strong assumptions on the signal, see the brief survey in Subsection~\ref{subsec:SOTA}.
As a consequence, it has been challenging to develop practically useful reconstruction methods for limited-angle tomography that provide sufficient improvement over traditional reconstruction methods. 
In this context, we consider a reconstruction method practically useful if it is computationally feasible and does not require a user to set multiple hyperparameters.

\subsection{Main contributions}\label{subsec:introresults}
In this paper, we develop theory and algorithms for reconstruction in severely ill-posed inverse problems that arise in tomographic imaging with limited data.
In particular, we develop a data-driven reconstruction method for limited-angle tomography that is microlocally consistent, which means `filling in' missing data in a way that is consistent with how singularities in data are related to those in the signal.

Our approach relies on deep neural networks (DNNs) that integrate a handcrafted forward operator and a theoretical characterisation of how singularities in data are related to those in the signal. In this context, we model singularities through the concept of the \emph{wavefront set}, that will be introduced in detail in Section \ref{sec:PlanarTomo} below. For the Radon transform, which is the forward operator underlying the tomography problem, it is known exactly how an application of it affects the wavefront set of underlying data. Indeed, the wavefront set of an image and its Radon transform are linked through so-called \emph{canonical relations}. To use this a priori information on the wavefront set in the reconstruction problem, we set up the following method: Our algorithm consists of two pipelines that act in parallel. First, a \emph{reconstruction line} that maps the data to the (unknown) signal. This is a specific deep neural network using the \emph{Learned Primal-Dual architecture}, \cite{adler2018lpd}. Second, a \emph{micro-local line} that identifies the wavefront set of the signal by using the following three steps: \begin{inparaenum}[(a)] \item A neural network that extracts the wavefront set of the data. This neural network has been established earlier in \cite{andrade2019wfset} under the name DeNSE and is based on the interaction of the shearlet transform with the wavefront set of a function. \item An analytical computation relating the wavefront set in the data domain to the image domain in a way that mirrors the action of the Learned Primal-Dual architecture of the reconstruction line. The two lines are coupled via this step. \item A neural network with the U-Net architecture \cite{ronnenberger2015unet} that inpaints the wavefront set in the image domain that was inferred from the incomplete data. \end{inparaenum} 

The neural networks in the microlocal line and the reconstruction line are now trained in parallel on an artificial training set consisting of 2D phantoms (signals) with corresponding noisy and incomplete sinogram data. The phantoms are made up of random shapes that are demarcated by piecewise smooth curves given by splines of degree at most four. Several examples from the data set are shown in Figure \ref{fig:dataSet}.

The combined procedure, which we coin the \emph{joint reconstruction algorithm}, now attempts to satisfy two objectives on that data set: \begin{inparaenum}[(a)] \item The reconstruction returned from the reconstruction line should agree as closely as possible with the ground truth. \item The wavefront set returned by the microlocal line should resemble the ground truth wavefront set of the data accurately.
\end{inparaenum}

We refer to Section~\ref{sec:DigitalMicrolocal} for the complete description of the joint reconstruction algorithm and in particular Figure~\ref{fig:task-adapt} for an illustration of the underlying DNN architecture. 

Note that, due to the simplicity of the training data set, we can analytically compute the true wavefront set for those signals, which in turn is needed for the aforementioned joint training of the DNNs for wavefront set inpainting and reconstruction. In contrast, our test data consists of images of brains and associated noisy incomplete sinograms. Therefore, the \emph{training data is substantially different from the test data}. As such, the empirical numerical study also shows the transfer learning properties of our approach.

We wish to emphasise that this approach applies in principle to any inverse problem where the forward operator is a Fourier integral operator, as is the case for most inverse problems arising in imaging applications.
However, for simplicity, we chose to work out the theoretical results with associated algorithms and numerical examples only for the specific case of planar limited-angle tomography.

A central part of the proposed algorithm consists of the theoretical analysis of how the Learned Primal-Dual architecture affects the extracted wavefront set. At this point, the reader may justifiably wonder if it would not instead be possible to estimate the wavefront set of the output of this architecture directly via the wavefront set extractor DeNSE. This is theoretically possible but turned out to be practically infeasible. While the wavefront sets of the data set can be precomputed, this cannot be done for the outputs of the Learned Primal-Dual network as the Learned Primal-Dual network varies during the training. This means that during training a full application of DeNSE would have to be performed in every training step. On a modern machine, an application of DeNSE to a single image in the data set takes approximately $20$ seconds, which shows that this approach would slow down the training process dramatically.

\begin{figure}[H]
\centering
\begin{minipage}[t]{0.25\textwidth}
\centering
  \vspace{0pt} 
  \begin{tikzpicture}[spy using outlines={
      rectangle, 
      red, 
      magnification=3,
      size=0.3\linewidth, 
      connect spies}]
    \node{\includegraphics[width = \linewidth]{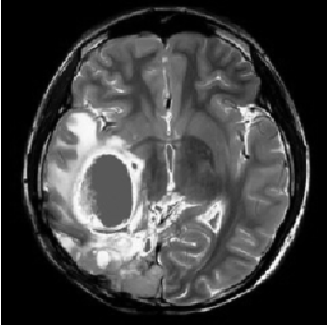}};
    \spy on (1.05,1.45) in node [left] at (-0.195\linewidth,0.345\linewidth);
    \spy on (0.6,-1.2) in node [left] at (-0.195\linewidth,-0.345\linewidth);
  \end{tikzpicture}
  \\
  {\small Phantom (ground truth signal)}
\end{minipage}
\hspace{3pt}
\begin{minipage}[t]{0.25\textwidth}
\centering
  \vspace{0pt} 
  \begin{tikzpicture}[spy using outlines={
      rectangle, 
      red, 
      magnification=3,
      size=0.3\linewidth, 
      connect spies}]
    \node{\includegraphics[width = \linewidth]{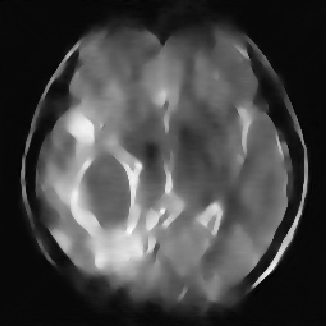}};
    \spy on (1.05,1.45) in node [left] at (-0.195\linewidth,0.345\linewidth);
    \spy on (0.6,-1.2) in node [left] at (-0.195\linewidth,-0.345\linewidth);
  \end{tikzpicture}  
  \\
  {\small Learned Primal-Dual network\\ (PSNR 24.90)}
\end{minipage}
\hspace{3pt}
\begin{minipage}[t]{0.25\textwidth}
\centering
  \vspace{0pt} 
  \begin{tikzpicture}[spy using outlines={
      rectangle, 
      red, 
      magnification=3,
      size=0.3\linewidth, 
      connect spies}]
    \node{\includegraphics[width = \linewidth]{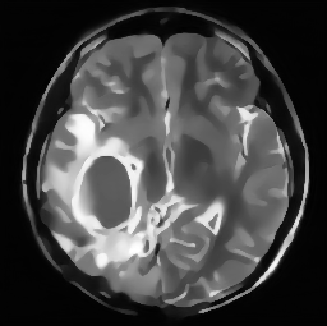}};
    \spy on (1.05,1.45) in node [left] at (-0.195\linewidth,0.345\linewidth);
    \spy on (0.6,-1.2) in node [left] at (-0.195\linewidth,-0.345\linewidth);
  \end{tikzpicture}  
  \\
  {\small Joint approach\\ (PSNR 30.20)}
\end{minipage}
\\[1em]
\begin{minipage}[t]{0.25\textwidth}
\centering
  \vspace{0pt} 
  \includegraphics[width = \linewidth]{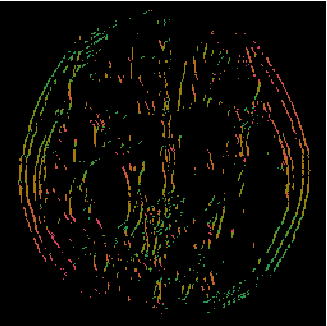}
  \\
  {\small Visible wavefront set of the ground truth extracted by DeNSE}
\end{minipage}
\hspace{3pt}
\begin{minipage}[t]{0.25\textwidth}
\centering
  \vspace{0pt} 
  \includegraphics[width = \linewidth]{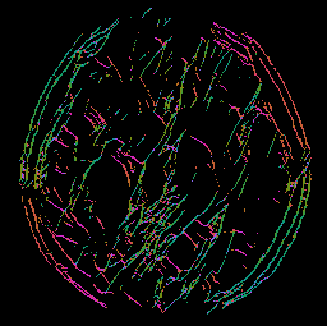}
  \\
  {\small Wavefront set inpainted as pre-processing using a trained U-Net}
\end{minipage}
\hspace{3pt}
\begin{minipage}[t]{0.25\textwidth}
\centering
  \vspace{0pt} 
  \includegraphics[width = \linewidth]{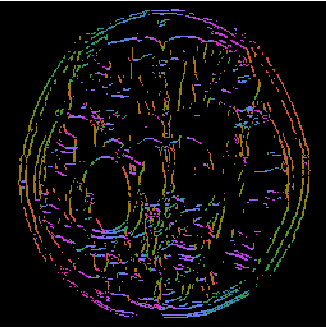}
  \\
  {\small Reconstructed wavefront set from joint approach}
\end{minipage}
\caption{Limited-angle parallel beam tomography with a realistic brain phantom.}
\label{fig:results}
\end{figure}
As a small appetiser, we conclude this overview of the main contributions by illustrating the superior performance of our approach.
Figure~\ref{fig:results} compares reconstruction performed by a DNN, here the learned primal-dual network, against our joint approach that combines this Learned Primal-Dual network with a DNN for inpainting the wavefront set. In this figure  the tomographic data is given in the form of highly noisy samples of the Radon transform with $40^\circ$ missing angular wedge. Reconstructions shown the in top row are from two supervised deep learning approaches, namely, Learned Primal-Dual (middle) and our joint approach (right). The latter is essentially the Learned Primal-Dual network combined with a learned wavefront set inpainting that  complements missing data in a microlocally consistent manner.

Both the Learned Primal-Dual network and the joint DNN for reconstruction and wavefront set inpainting were trained against the same training data with the same total number of training steps. The joint approach clearly shows the benefit of complementing missing data as it is able to recover essential features that are lost using the Learned Primal-Dual network. Moreover, the joint approach also seems to recover reasonably well singularities (edges) in parts of the image that are not in the visible wavefront set. This is remarkable as these singularities leave, according to microlocal theory, no measurable footprint in data.
Furthermore, applying wavefront set inpainting to data separately as a pre-processing step prior to reconstruction yields a significantly worse estimator for the wavefront set. 

This shows the benefit of performing these steps simultaneously as in the joint approach.
Figures~\ref{fig:F12} and \ref{fig:F13} present a more extensive comparison that also includes sparse-view tomography and other model based and data-driven methods.

\subsection{Outline of the paper}\label{subsec:Outline} 

Sections~\ref{sec:Intro}, \ref{sec:PlanarTomo} and \ref{subsec:InvProb} along with Appendix~\ref{app:DistrTheory} primarly serve as background material, whereas the main scientific contribution is contained in Sections~\ref{sec:ResNetMicroLocal} and \ref{sec:DigitalMicrolocal}. 
Finally, Section~\ref{sec:numresults} provides numerical examples that showcase the performance of the suggested approach and compares it to other methods for image reconstruction in limited-angle tomography. 
We present a more detailed outline below:

Section~\ref{sec:Intro} provides background motivation from applications along with an overview of the main scientific contributions (Subsection~\ref{subsec:introresults}) and a survey of current state-of-the-art for limited-angle tomography (Subsection~\ref{subsec:SOTA}).
This is followed by Section~\ref{subsec:InvProb} that mathematically formalises an inverse problem, both as an operator equation as in \eqref{eq:InvProbFreq} as well as a statistical inference problem as in \eqref{eq:InvProbBayes}. 
This section also introduces the notions of ill-posedness and regularisation (Subsection~\ref{subsec:IllPosedReg}). Emphasis is next on the Learned Primal-Dual network introduced in Subsection~\ref{subsec:LPDArchitecture}. 
This is a DNN with an architecture that incorporates the forward operator which is to be inverted.

\begin{figure}
    \centering
    \includegraphics[width = 0.25\textwidth]{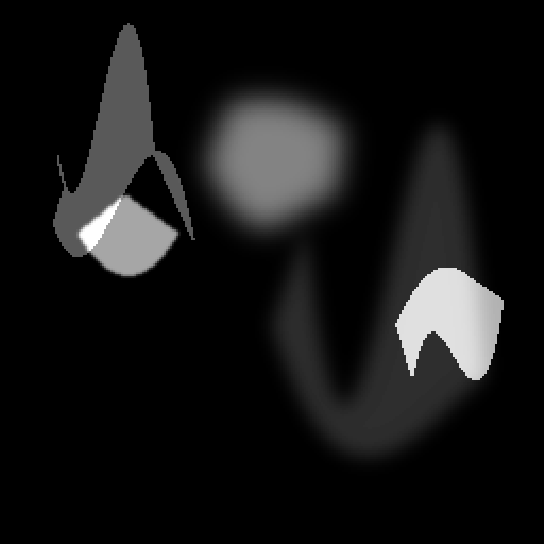} \includegraphics[width = 0.25\textwidth]{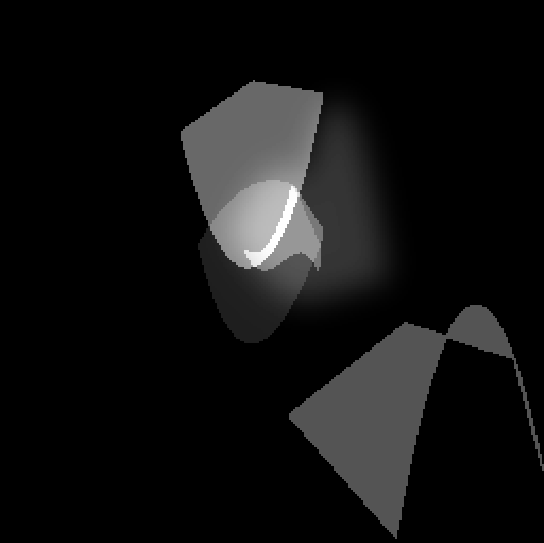}
    \includegraphics[width = 0.25\textwidth]{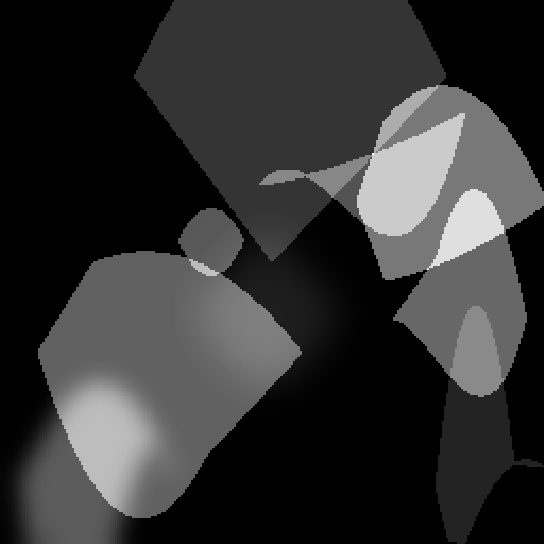}
    
    \smallskip
    
    \includegraphics[width = 0.25\textwidth]{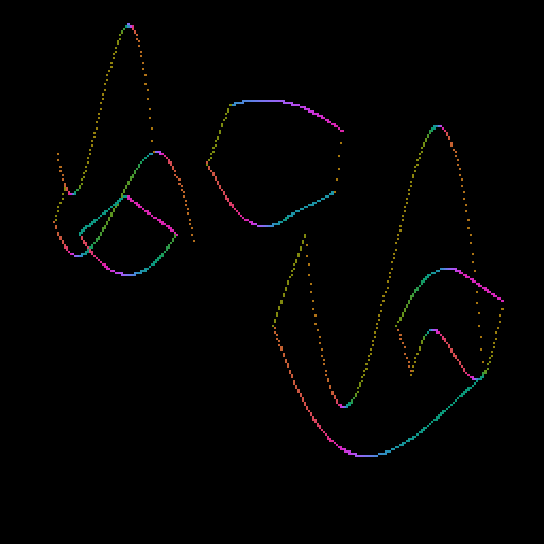} \includegraphics[width = 0.25\textwidth]{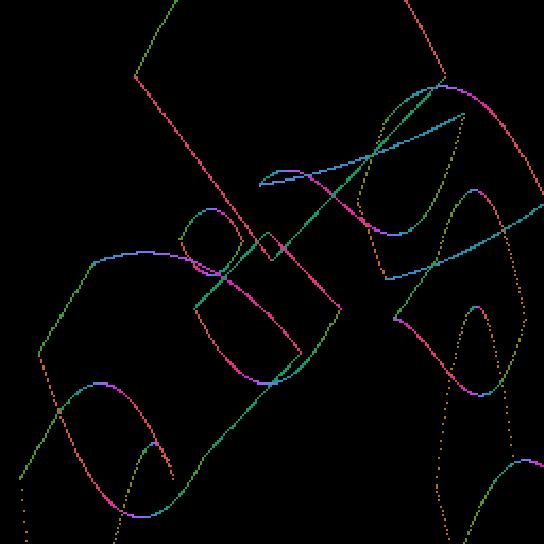}
    \includegraphics[width = 0.25\textwidth]{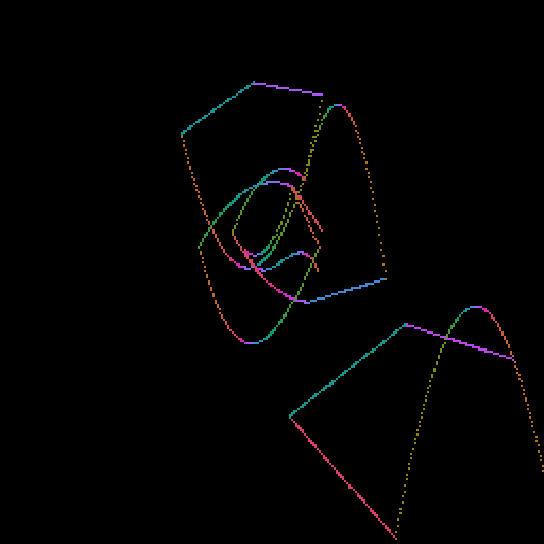}
    
    \smallskip
    
    \includegraphics[width = 0.25\textwidth]{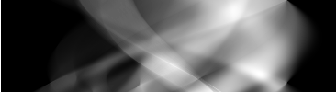}  \includegraphics[width = 0.25\textwidth]{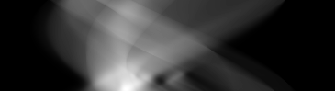} 
    \includegraphics[width = 0.25\textwidth]{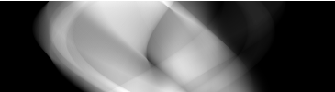}
    
    \smallskip
    
    \includegraphics[width = 0.25\textwidth]{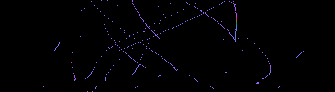}  \includegraphics[width = 0.25\textwidth]{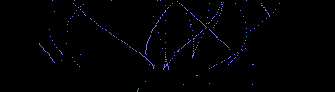} 
    \includegraphics[width = 0.25\textwidth]{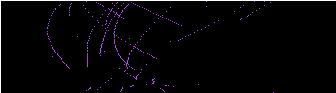}
    \caption{Some examples of the piecewise smooth data set comprised of functions which are piecewise polynomial and where the interfaces are given by splines of degree at most four. In the second row, we depict the associated analytically found wavefront sets. The third and fourth rows show, respectively, the corresponding sinograms and their wavefront set, both in the limited-angle setting with $40^{\circ}$ wedge. }
    \label{fig:dataSet}
\end{figure}

Background on the mathematics of tomographic imaging is provided in Section~\ref{sec:PlanarTomo}. There, we define the Radon transform arising in planar tomography (Definition~\ref{def:radontrans}) and recall some of its properties (Subsection~\ref{subsec:RadonProp}). 
Focus here is on extending the Radon transform and its corresponding back-projection (Definition~\ref{def:BackProj}) to tempered distributions. 
This uses basic notions from distribution theory provided in Appendix~\ref{app:DistrTheory} (some of the material in this appendix is also used later in Section~\ref{sec:ResNetMicroLocal}).
Subsection~\ref{subsec:RadonProp} also introduces the restricted Radon transform arising in limited-angle tomography along with its corresponding restricted back-projection defined in \eqref{eq:LimitedAngBP}.
Finally, Subsection~\ref{subsec:wavefrontset} formally defines the wavefront set (Definition~\ref{def:S1D4}) and then states the canonical relation for the Radon transform given in \eqref{eq:RadonMCR}. 
In \eqref{eq:RadonWFvis} we also provide a characterisation from microlocal analysis of the visible (and invisible) parts of the wavefront set for the Radon transform.

The main scientific contributions are contained in Sections~\ref{sec:ResNetMicroLocal} and \ref{sec:DigitalMicrolocal}.
More precisely, Section~\ref{sec:ResNetMicroLocal} provides a theoretical analysis of the propagation of the wavefront set through the distinct layers in a continuum version of the Learned Primal-Dual architecture with the Radon transform as forward operator. 
Section~\ref{sec:DigitalMicrolocal} then presents, in the digital setting, how the digital wavefront set is propagated through the layers of the Learned Primal-Dual architecture. 
This propagated wavefront set, along with the characterisation of the visible wavefront set, is used for setting up the DNN for wavefront set inpainting, which recovers the invisible part of the wavefront set of a signal from the visible part. 
This DNN for wavefront set inpainting is then combined with the Learned Primal-Dual network for reconstruction, and both DNNs are trained jointly following the task-adapted reconstruction paradigm outlined in \cite{adler2018taskadapt}, see Section~\ref{sec:jointreconinpaint} for further details. 

Finally, Section~\ref{sec:numresults} provides numerical evidence and benchmarks related to image reconstruction in limited-angle tomography. 

\subsection{Related work}\label{subsec:SOTA}
The severe ill-posedness associated with limited data in inverse problems has attracted much interest within the research community.
This survey will mainly focus on the development of theory and methods for inversion of the Radon transform from limited-angle data, which is an archetype of a severely ill-posed inverse problem.


\subsubsection{Analytic approaches}
Lambda tomography is one of the first examples of specifically designing a reconstruction method for limited data. 
It is an analytic approach that was initially developed for `inverting' the Radon transform from region of interest data \cite{Vainberg:1981aa,Faridani:1992aa,Faridani:1997aa,Katsevich:1997aa}. 
The idea is to replace the standard filter in filtered back-projection \cite{Bracewell:1967aa,Shepp:1974aa,Gullberg:1979aa} (that has infinite support) with a filter that takes a second derivative in the detector variable. 
This numerical derivative filter has small support, just near the line being evaluated, hence the reconstruction method is local and thus it applies to region of interest data. 
Lambda tomography was later applied to limited-angle data \cite{Katsevich:1997ab}, where it recovers the visible wavefront set \cite{quinto1993singularities,quinto2008local,Frikel:2013ab,Quinto:2017aa}. 
The recovery is only mildly ill-posed \cite{quinto2008local,krishnan2015microlocal}, which is in stark contrast to the severe ill-posedness that is associated with attempts at reconstructing a function from limited-angle Radon data.

Lambda tomography is computationally efficient and offers an improvement over the traditional filtered back-projection method. 
Its robustness and usefulness was successfully demonstrated when Lambda tomography was applied for solving the limited-angle problem with highly noisy data in electron tomography \cite{Quinto:2009aa}. 
The analytic nature of the method also means the method is feasible for time-critical and/or large-scale problems.
The drawbacks of Lambda tomography are similar to those of filtered back-projection, namely that its filter needs to be specifically designed for the acquisition geometry and that the prior implicitly contained for regularising the problem has limited power.

\subsubsection{Variational models for reconstruction}
Much effort has been spent on adapting variational models of the form in \eqref{eq:C4S2E2} to limited-angle tomography. 
Methods cited here aim to design regularisers that are specifically adapted for limited-angle tomography, e.g., by 
accounting for the anisotropic resolution due to limited-angle data.

Many variational models for limited-angle tomography build on modifying the total variation (TV) regulariser, like various anisotropic versions that were  introduced in \cite{Esedoglu:2004aa,Berkels:2006aa}, see also \cite[Section~2.4]{Burger:2013aa}.
These can account for the scanning configuration bearing in mind the missing angular region in limited-angle data, which was also  done in \cite{Chen:2013aa, Zhang:2021aa}.
This idea was further elaborated in \cite{Wang:2017ab}, where an iteratively re-weighted anisotropic TV regularisation method was introduced to approximate the sparsity of $\ell_0$ semi-norm.
A further development came in \cite{Xu:2019aa} which set-up an alternating (directional) edge-preserving diffusion based on the one-dimensional $\ell_0$ semi-norm and a (directional) smoothing model based on a one-dimensional Dirichlet energy. The aforementioned directions are dictated by the missing angular region, and the regularisation strategy takes into consideration the fact that the $\ell_0$ norm is better at preserving edges, while the $\ell_1$ norm is better at suppressing noise. 
A further refinement to better balance the need for edge-preservation against noise-suppression was presented in 
\cite{Deng:2021aa}. 
Based on the theory of visible and invisible wavefront sets developed in \cite{quinto1993singularities}, the authors of \cite{Deng:2021aa} propose a reconstruction method that encodes the visible singularities as priors to recover the invisible ones. The model utilises generalised $p$-shrinkage operators introduced in \cite{Chartrand:2014aa} as regularisers to perform edge-preserving smoothing while using visible edges as anchors to recover piecewise constant or piecewise smooth reconstructions, while noises and artefacts are suppressed or removed. 
Another similar approach minimises an $L_1/L_2$ term on the gradient with additional box constraint that is reasonable for imaging applications \cite{Wang:2021aa}.

Other attempts at sparsity promoting variational models for limited-angle tomography rely on dictionary learning \cite{Tan:2015aa} or use a regulariser that promotes sparse solutions with respect to wavelets \cite{Wang:2017aa,Zhang:2018aa} or curvelets/shearlets \cite{Frikel:2013aa,Riis:2018aa}.
One can further constrain a sparse solution against a given prior image as shown in \cite{Chen:2008aa,Wang:2017aa,Zhang:2018aa,Gong:2019aa}.

Variational models tend to preserve edges better and reduce streak artefacts that are common in limited-angle tomography.
However, the improvements are most notable for the recovery of simplistic images, like piecewise constant or piecewise smooth images. 
Furthermore, this improvement is only realised when (multiple) hyper-parameters are properly tuned.
Finally, variational models are computationally demanding regarding run time and memory footprint, which in turn limits their usefulness in time-critical and/or large-scale imaging studies. 
In summary, the above cited variational models offer surprisingly little, if any, benefit in many cases where images have more complex features that are essential for their interpretation.

\subsubsection{Deep learning based methods for reconstruction}
There are several attempts at performing tomographic reconstruction by deep neural networks that are extensively surveyed in \cite{Arridge:2019aa}.
Most can be viewed as using techniques from deep learning to approximate different estimators for the statistical formulation of the reconstruction problem in \eqref{eq:InvProbBayes}. 

A popular approach for using deep learning in tomographic image reconstruction is to use a trained DNN as a  post-processing step for improving an initial imperfect reconstruction.
Such an approach was used for limited-angle tomography in \cite{Huang:2020aa}, which trains a U-Net against synthetic ellipsoid data and multi-category data to reduce artefacts from images obtained by filtered back-projection.

A more domain-adapted approach that outperforms post-processing is to consider DNN architectures for reconstruction that are obtained from unrolling a suitable iterative scheme.
One example is the Learned Primal-Dual network in \cite{adler2018lpd} outlined in Subsection~\ref{subsec:LPDArchitecture} that incorporates a handcrafted forward operator and the adjoint of its derivative along with the acquisition geometry. 
This DNN was recently further adapted to limited-angle tomography in breast tomosynthesis \cite{Teuwen:2021aa} by incorporating additional prior information about the geometry in the form of the thickness measurement of the breast under compression. 

A different deep learning based approach for limited-angle tomography is presented in \cite{bubba2018learning}.
Here, one learns the invisible part of the image using the visible part of its shearlet coefficients. This amounts to learning an anisotropic regulariser in a variational model.

\subsubsection{Sinogram inpainting}
An entirely different approach to limited-angle tomography is to fill in the missing angular data by some extrapolation scheme (sinogram inpainting or sinogram-recovery).
This pre-processing step needs to be done in a stable manner and \cite{Defrise:1983aa,Defrise:1984aa} uses a regularised iterative scheme for this purpose.
Another approach uses projection onto convex sets to ensure the extrapolated sinogram is indeed valid by making sure it lies in the range of the Radon transform (Helgason-Ludwig consistency condition) \cite{Kudo:1991aa}.

As to be expected, there have also been attempts at using deep learning for sinogram inpainting.
Here, much work has been inspired by the success that generative adversarial networks (GANs) have had in restoring missing parts of an image (image inpainting). 
In particular, \cite{nazeri2019edgeconnect} develops a deep learning-based image inpainting where one jointly trains the DNN that inpaints the edges and the image, using the philosophy `lines first, colour next'. 
The DNN for wavefront set inpainting (Subsection~\ref{sec:jointreconinpaint}) that we use in this paper is strongly inspired by \cite{nazeri2019edgeconnect}. In our case, the missing parts are in the sinogram domain. Consequently, we need to use the reconstruction method to map the singularities to the image domain.

Most approaches for deep learning based sinogram inpainting are based on setting up and training a GAN to generate the missing sinogram data in order to suppress the streak artefacts from the truncated sinogram in limited-angle data.
An example of such an approach is \cite{Li:2019aa} that uses a U-Net generator and patch-design discriminator in the GAN to make the network suitable for standard medical tomography images. 
The GAN is trained against paired limited-angle and full sinogram data using a joint projection domain and image domain loss function where the weighted image domain loss can be added by back-projection.
In this regime, we also refer to  \cite{Wang:2021ab,Podgorsak:2021aa} for similar approach based on GANs.

\subsubsection{Joint sinogram inpainting and reconstruction}
The final series of methods jointly perform the two tasks of sinogram inpainting and reconstruction. 
A variational model for doing this is presented in \cite{Tovey:2019aa} where the resulting non-convex and non-smooth minimisation is solved using an alternating (block) descent approach. 

A deep learning based approach is developed in \cite{Zhao:2018aa}. This approach combines a sinogram inpainting network and an image processing network. A key step is to use layers that correspond to Radon transform and its inverse inserted into existing convolutional network architectures. These allow one to go between sinogram and image domains and one can use the image processing network to reduce the artefacts caused by inconsistencies in the inpainted sinogram generated by the sinogram inpainting network. The three parts form an end-to-end network from sinogram domain to image domain, with benefits of taking both image error and sinogram error into account in the sync process in supervised training, i.e., with limited-angle/full-view sinogram pairs.
To tackle this training data bottleneck, we develop an unsupervised train method with only limited-angle projection on our proposed network. Inspired by the observation: reconstruction and projection can form a closed loop, we can derive a fake projection from the reconstructed image, and the disparity between the fake projection and real projection provides feedback signals to train the proposed network unsupervised.

Finally, we also mention \cite{Zhou:2021aa} that develops an approach with a transformer-based DNN architecture instead of convolutional DNNs. 
Streak artefacts in limited-angle tomography are non-local. Hence removing these with convolutional DNNs is challenging. 
One approach to encode such long-range dependencies is to use unrolling based DNN architectures, like the Learned Primal-Dual network, that couple many convolutional neural networks with the forward operator and the derivative of its adjoint.  
Another approach is to consider transformer-based architectures that are better suited than convolutional DNNs due to their non-local nature. 
As shown in \cite{Zhou:2021aa}, such transformer-based DNNs have excellent performance for reconstruction in limited angle tomography.
A downside of transformer DNNs is that they are very demanding to train and require massive amounts of training data, so this approach will scale poorly. 
Some of the issues related to memory footprint could perhaps be addressed by using more domain adapted transformer architectures, like Fourier Image Transformer \cite{Buchholz:2021aa}.
Furthermore, there are still no guarantees that \cite{Zhou:2021aa} extrapolates the missing data in a microlocally consistent manner.

\section{Inverse problems, ill-posedness and regularisation}\label{subsec:InvProb}
To mathematically formalise the notion of an abstract inverse problem, we introduce the separable Banach spaces $\RecSpace$ (\emph{reconstruction space}) and $\DataSpace$ (\emph{data space}) whose elements represent possible signals and data, respectively. 
In many applications, the reconstruction and data spaces are also Hilbert spaces.
Next, the mapping $\ForwardOp \colon \RecSpace \to \DataSpace$ (\emph{forward operator}) represents a model for how a signal generates noise free data.

The classical functional analytic/frequentist formalisation views an inverse problem as an operator equation. 
More precisely, an inverse problem is the task of recovering an unknown signal $\signaltrue \in \RecSpace$ from data $\data \in \DataSpace$ that is a single sample of the $\DataSpace$-valued random variable $\stdata$ defined as 
\begin{equation}\label{eq:InvProbFreq}
    \stdata \coloneqq \ForwardOp(\signaltrue)+\stdatanoise
    \quad\text{where $\stdatanoise$ is a $\DataSpace$-valued random variable representing observation noise.}
\end{equation}
The statistical formulation further generalises the above by first assuming one can equip $\RecSpace$ and $\DataSpace$ with Borel $\sigma$-algebrae. 
One furthermore assumes that the signal and corresponding data are generated by some $(\RecSpace \times \DataSpace)$-valued random variable $(\stsignal,\stdata)$ \cite{stuart2010invprob}. 
The inverse problem is now to recover a suitable estimator of the conditional random variable $(\stsignal \mid \stdata= \data)$ from data $\data \in \DataSpace$ that is a single sample of the $\DataSpace$-valued conditional random variable $(\stdata \mid \stsignal=\signaltrue)$, where $\signaltrue \in \RecSpace$ is the true unknown signal and  
\begin{equation}\label{eq:InvProbBayes}
  \stdata = \ForwardOp(\stsignal) + \stdatanoise
  \quad\text{where $\stdatanoise$ is a $\DataSpace$-valued random variable representing observation noise.}  
\end{equation}

\subsection{Ill-posedness and regularisation}\label{subsec:IllPosedReg}
A reconstruction method is formally a mapping $\RecOp \colon \DataSpace \to \RecSpace$ that approximates the inverse of the forward operator. 
An inverse problem is said to be (intrinsically) unstable, i.e., ill-posed, whenever the forward operator $\ForwardOp$ is not continuously invertible with respect to the topologies of $\RecSpace$ and $\DataSpace$.
As already indicated, such inverse problems cannot be reliably solved by merely enforcing consistency against data. 
Regularisation theory addresses this by balancing the need for data consistency against consistency with respect to a \emph{prior model}. 
Defining an appropriate prior model and how to balance this against data consistency are key topics in regularisation theory.

Most regularisation methods are based on the functional analytic/frequentist formalisation in \eqref{eq:InvProbFreq} of an inverse problem.
Variational models offer a powerful framework for regularisation. The reconstruction method $\widehat{\RecOp} \colon \DataSpace \to \RecSpace$ is here defined as solving a variational problem: 
\begin{equation}\label{eq:C4S2E2}
  \widehat{\RecOp}(\data) \in 
  \argmin_{\signal \in \RecSpace} \Bigl\{ \DataDiscrep\bigl( \ForwardOp(\signal),\data \bigr)
    +\Regulariser_{\theta}(\signal)
  \Bigr\}
  \quad \text{for given data $\data \in \DataSpace$.}
\end{equation}
In the above, $\DataDiscrep \colon \DataSpace \times \DataSpace \to \Real$ is the data discrepancy functional that quantifies the data consistency. It is usually taken as a suitable affine transformation of the data log-likelihood \cite{bertero2008iterative}. 
For an ill-posed problem, merely minimising $\signal \mapsto \DataDiscrep\bigl( \ForwardOp(\signal),\data \bigr)$, which translates to seeking the solution that is maximally consistent with data, is an unstable procedure. 
The \emph{regularisation functional} $\Regulariser_{\theta} \colon \RecSpace \to \Real$ stabilises the reconstruction by penalising those candidate solutions that are not consistent with respect to some prior model. The latter is typically given in terms of a-priori information about the (unknown) ground truth signal $\signaltrue \in \RecSpace$ such as sparsity or some type of regularity, see \cite{CTmodel3,Burger:2013aa,Benning:2018aa} for an extensive survey of various options.
The parameter $\theta$ (\emph{regularisation parameter}) governs the balance between data consistency and having a solution consistent against the prior. It is low-dimensional, in many cases a scalar, and its choice depends on the magnitude of the noise in data.

The statistical formalisation of the inverse problem in \eqref{eq:InvProbBayes} contains many of the variational models.
More precisely, if the regularisation functional is proportional to the negative log of a density on $\RecSpace$, then one can often interpret \eqref{eq:C4S2E2} as computing a maximum a posteriori estimator.
An advantage with the statistical formalisation is that it opens up for other reconstruction methods (estimators) such as the posterior mean estimator that is known to be stable in most cases \cite{Latz:2020aa}. 
This estimator is given as $\Expect[ \stsignal \mid \stdata = \data ]$, a formulation that requires access to the (posterior) distribution of $( \stsignal \mid \stdata = \data)$.
Using Bayes' theorem, one can in principle express this posterior distribution in terms of the data likelihood $( \stdata \mid \stsignal = \signal)$, which is given by the physics of data acquisition, and the true (prior) distribution of $\stsignal$, which among others generates the true unknown signal $\signaltrue \in \RecSpace$. 
An alternative formulation is to phrase the posterior mean as a Bayes estimator with respect to the squared $\Lp^2$-risk. 
Stated more precisely, given a parametrised family $\{ \RecOp_{\theta} \}_{\theta \in \Theta}$ of admissible estimators $\RecOp_{\theta} \colon \DataSpace \to \RecSpace$, we consider the reconstruction operator 
\begin{equation}\label{eq:PostMean2}
  \RecOp_{\hat{\theta}} \colon \DataSpace \to \RecSpace
  \quad\text{where}\quad
  \hat{\theta}
    \in \argmin_{\theta \in \Theta} \Expect_{(\stsignal,\stdata)}\Bigl[ \bigl\Vert \RecOp_{\theta}(\stsignal) - \stdata \bigr\Vert_2^2 \Bigr].
\end{equation}
The above approximates the posterior mean, i.e.,
$\RecOp_{\hat{\theta}}(\data) = \Expect[ \stsignal \mid \stdata = \data ]$ whenever $\{ \RecOp_{\theta} \}_{{\theta \in \Theta}}$ has sufficient expressive power. 
Note also that the expectation in \eqref{eq:PostMean2} is w.r.t.\@ the joint law for $(\stsignal, \stdata)$ in $\RecSpace \times \DataSpace$.
Many supervised learning approaches for reconstruction are based on replacing the expectation in \eqref{eq:PostMean2} with its empirical counterpart given by supervised training data, see \cite[Section~5]{Arridge:2019aa} for an extensive survey.
In particular, it is common to use a deep neural network (DNN) with an appropriate architecture to parametrise the family of estimators $\{ \RecOp_{\theta} \}_{{\theta \in \Theta}}$.
One can use unrolling to define such an architecture as outlined in Subsection~\ref{subsec:LPDArchitecture}. 

Finally, besides unrolling, there exist also DNN architectures that are specifically tailored for approximating Fourier integral operators. Examples are \cite{Fan:2019aa,Feliu-Faba:2020aa} which develop a DNN architecture based on operator splitting techniques derived from multiscale numerical analysis.
Another example is \cite{Kothari:2020aa}, which builds an interpretable DNN inspired by Fourier integral operators that approximate the wave physics. Its main focus is on using a loss based on optimal transport to learning the geometry of wave propagation captured by Fourier integral operators, which is implicit in the training data.

\subsection{The Learned Primal-Dual network}\label{subsec:LPDArchitecture}
A particular challenge that arises in deep learning based approaches for solving inverse problems in imaging is to handle the large scale nature of the problem as both images and data typically result in high dimensional arrays once digitised.
Many of these applications also lack sufficient amount of training data, hence using a generic DNN architecture will result in a learned reconstruction operator that does not generalise well.   
For such applications, one needs to use a DNN architecture that is domain-adapted. 

If the trained DNN $\RecOp_{\hat{\theta}} \colon \DataSpace \to \RecSpace$ should correspond to a (learned) reconstruction operator for an inverse problem of the form in \eqref{eq:InvProbBayes}, then a natural domain adaptation is to account for the requirement that $\RecOp_{\hat{\theta}} \approx \ForwardOp^{-1}$ where $\ForwardOp \colon \RecSpace \to \DataSpace$ (forward operator) is handcrafted (not learned).  
Such domain adaptation can be achieved by choosing a DNN architecture that is given by unrolling a suitable iterative scheme, see \cite[Section~4]{Arridge:2019aa}. 
The highly successful Learned Primal-Dual network introduced in \cite{adler2018lpd} provides state-of-the-art results for tomographic imaging. A similar unrolling technique is used in \cite{Bubba:2021aa} to construct a novel  convolutional DNN architecture, called $\Psi$DONet, for learning pseudodifferential operators. This is applied to reconstruction in limited-angle tomography. 

It is based on unrolling the non-linear primal-dual hybrid gradient method \cite{chambolle2010primal}.
Stated in a general setting, the Learned Primal-Dual network is a DNN with an architecture specifically adapted for solving an inverse problem of the form in \eqref{eq:InvProbBayes}.
The Learned Primal-Dual network is here a reconstruction operator $\RecOp_{\theta} \colon \DataSpace \to \RecSpace$ with $\theta = (\theta_{0}^d,\theta_{0}^p, \ldots, \theta_{N}^d,\theta_{N}^p)$ that is defined as $\RecOp_{\theta}(\data) \coloneqq \signal_N$ where $\signal_N$ is given by the following finite recursive scheme initialised by $(\signal_0,\dualvar_0) \in \RecSpace \times \DataSpace$:
\begin{equation}\label{eq:LPDiterates}
\begin{cases}
  \dualvar_{i} \coloneqq \Lambda^{\mathrm{dual}}_{\theta_{i}^d}(\dualvar_{i-1},\ForwardOp(\signal_{i-1}),\data), &
  \\[1em]
  \signal_{i} \coloneqq \Lambda^{\mathrm{primal}}_{\theta_{i}^p} 
    \bigl(\signal_{i-1},[\partial\!\ForwardOp(\signal_{i-1})]^*(\dualvar_{i})\bigr), &
\end{cases}
\quad\text{for $i=1,\ldots, N$.}
\end{equation}
A typical initialisation $(\signal_0,\dualvar_0) \in \RecSpace \times \DataSpace$ is $\signal_0= 0$ and $ \dualvar_0 = \data$.
Next, $\ForwardOp \colon \RecSpace \to \DataSpace$ is the (handcrafted) forward operator in \eqref{eq:InvProbBayes} and $\partial\!\ForwardOp \colon \RecSpace \to L(\RecSpace,\DataSpace)$ is its Fr\'echet derivative with $*$ denoting the dual.
Finally, the operators 
\begin{equation}\label{eq:LPDResNets} 
\Lambda^{\mathrm{dual}}_{\theta_{i}^d} \colon \DataSpace \times \DataSpace \times \DataSpace \to \DataSpace 
   \quad\text{and}\quad
   \Lambda^{\mathrm{primal}}_{\theta_{i}^p} \colon \RecSpace \times \RecSpace \to \RecSpace
\end{equation}
are neural networks with suitable architectures. 

The Learned Primal-Dual network introduced in \cite{adler2018lpd} for tomographic image reconstruction operates directly on arrays, which here represent discretised functions. 
As outlined in Subsection~\ref{subsec:TomoDisc}, discretised images in tomography are represented by arrays in $\Real^{n_1 \times n_2}$, whereas arrays in $\Real^{m_1 \times m_2}$ are sinograms, which are discretisations of a real-valued function on $\datadomain \subset \Real \times (0,\pi)$. 
The Learned Primal-Dual network is then a mapping $\vecRecOp_{\theta} \colon \Real^{m_1 \times m_2} \to \Real^{n_1 \times n_2}$ given as in \eqref{eq:LPDiterates}, where the operators in \eqref{eq:LPDResNets} are convolutional residual neural networks (ResNets) \cite{he2016resnet}.
The precise architecture used in \cite{adler2018lpd} for these ResNets is given next, see also Figure~\ref{fig:learnedPrimalDualDiscrete} for an illustration of these ResNets alongside the particular architecture for the corresponding Learned Primal-Dual network.
We will in Subsection~\ref{subsec:CLDP} extend these ResNets, and consequently the Learned Primal-Dual network, to operators that act on functions.
\begin{definition}[Discrete two-dimensional convolutional ResNet]\label{def:ResNetDiscrete}
Assume arrays in $\Real^{n_1 \times n_2}$ represent real-valued functions on $\Real^2$ that are discretised at $n_1 \times n_2$ sample points.
Let $\vecConvAffine_{\boldsymbol{\theta}_j,b_j} \colon (\Real^{n_1 \times n_2})^{k_{j-1}} \to (\Real^{n_1 \times n_2})^{k_j}$ for $j=1,\ldots, 4$ denote the following (discretised) convolutional affine operator: 
\begin{equation}
\label{eq:discconvresent}
    \vecConvAffine_{\boldsymbol{\theta}_j,b_j}(\vecsignal)(i_1,i_2,\tau) 
    \coloneqq
    b_j^{\tau}(i_1,i_2) + \sum_{l = 1}^{k_{j-1}} (\boldsymbol{\theta}_j^{l,\tau} * \vecsignal)(i_1,i_2) 
    \quad\text{for $\tau \in \{1,\ldots, k_j\}$ and $ 
    \vecsignal\in (\Real^{n_1 \times n_2})^{k_{j-1}}$.}
\end{equation}
In the above, $k_j \in \Natural$ for $j=0,\ldots, 4$ is the \emph{numbers of channels in the $j$:th layer} with $k_4 = 1$, i.e., the final layer is always a single channel. 
Next, 
\[ \boldsymbol{\theta}_j
   \coloneqq 
   (\boldsymbol{\theta}_j^{l,\tau})^{k_{j-1},k_j}_{l=1,\tau=1}\in (\Real^{3 \times 3})^{k_{j-1} \times k_{j-1}}
   \quad\text{for $j = 1,\dots, 4$}
\]
are the channel-wise $3 \times 3$ convolutional filters with $b_j \in (\Real^{n_1 \times n_2})^{k_j}$ as corresponding channel-wise bias terms. 
The \emph{residual convolutional network (ResNet)} operator is a mapping $\vecResNet \colon (\Real^{n_1 \times n_2})^{n_0} \to \Real^{n_1 \times n_2}$ defined as 
\[
    \vecResNet(\vecsignal_1, \dots, \vecsignal_{n_0})
    \coloneqq 
    \vecsignal_1 +\vecoperator{F}(\vecsignal_1, \dots , \vecsignal_{n_0}) \quad \text{ for } \vecsignal_1,\dots, \vecsignal_{n_0}\in \Real^{n_1 \times n_2},
\]
where $\vecoperator{F} \colon (\Real^{n_1 \times n_2})^{n_0} \to \Real^{n_1 \times n_2}$ is the operator
\[
    \vecoperator{F}(\vecsignal_1, \dots, \vecsignal_{n_0}) 
    \coloneqq 
    \bigl(\vecConvAffine_{\boldsymbol{\theta}_4, b_4} \circ \ReLU
  \circ 
  \vecConvAffine_{\boldsymbol{\theta}_3, b_3} \circ \ReLU
  \circ 
  \vecConvAffine_{\boldsymbol{\theta}_2, b_2} \circ \ReLU
  \circ 
  \vecConvAffine_{\boldsymbol{\theta}_1, b_1}
  \bigr)(\vecsignal_1, \dots, \vecsignal_{n_0}).
\]
In the above, $\ReLU(x) \coloneqq \max\{x,0\}$ is applied component-wise.
\end{definition}
\begin{figure}[htb]
\centering
\includegraphics[width = 0.5\textwidth]{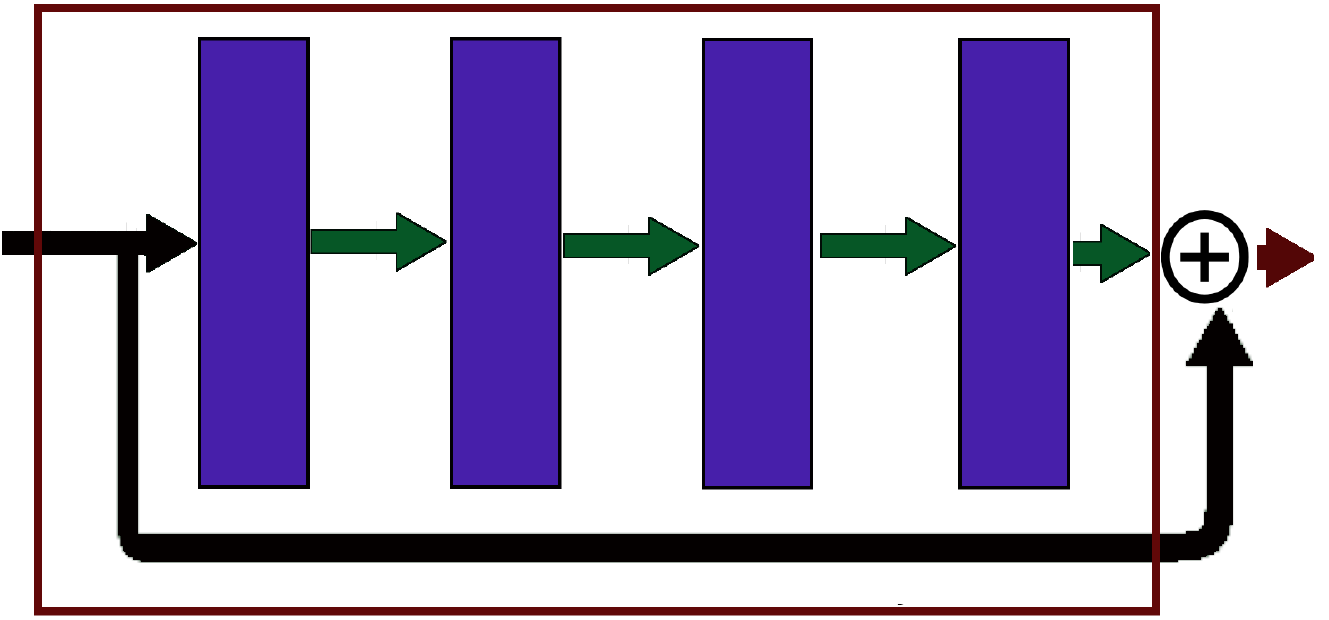}
\begin{tikzpicture}[remember picture,overlay]
  \node at (-11.5,0.55) {\includegraphics[scale=0.25]{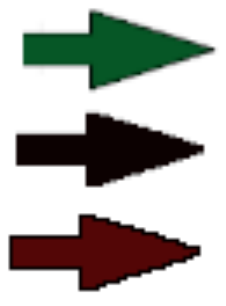}};
  \node[anchor=west, draw=none, fill=none] at (-11.3,0.8)
    {{\scriptsize \scriptsize $3 \times 3$-conv + ReLU}};
  \node[anchor=west, draw=none, fill=none] at (-11.3,0.53)
    {{\scriptsize \scriptsize Copy}};
  \node[anchor=west, draw=none, fill=none] at (-11.3,0.3)
    {{\scriptsize \scriptsize ReLU}};
  \node[anchor=west, draw=none, fill=none] at (-8.8,2.35)
    {{$\vecsignal$}};
  \node[anchor=west, draw=none, fill=none] at (-7.05,0.67)
    {{\scriptsize 32}};
  \node[anchor=west, draw=none, fill=none] at (-5.5,0.67)
    {{\scriptsize 32}};
  \node[anchor=west, draw=none, fill=none] at (-3.95,0.67)
    {{\scriptsize 32}};
  \node[anchor=west, draw=none, fill=none] at (-2.32,0.67)
    {{\scriptsize 32}};
  \node[anchor=west, draw=none, fill=none] at (-5,0.22)
    {{\scriptsize Identity}};
  \node[anchor=north, draw=none, fill=none] at (0.25,2.55)
    {{$\vecoperator{F}(\vecsignal)$}};
  \node[anchor=north, draw=none, fill=none] at (-4.5,4.35) {{$\vecoperator{F}$}};
\end{tikzpicture}
\par{\textbf{(i)}}
\par\bigskip
\includegraphics[width = 0.8\textwidth]{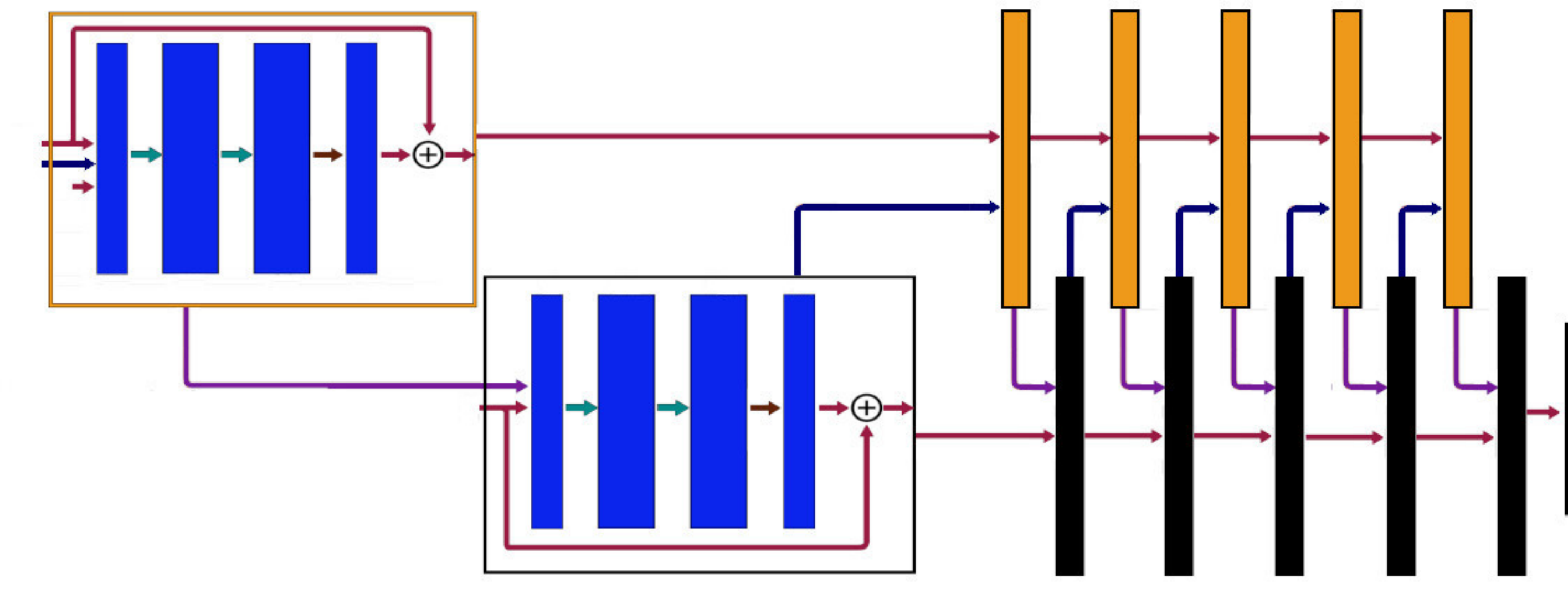}
\begin{tikzpicture}[remember picture,overlay]
  \node at (-14,0.75) {\includegraphics[scale=1.1]{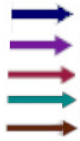}};
  \node[anchor=west, draw=none, fill=none] at (-13.6,1.375)
    {{\scriptsize Apply $\vecForwardOp$}};
  \node[anchor=west, draw=none, fill=none] at (-13.6,1.05)
    {{\scriptsize Apply $\vecForwardOp^{\ast}$}};
  \node[anchor=west, draw=none, fill=none] at (-13.6,0.7)
    {{\scriptsize Copy}};
  \node[anchor=west, draw=none, fill=none] at (-13.6,0.4)
    {{\scriptsize $3 \times 3$-conv + ReLU}};
  \node[anchor=west, draw=none, fill=none] at (-13.6,0.1)
    {{\scriptsize $3 \times 3$-conv}};
  \node[anchor=west, draw=none, fill=none] at (-13.45,3.8)
    {{\scriptsize $\vecdualvar_0$}};
  \node[anchor=west, draw=none, fill=none] at (-13.45,3.55)
    {{\scriptsize $\vecsignal_{\!0}$}};
  \node[anchor=west, draw=none, fill=none] at (-12.99,3.35)
    {{\scriptsize $\vecdata$}};
  \node[anchor=west, draw=none, fill=none] at (-12.55,2.52) {{\tiny $7$}};    
  \node[anchor=west, draw=none, fill=none] at (-11.95,2.52) {{\tiny $32$}};    
  \node[anchor=west, draw=none, fill=none] at (-11.2,2.52) {{\tiny $32$}};    
  \node[anchor=west, draw=none, fill=none] at (-10.45,2.52) {{\tiny $5$}};    
  \node[anchor=west, draw=none, fill=none] at (-9.7,1.5) 
      {{\scriptsize $\vecsignal_{\!0}$}}; 
  \node[anchor=west, draw=none, fill=none] at (-8.9,2.545) {{\tiny $6$}};    
  \node[anchor=west, draw=none, fill=none] at (-8.3,2.545) {{\tiny $32$}};    
  \node[anchor=west, draw=none, fill=none] at (-7.5,2.545) {{\tiny $32$}};    
  \node[anchor=west, draw=none, fill=none] at (-6.8,2.545) {{\tiny $5$}};     
\end{tikzpicture}
\par{\textbf{(ii)}}
\caption{\textbf{(i)} The architecture for the two-dimensional convolutional ResNet in Definition~\ref{def:ResNetDiscrete}.
\textbf{(ii)} The full Learned Primal-Dual architecture, every yellow and black block corresponds to a two-dimensional convolutional ResNet.}
\label{fig:learnedPrimalDualDiscrete}
\end{figure}
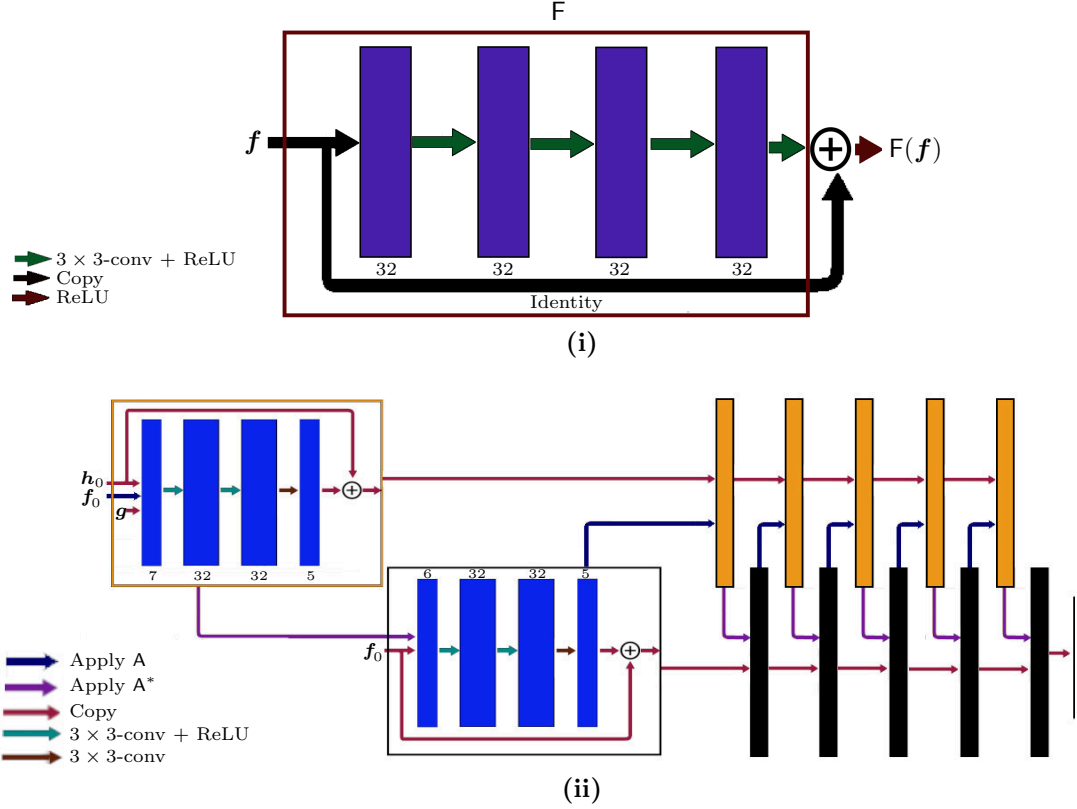

\section{Tomographic imaging and the Radon transform}\label{sec:PlanarTomo}
Planar tomographic imaging aims to recover a 2D image (signal) of the interior of an object from corresponding tomographic data.
Thus, elements in $\RecSpace$ are real-valued functions on a fixed domain $\domain \subset \Real^2$ that represent possible 2D images, and $\RecSpace$ itself is some suitable function space, e.g., $\RecSpace \subset \Lp^2(\Real^2)$.

Next, many tomographic imaging modalities rely on probing the object with x-rays.
In the planar setting, these rays are all contained in a cross sectional hyperplane through the object.  
After adopting a simplified physics model for the interaction between x-rays and the object, measured data (after appropriate pre-processing) can be viewed as noisy digitised samples of the Radon transform of the aforementioned signal. 
\begin{definition}[Radon transform]\label{def:radontrans}
The \emph{(planar) Radon transform} of $\signal \colon \Real^2 \to \Real$ is defined as
\begin{equation}\label{eq:S1E3}
  \Radon(\signal)(s,\theta) \coloneqq \int_{-\infty}^{\infty} \signal\bigl(s\omega(\theta)+t\omega(\theta)^{\perp}\bigr)dt, 
  \quad \text{for $(s,\theta)\in \Real\times (0,\pi)$.}
\end{equation}
In the above, $\omega(\theta) \coloneqq (\cos\theta,\sin\theta)$ is the unitary vector with orientation described by the angle $\theta$ with respect to the $x_1$-axis and $\omega(\theta)^{\perp} \coloneqq (-\sin\theta,\cos\theta)$. 
\end{definition}
\begin{remark}
There do exist integrable functions $\signal$ for which there is a direction $\omega(\theta) \in \mathbb{S}^1$ and a point $s \in \Real$ such that $\Radon(\signal)(s,\theta)$ does not exist. Nonetheless, it is possible to show that $\Radon(\signal)(s,\theta)$ exists almost everywhere when $\signal$ is an integrable function \cite[Proposition~2.38]{markoe2006tomo}.
\end{remark}
Note that $(s,\theta)$ represents the line $t \mapsto s\omega(\theta)+t\omega(\theta)^{\perp}$ that is orthogonal to $\omega(\theta) \in \mathbb{S}^1$ with (signed) distance $s \in \Real$ to the origin, so $\Radon(\signal)$ is a function on lines in $\Real^2$.
Since data in tomographic imaging can be seen as noisy samples of $\Radon(\signal)$, the data space $\DataSpace$ is an appropriate space of real-valued functions on lines in $\Real^2$ representing non-digitised data.
The angle $\theta \in (0,\pi)$ governs the direction of the x-ray that probes the object. It typically varies during the tomographic data acquisition and limited-angle tomography is the case when $\theta$ is restricted to some interval $I \subset (0,\pi)$.
\begin{definition}[Tomographic reconstruction]\label{def:tomrecon}
\emph{Tomographic (image) reconstruction} is the inverse problem of recovering a \emph{ground truth} image $\signaltrue\in \Lp^2(\Real^2)$ from noisy measurements $\data \in \Lp^2(\datadomain)$ for some open set $\datadomain \subset \Real\times [0,\pi]$. 
Here, $\data$ is a single sample of the $\Lp^2(\datadomain)$-valued random variable $\stdata$ in \eqref{eq:InvProbFreq} (or \eqref{eq:InvProbBayes} for the statistical formulation) with $\Radon$ denoting the Radon transform (Definition~\ref{def:radontrans}).
\end{definition}
 
\subsection{Basic properties of the Radon transform}\label{subsec:RadonProp}
From a functional analytic viewpoint, the Radon transform is a linear operator that maps functions on $\Real^2$ to functions on the open set $\Real\times (0,\pi)$ (representing lines in $\Real^2$). 
Continuity of the Radon transform depends on the domain of the functions chosen. 
As an example, the Radon transform is a bounded map on $\Lp^1(\Real^2)$, implying that it is a continuous linear operator on $\Lp^1(\Real^2)$ \cite[Corollary 3.25]{markoe2006tomo}, it is however unbounded on $\Lp^2(\Real^2)$. 
The Radon transform is also invertible but its inverse is not continuous on $\Lp^1\bigl(\Real\times(0,\pi)\bigr)$ \cite{Hertle:1983aa}, yielding that the tomographic image reconstruction problem in Definition~\ref{def:tomrecon} is an ill-posed inverse problem. 

We will consider the Radon transform on Schwartz functions $\SchwartzFunc(\Real^2)$, i.e., rapidly decreasing and smooth functions with the typical locally convex topology. It is well-known that the Fourier transform maps $\SchwartzFunc(\Real^2)$ onto itself. 
An analogous result holds for the Radon transform, namely that $\Radon \colon \SchwartzFunc(\Real^2) \to \SchwartzFunc\bigl(\Real\times(0,\pi)\bigr)$ is a linear one-to-one mapping \cite[Theorem~2.4]{Helgason:1999aa}.
\begin{remark}
$\SchwartzFunc(\domain)$ for some open set $\domain \subset \Real^2$ is defined as the set of functions on $\domain$ such that their extension by $0$ to all of $\Real^2$ is a Schwartz function. 
In particular, this defines the space of Schwartz functions on any open set $\datadomain \subset \Real\times(0,\pi)$.
\end{remark}

Next, we introduce the back-projection as the dual to the Radon transform in a sense analogous to the way the adjoint of a linear transformation on Euclidean space is dual to the original transformation.
\begin{definition}[Back-projection]\label{def:BackProj}
The \emph{back-projection} of $\data \colon \Real\times (0,\pi) \to \Real$ is the function $\Radon^*(\data) \colon \Real^2 \to \Real$ defined as  
\begin{equation}\label{eq:BackProj}
  \Radon^*(\data)(x) \coloneqq  \int_{0}^{\pi} \data(x \cdot \omega(\theta),\theta)\,d\theta
  \quad\text{for $x \in \Real^2$.}
\end{equation}
\end{definition}
The back-projection maps a function $\data$ on lines in $\Real^2$ to a function on points in $x \in \Real^2$ by simply averaging $\data$ over all lines that pass through $x$. 
A simple calculation shows that the back-projection is the dual to the Radon transform \cite[Theorem~2.75]{markoe2006tomo}:
\begin{equation}\label{eq:FormalAdjointL2}
 \bigl\langle \Radon(\signal), \data \bigr\rangle = \bigl\langle \signal, \Radon^*(\data) \bigr\rangle.
\end{equation} 
The inner product on the right-hand side refers to the natural inner product on $\Lp^2(\Real^2)$, whereas the inner product on the left-hand side is the natural inner product on $\Lp^2\bigl(\Real \times (0,\pi) \bigr)$. This product is well defined in tempered distributions since $\SchwartzFunc(\Real^2) \subset \Lp^2(\Real^2)$ (respectively  $\SchwartzFunc(\Real \times (0,\pi) \bigr) \subset \Lp^2(\Real \times (0,\pi) \bigr)$), \cite{schwartz1954multdistr}.

The duality in \eqref{eq:FormalAdjointL2} can be used to extend the Radon transform to various classes of distributions, like compactly supported distributions \cite{Helgason:1999aa} and tempered distributions \cite{Gelfand:1966aa,Ludwig:1966aa}.

In particular, one can define the Radon transform on a tempered distribution $\signal \in \SchwartzFunc'(\Real^2)$ as
\begin{equation*}
    \Radon(\signal)(\phi) \coloneqq \signal\bigl(\Radon^*(\phi)\bigr) 
    \quad\text{for all $\phi \in \SchwartzFunc\bigl(\Real \times (0,\pi) \bigr)$.}
\end{equation*}
The extension of $\Radon^*$ to $\SchwartzFunc\bigl(\Real \times (0,\pi) \bigr)$ is defined analogously and (following \cite[Section~2.9.3.2 and 4.3.1]{markoe2006tomo}) one can in addition show that, 
\begin{equation}\label{eq:RadonOnTempDistr}
\begin{split}
  \Radon \colon \SchwartzFunc'(\Real^2) \to \SchwartzFunc'\bigl(\Real\times(0,\pi)\bigr)
  & \quad\text{is a topological isomorphism,}
\\  
  \Radon^* \colon \SchwartzFunc'\bigl(\Real\times(0,\pi)\bigr) \to \SchwartzFunc'(\Real^2) 
  & \quad\text{is surjective.} 
\end{split}
\end{equation}

In limited-angle tomography, data is given on lines contained in some open subset $\datadomain \subset \Real \times [0, \pi)$.
The \emph{partial} Radon transform is then defined as 
\[ \Radon_\datadomain \coloneqq P_{\SchwartzFunc'(\datadomain)} \circ \Radon, \]
where $P_{\SchwartzFunc'(\datadomain)}$ is the restriction of $\SchwartzFunc'(\Real^2)$ to $\SchwartzFunc'(\datadomain)$. 
Note that the restriction of a tempered distribution $\SchwartzFunc'(\Real^2)$ to $\SchwartzFunc(\datadomain)$, where $\datadomain \subset \Real^2$ is open, is well defined and it corresponds to a tempered distribution in $\SchwartzFunc'(\datadomain)$.
The corresponding (restricted) back-projection is simply defined as in \eqref{eq:BackProj} but by setting $\data$ to 0 on lines not contained in $\datadomain$. In particular, if $\datadomain \coloneqq \Real\times I \subset \Real\times (0,\pi)$ for some open interval $I \subset  (0,\pi)$, then 
\begin{equation}\label{eq:LimitedAngBP}
  \Radon_{\datadomain}^*(\data)(x) =  \int_{I} \data(x \cdot \omega(\theta),\theta)\,d\theta
  \quad\text{for $x \in \Real^2$.}
\end{equation}

\subsection{Discretisation}\label{subsec:TomoDisc}
An image (signal) $\signal \colon \domain \to \Real$ defined on a domain $\domain \subset \Real^2$ is in tomographic imaging typically digitised as $\vecsignal=( \signal_{i,j} )_{i,j=1}^{n_1, n_2}$ where $\signal_{i,j} \coloneqq \signal(x_{i,j})$ for $x_{i,j} \in \domain$.
Similarly, a sinogram $\data \in \DataSpace$ is digitised by an array $\vecdata\in \Real^{m_1 \times m_2}$ by evaluating $\data$ on $m_1 \times m_2$ sample points in $\datadomain \subset \Real \times (0,\pi)$ that are given by the data acquisition protocol used in the actual scanner. 
The $m_1$ sample points in $\Real$ would simply correspond to detector elements in the scanner, whereas the $m_2$ sample points in $(0,\pi)$ correspond to directions used in the actual scanning.

The (discretised) Radon transform and back-projection are then mappings 
\begin{equation}\label{eq:DiscRadon} 
\vecRadon \colon \Real^{n_1 \times n_2} \to \Real^{m_1 \times m_2}
\quad\text{and}\quad
\vecRadon^* \colon \Real^{m_1 \times m_2} \to \Real^{n_1 \times n_2}.
\end{equation}
We wish to remark that there exist sophisticated numerical schemes for evaluating these mappings in an accurate, yet computationally feasible manner, see, e.g., \cite{Kim:2012aa,Long:2010aa,Aarle:2015aa,Karimi:2016aa,Aarle:2016aa}.

It is now possible to formulate the tomographic reconstruction problem entirely in the discrete setting.
This is essentially a digitised version of the tomographic reconstruction problem in Definition~\ref{def:tomrecon}, i.e., to recover the digital image $\vecsignal_{\mathrm{true}}\in \Real^{n_1 \times n_2}$ from noisy measurements $\vecdata \in \Real^{m_1 \times m_2}$ where 
\[
  \vecdata = \vecRadon(\vecsignal_{\mathrm{true}}) + \boldsymbol{\datanoise}
\quad\text{for some (unknown) observation error $\boldsymbol{\datanoise} \in \Real^{m_1 \times m_2}$.}
\]
One could then proceed to develop regularised reconstruction methods for the above finite dimensional inverse problem.
This `discretise first then reconstruct' approach was common in the 1970's with the advent of iterative schemes with early stopping, like
algebraic reconstruction techniques (ART) \cite{Gordon:1970aa}, simultaneous ART (SART) \cite{Andersen:1984aa} or simultaneous iterative reconstruction technique (SIRT) \cite{Gilbert:1972aa}.
The success of FBP and subsequent variational models did however point to the advantage of developing reconstruction methods for the continuum formulation, which then are discretised.
This is also the approach we take in our work, since the microlocal analysis that we will rely on is naturally formulated in the continuum setting.

\subsection{Microlocal analysis in tomography}\label{subsec:wavefrontset}
A central part of microlocal analysis is the study of how singularities are transformed by specific operators.
This is not possible given only the singular support that describes the location of the singularities.
Instead, an appropriate notion of singularity needs to include information about directions in the Fourier domain that causes the singularity.
This leads to the wavefront set that we define below.  

\begin{definition}[Wavefront set]\label{def:S1D4}
Let $u\in \SchwartzFunc'(\Real^n)$ and $N\in \mathbb{N}$. A point $(x,\xi)\in \Real^n\times \mathbb{S}^{n-1}$ is an \emph{$N$-regular directed point} of $u$ if there exist an open neighborhood $U_x$ of $x$, a conical neighborhood $V_{\xi}$ of $\xi$, and a smooth cut-off function $\psi\in \SchwartzFunc(\Real^n)$ with $\supp \psi \subset U_x$ and $\psi(x) = 1$ such that the following holds for some $C_N>0$:
\[
  \bigl| \widehat{\psi u}(\eta) \bigr| \leq C_N \bigl(1+|\eta| \bigr)^{-N} 
  \quad \text{ for all $\eta \in \Real^n$ with $\eta/|\eta|\in V_{\xi}$.}
\]
The \emph{$N$-wavefront set} $\WF_N(u)$ is the complement of the set of all $N$-regular directed points.
Finally, the \emph{wavefront set} $\WF(u)$ of $u$ is defined as 
\[
\WF(u) \coloneqq \bigcup_{N\in\mathbb{N}}\WF_N(u).
\]
\end{definition}
From the above, we see that the wavefront set $\WF(u)$ is the set of all position-orientation pairs in $\domain \times \mathbb{S}^1$ along which $u$ is non-smooth.
\begin{remark}
The general setting in \cite[Section~8.1]{AnLinPDOHoermander}, where functions/distributions are defined on manifolds, views the wavefront set as a conical subset of the cotangent bundle.
This is equivalent to Definition~\ref{def:S1D4} when the manifold is an open domain in $\Real^2$.
\end{remark}

\subsubsection{The canonical relation for the Radon transform and its back-projection}
\label{subsec:ResNetMicroLocalS4}
For many operators arising in applications, it is possible to describe how they transform the wavefront set. 
This description is called  \emph{(microlocal) canonical relation}. It plays an important role in inverse problems, since it allows one to relate singularities in data to those in the unknown signal.

The canonical relation for the Radon transform in  Definition~\ref{def:radontrans} at tempered distribution $\signal \in \SchwartzFunc'(\Real^2)$ is a precise relationship between $\WF(\Radon(\signal))$ and $\WF(\signal)$. 
If $\mathcal{P}$ denotes taking the power set, then this can be expressed as a map 
\begin{equation}\label{eq:RadonMCR}
K \colon \mathcal{P}\Bigl( \bigl(\Real \times (0,\pi) \bigr) \times \mathbb{S}^1 \Bigr) 
     \to \mathcal{P}(\Real^2 \times \mathbb{S}^1),
\quad\text{where}\quad 
K\Bigl(\WF\bigl(\Radon(\signal)\bigr)\Bigr) 
  = \WF(\signal)
\quad\text{for $\signal \in \SchwartzFunc'(\Real^2)$.}
\end{equation}  
In limited-angle tomography we only have access to the Radon transform on an open subset $\datadomain \subset \Real \times (0, \pi)$. The canonical relation then holds for the so-called \emph{visible} wavefront set of a function/distribution $\signal$ that is given by 
\begin{equation}\label{eq:RadonWFvis}
 \WFvis(\signal) \coloneqq \WF(\signal) \cap K(\datadomain). 
\end{equation}  
Following \cite{krishnan2015microlocal}, we next provide a more precise characterisation of the canonical relation in terms of a mapping between wavefront sets in image and sinogram, respectively.

\begin{theorem}\label{thm:microlocalrad}
The canonical relation for the Radon transform $\Radon \colon \SchwartzFunc'(\Real^2)\to \SchwartzFunc'\bigr(\Real\times(0,\pi)\bigl)$ at $\signal \in \SchwartzFunc'(\Real^2)$ can be represented by the mapping 
\[
\Can_{\Radon(\signal)} \colon \WF(\signal) \to \WF\bigl(\Radon(\signal)\bigr)
\]
defined as 
\begin{equation}\label{eq:S1E8}
\Can_{\Radon(\signal)}\bigl(x;\omega(\theta)\bigr) 
\coloneqq 
\Bigl(\bigl(x\cdot\omega^\perp(\theta),\theta+\pi/2\bigr);   
  \omega\Bigl(\arctan\bigl(-x\cdot\omega(\theta)\bigr)\Bigr)
\Bigr) 
\quad \text{for $\bigl(x;\omega(\theta)\bigr)\in\WF(\signal)$,}
\end{equation}
with $\omega(\theta) \coloneqq (\cos\theta,\sin\theta)$ and $\omega(\theta)^\perp \coloneqq (-\sin\theta,\cos\theta)$. 
This means, 
\[
\bigl(x;\omega(\theta)\bigr) \in \WF(\signal)
\iff
\Can_{\Radon(\signal)}\bigl(x;\omega(\theta)\bigr) \in \WF\bigl(\Radon(\signal)\bigr).
\]
\end{theorem}
\begin{proof}
Following {\cite[Definition~7]{krishnan2015microlocal}} a Fourier integral operator, $\mathcal{P}:\SchwartzFunc'(\Rtwo)\rightarrow \SchwartzFunc(\Real\times (0,\pi)$, has the form
\begin{equation}
\label{eq:fioform}
\mathcal{P}f(y)=\int_{\xi\in\Rtwo}\int_{x\in \Rtwo} e^{i\phi(y,x,\xi)}p(y,x,\xi)dxd\xi \quad \text{for $y\in\SchwartzFunc(\Real\times(0,\pi))$},
\end{equation}
where $p,\phi \in C^{\infty}(\Rtwo\times\Real\times(0,\pi)\times\Rtwo)$ are the amplitude and the phase functions, respectively (see {\cite[Definition~4 and Definition~6]{krishnan2015microlocal}}). Using \cite[Definition~7]{krishnan2015microlocal} and duality gives the following expression for the canonical relation of $\mathcal{P}\colon \SchwartzFunc'(\Real^2)\to \SchwartzFunc'\bigl(\Real\times (0,\pi)\bigr)$, given by
\begin{equation}\label{eq:microcanonradon}
\WF\bigl(\Radon(\signal)\bigr)\subset C_\phi\circ\WF(\signal) 
\quad \text{for $\signal\in\SchwartzFunc'(\Real^2)$,}
\end{equation}
where
\[
C_{\phi}\coloneqq 
\Bigl\{ 
  \Bigl((y;\partial_y\phi(y,x,\xi)), \bigl(x,-\partial_x\phi(y,x,\xi)\bigr)\Bigr) :
  (y,x,\xi)\in\Sigma_{\phi}
\Bigr\}
\]
with
\[
\Sigma_{\phi} \coloneqq 
\Bigl\{ 
  \bigl((s,\theta),x,\xi \bigr)\in \bigl(\Real\times [0,2\pi)\bigr)\times \Real^2 \times \Real\setminus\{0\} 
  :
  \partial_{\xi}\phi\bigl((s,\theta),x,\xi\bigr)=0
\Bigr\}.
\]
We now aim to write the Radon transform $\mathcal{A}$ in the form~\eqref{eq:fioform}. By the Fourier slice theorem \cite[Theorem~2]{krishnan2015microlocal}, the Radon transform $\Radon \colon \SchwartzFunc'(\Real^2)\to \SchwartzFunc'\bigr(\Real\times(0,\pi)\bigl)$ can be written in the integral form as
\begin{equation}\label{eq:S1E6}
\Radon(\signal)(s,\theta) = \frac{1}{2\pi} \int_{\xi\in\Real}\int_{x\in\Real^2}e^{i(s-(x\cdot\omega(\theta)))\xi}f(x)dx d\xi
\quad\text{for $s \in \Real$ and $\theta \in (0, \pi)$,}
\end{equation}
where $\omega(\theta) \coloneqq (\cos\theta,\sin\theta)$.
Hence, the Radon transform is a Fourier integral operator with phase function $\phi((s,\theta),x,\xi) \coloneqq \bigl(s-\bigl(x\cdot\omega(\theta)\bigr)\bigr)\xi$ and amplitude $p(y,x,\xi) \coloneqq 1/(2\pi)$. Thus the canonical relation of $\mathcal{A}:\SchwartzFunc'(\Rtwo)\rightarrow \SchwartzFunc'(\Real\times (0,\pi))$ is given by~\eqref{eq:microcanonradon}.

We then derive the exact form of $C_{\phi}$ by computing the derivatives of $\phi$, which are 
\begin{equation}
\label{eq:S1E7}
\begin{aligned}
\partial_x\phi\bigl((s,\theta),x,\xi\bigr) 
  &= -\xi \omega(\theta),
\\
\partial_{(s,\theta)}\phi\bigl((s,\theta),x,\xi\bigr) 
  &= \xi\bigl(1, -x\cdot\omega^{\perp}(\theta)\bigr),
\\
\partial_{\xi}\phi\bigl((s,\theta),x,\xi\bigr) 
  &= \bigl(s-x\cdot\omega(\theta),0\bigr).
\end{aligned}
\end{equation}
Notice that $\partial_x\phi$ and $\partial_{(s,\theta)}\phi$ are not zero for $\xi\neq 0$, which means that the phase function is non-degenerate. This confirms that $\Radon$ is a Fourier integral operator. 

Next, $\Radon$ is of order $m=-1/2$ with an amplitude function $p((s,\theta),x,\xi) = 1/(2\pi)$ that is homogeneous of degree zero, implying that $\Radon$ is elliptic. 
Using the derivatives~\eqref{eq:S1E7}, $\Sigma_{\phi}$ is given by
\[
    \Sigma_{\phi} = 
    \Bigl\{
      \bigl((s,\theta),x,\xi \bigr)\in \bigl(\Real\times [0,2\pi)\bigr)\times \Real^2 \times \Real\setminus\{0\} 
      \colon 
      s-x\cdot\omega(\theta)=0
    \Bigr\}.
\]
Therefore, the canonical relation can be represented by the coordinate mapping
\begin{equation}\label{eq:mappingradon}
\begin{aligned}
\bigl((s,\theta),x,\xi\bigr) 
&\mapsto \Bigl(
    \Bigl(\bigl(x\cdot\omega(\theta),\theta\bigr);
      \partial_{(s,\theta)\phi}\Bigr), 
    (x,-\partial_x\theta)
\Bigr)
\\
&= \Bigl(
     \Bigl(
       \bigl(x\cdot\omega(\theta),\theta\bigr);
       \xi \bigl(1, -x\cdot\omega(\theta)^{\perp} \bigr)
     \Bigr), 
     \bigl(x,\xi \omega(\theta)\bigr)
   \Bigr).
\end{aligned}
\end{equation}
Now, let $\bigl(x;\omega(\theta)\bigr)\in\WF(\signal)$ be an oriented singular point of $\signal$. 
By \eqref{eq:mappingradon} we obtain that 
\[
\bigl(x;\omega(\theta)\bigr)\in\WF(\signal)
\to 
\Bigl(
  \bigr(x\cdot\omega(\theta)^\perp,\theta+\pi/2\bigr); 
  \omega\Bigl(\arctan\bigl(-x\cdot\omega(\theta)\bigr)\Bigr)
\Bigr) 
\in \WF\bigl(\Radon(\signal)\bigr).
\]
Finally, \cite[Theorem~6.3]{quinto2008local} gives 
\begin{equation}\label{eq:microlocalradonmap}
\bigl(x;\omega(\theta)\bigr)\in\WF(\signal)
\iff 
\Bigl(
  \bigl(x\cdot\omega(\theta)^\perp,\theta+\pi/2\bigr); \omega\Bigl(\arctan\bigl(-x\cdot\omega(\theta)\bigr)\Bigr)
\Bigr) 
\in \WF\bigl(\Radon(\signal)\bigr).
\end{equation}
This concludes the proof.
\end{proof}

We next focus on the propagation of singularities performed by the adjoint Fr\'echet derivative of the Radon transform, which is the back-projections operator $\Radon^*$ in Definition~\ref{def:BackProj}.
Using {\cite[Theorem~13]{krishnan2015microlocal}} we know that $\Radon^*$ is also a Fourier integral operator and
\begin{equation}\label{eq:dualradonpseudo}
\Radon^*\Radon(\signal)(x)
  =\int_{\Real^2} e^{ix\cdot\xi}\frac{2}{\|\xi\|}\widehat{\signal}(\xi)d\xi 
  = \frac{1}{\pi}\int_{\Real^2}\int_{\Real^2} 
      e^{i(x-y)\cdot \xi}\frac{1}{\|\xi\|}\signal(y) dy d\xi, 
\quad 
\text{for $\signal \in \SchwartzFunc(\Real^2)$ and $x\in \Real^2$.} 
\end{equation}
In the next proposition, we use \eqref{eq:dualradonpseudo} to introduce the corresponding mapping associated with the canonical relation for $\Radon^*$ in a similar fashion to Theorem~\ref{thm:microlocalrad}.
\begin{proposition}
\label{prop:mircolocalbackproj}
The canonical relation for the back-projection operator $\Radon^* \colon \SchwartzFunc'\bigl(\Real\times (0,\pi)\bigr)\to \SchwartzFunc'(\Real^2)$ in \eqref{eq:BackProj} at $\data \in \SchwartzFunc'\bigl(\Real\times (0,\pi)\bigr)$ can be represented by the mapping 
\[
\Can_{\Radon^*(\data)} \colon \WF(\data) \to \WF\bigl(\Radon^{\ast}(\data)\bigr),
\]
which is defined at $\bigl((s,\theta);\omega(\vartheta)\bigr)\in\WF(\data)$ as  
\begin{equation}\label{eq:microcanonbackproj}
\Can_{\Radon^{\ast}(\data)}\bigl((s,\theta);\omega(\vartheta)\bigr) \coloneqq 
\bigl( 
  (s\cos\theta-\tan\vartheta\sin\theta, s\sin\theta+\tan\vartheta\cos\theta);
  \theta-\pi/2
\bigr), 
\end{equation}
where $\omega(\theta) \coloneqq (\cos\theta,\sin\theta)$. This means, 
\[
\bigl((s,\theta);\omega(\vartheta)\bigr) \in\WF(\data) \iff \Can_{\Radon^* (\data)}\bigl((s,\theta);\omega(\vartheta)\bigr) \in \WF\bigl(\Radon(\signal)\bigr).
\]
\end{proposition}
\begin{proof}
Note first that \eqref{eq:dualradonpseudo} implies that the operator $\Radon^*\Radon \colon \SchwartzFunc(\Real^2)\to \SchwartzFunc(\Real^2)$ is an elliptic pseudodifferential operator with amplitude function 
$p(y,x,\xi)\coloneqq 1/\|\xi\|$.
By duality we can extend this to a mapping $\Radon^*\Radon \colon \SchwartzFunc'(\Real^2)\to \SchwartzFunc'(\Real^2)$.
In addition, the pseudolocal property of pseudodifferential operators (see Theorem~\ref{thm:S1T11}) implies that $\Radon^*\Radon$ will preserve the wavefront set of functions in $\SchwartzFunc(\Real^2)$, i.e., $\WF(\Radon^*\Radon(\signal))=\WF(\signal)$. 
This allows us to represent the canonical relation for the inverse mapping in terms of the canonical relation mapping for $\Radon$. 
Finally, by inverting the mapping implicit in \eqref{eq:microlocalradonmap}, for $\data\in\SchwartzFunc'\bigl(\Real\times (0,\pi)\bigr)$ we obtain that
\begin{equation}
\label{eq:microlocalbackprojmap}
\bigl((s,\theta);\omega(\vartheta)\bigr)\in \WF(\data) 
\iff 
\bigl(
  (s\cos\theta-\tan\vartheta\sin\theta,
   s\sin\theta+\tan\vartheta\cos\theta);
  \theta-\pi/2
\bigr)
\in \WF\bigl(\Radon^{\ast}(\data)\bigr).
\end{equation}
\end{proof}

\subsection{Computational microlocal analysis}\label{sec:discretisation}
Our aim is to develope a computational counterpart to microlocal analysis that is based on defining a notion of a `digital wavefront set' and also to provide computational means for extracting such an object from an array that represents a discretised function. 
This turns out to be theoretical and computationally challenging. 
The definition of a digital wavefront set we will exploit in our work is based on `discretising' Definition~\ref{def:S1D4}. More precisely, the \emph{digitial wavefront set} of an array is defined as follows.
\begin{definition}[Digital wavefront set]
\label{def:digwfset}
Let the array $\boldsymbol{u} \in \Real^{N}$ represent $u \colon \Omega \to \Real$ for some fixed domain $U \subset \Real^{n}$ at sample points $\Omega = \{ x_1,\ldots, x_N \} \subset U$, i.e., $\boldsymbol{u} = (u_1, \ldots, u_N)$, where $u_j \coloneqq u(x_j)$ for $x_j \in \Omega$.
For a fixed set of $M$ uniformly distributed points $\Sigma = \{ \omega_1, \ldots, \omega_M \} \subset \mathbb{S}^{n-1}$ and a set of neighbourhoods $U_{j,k} \ni (x_j,\omega_k)$ where $j \in \{1, \dots, N\}$ and $k \in \{1, \dots, M \}$, such that $\bigcup_{j,k} U_{j,k} \supset U \times \mathbb{S}^{n-1}$, the \emph{digital wavefront set} of the array $\boldsymbol{u}$ is defined as $\vecWF(\boldsymbol{u}) \subset \Omega \times \Sigma$ where 
\begin{equation}\label{eq:discreteWF} 
  (x_j, \omega_k) \in \vecWF(\boldsymbol{u})
  \quad\text{if $U_{j,k}$ intersects $\WF(u) \subset U \times \mathbb{S}^{n-1}$ non-trivially.}
\end{equation}
The \emph{visible digital wavefront set} is defined in a similar manner.
\end{definition}
The definition also applies to arrays $\vecsignal\in \Real^{n_1 \times n_2}$ and $\vecdata\in \Real^{m_1 \times m_2}$ representing images and sinograms in tomographic imaging (Subsection~\ref{subsec:TomoDisc}).
Note finally that Definition~\ref{def:digwfset} suggests a natural way to represent a digital wavefront set as a multi-channel `image'.
Simply assign $M$ binary channels to each sample point in $x_j \in \Omega$ and set the $k$:th channel at that point to $1$ if $(x_j,\omega_k) \in \vecWF(\boldsymbol{u})$, otherwise set the value of that channel to $0$.

Having defined a notion of a digital wavefront set of an array that represents a discretised function, a natural task that follows is to computationally extract such an object. 
This is however not possible to do in a mathematically consistent way as there is \emph{no analytic connection between the digitisation of a function and its digital wavefront set} \cite[Theorem~3.3]{andrade2019wfset}. 
The approach taken in \cite{andrade2019wfset} is therefore to view the task of extracting the digital wavefront as a statistical estimation problem. 
More precisely, extracting the digital wavefront set is phrased as computing the probability distribution of possible digital wavefront sets.
A \emph{single} digital wavefront set can then be computed by choosing at each point of the digitised function the wavefront set orientation with the highest probability. 
Much of \cite{andrade2019wfset} is devoted to developing a deep learning based approach for this task.
A key part is the development of DeNSE, which is a DNN with an architecture specifically suited for computing probability distributions of digital wavefront sets of functions represented by their discrete shearlet transforms. 
DeNSE was in \cite{andrade2019wfset} successfully applied to extract digital wavefront sets of functions discretised by shearlets in both image and sinogram space in tomographic imaging.


\section{Microlocal analysis of the Learned Primal-Dual network}
\label{sec:ResNetMicroLocal}
This section aims to derive the canonical relation for the non-linear operator given by the Learned Primal-Dual network defined in \eqref{eq:LPDiterates} with $\ForwardOp$ as the Radon transform (Definition~\ref{def:radontrans}).
The relation allows to describe how the Learned Primal-Dual network transforms the digital wavefront set of an input array that is a discretisation of a  function/distribution representing data, which
is a key step in our approach to tomographic image reconstruction, outlined in Subsection~\ref{subsec:introresults}.

\subsection{Overview of approach}\label{subsec:MicroLocalLPDOverview}
The starting point is the discretised tomographic inverse problem given in Subsection~\ref{subsec:TomoDisc}, where data and images are arrays in $\Real^{m_1 \times m_2}$ and $\Real^{m_1 \times m_2}$, respectively.
Assume next that the digital wavefront set is known for an input array $\vecdata \in \Real^{m_1 \times m_2}$, representing measured data. 
Our aim is to compute the visible digital wavefront set of the array $\vecRecOp_{\theta}(\vecdata) \in \Real^{n_1 \times n_2}$ representing the reconstructed image, i.e., to compute $\vecWF^{\visible}(\vecRecOp_{\theta}(\vecdata))$ where $\vecRecOp_{\theta} \colon \Real^{m_1 \times m_2} \to \Real^{n_1 \times n_2}$ is the non-linear Learned Primal-Dual network in  \eqref{eq:LPDiterates} for tomographic reconstruction that was first introduced in \cite{adler2018lpd}.
This is a DNN that is made up of stacked convolutional residual neural networks of the form \eqref{eq:LPDResNets} that in our setting are given as in Definition~\ref{def:ResNetDiscrete} with $\ForwardOp = \vecRadon \colon \Real^{n_1 \times n_2} \to \Real^{m_1 \times m_2}$ denoting the discretised Radon transform in \eqref{eq:DiscRadon}.

One approach to derive the mapping between $\vecdata$ and $\vecWF^{\visible}(\vecRecOp_{\theta}(\vecdata))$ is to work entirely within the discrete setting. 
Such an approach would require us to describe how the discrete Radon transform along with its adjoint transforms the digital wavefront set of an array.
An alternative approach is to utilise the rich and well-developed microlocal theory for the (continuum) Radon transform outlined in Subsection~\ref{subsec:wavefrontset}.
In particular, this theory describes how the Radon transform, and hence also its adjoint (back-projection), modifies visible wavefront sets in limited-angle tomography.
However, such an attempt at leveraging on the continuum theory requires us to formulate a continuum version of the Learned Primal-Dual network.

In Subsection~\ref{subsec:CLDP}, we derive the non-linear operator, which is a continuum version of the Learned Primal-Dual network.
This results from replacing every step used in the construction of $\vecRecOp_{\theta} \colon \Real^{m_1 \times m_2} \to \Real^{n_1 \times n_2}$ with a natural corresponding continuum version, resulting in an operator $\RecOp_{\theta} \colon \DataSpace \to \RecSpace$, where $\DataSpace$ and $\RecSpace$ are not necessarily finite dimensional vector spaces.
The key part is to assemble appropriate continuum versions of the convolutional residual neural networks in \eqref{eq:LPDResNets}.
While a continuous convolution seems to be the most natural correspondence to a discrete convolution, this begs the question, which continuous filter to choose. 
Alternatively, following \cite{ruthotto2018pdes} one can view each discrete convolution as a discretisation of a corresponding fourth-order differential operator with coefficients that relate precisely to the discrete filter. 
Naturally, the coordinate-wise application of an activation function to a discrete input corresponds to the composition with that activation function in the continuous realm. 
Since the differential operators do not necessarily yield $\Lp^2$ functions but only distributions, this concept needs to be extended to tempered  distributions (see Definition~\ref{def:ReLUonDistributions}). 
Finally, the residual block consists only of summation, which naturally translates to summation of continuous inputs. 

Subsection~\ref{subsec:CanonRelCLPD} shows how each of the above operations affect the wavefront set of a function or tempered distribution.
A key result is Theorem~\ref{thm:microcanonrelu} that analyses the action of $\ReLU$.
When combined, these provide a precise theoretical description of the way a continuum version of the Learned Primal-Dual network transforms the wavefront set.
Its canonical relation can then be used to derive the associated digital canonical relation for the Learned Primal-Dual network in \cite{adler2018lpd}.

\subsection{Continuum Learned Primal-Dual network}\label{subsec:CLDP}
Towards the end of Subsection~\ref{subsec:LPDArchitecture}, we specified the architecture for the Learned Primal-Dual network that was introduced in \cite{adler2018lpd} for tomographic reconstruction.
This is a mapping
\[\vecRecOp_{\theta} \colon \Real^{m_1 \times m_2} \to \Real^{n_1 \times n_2} \]
given by a DNN that acts on arrays in finite dimensional vector spaces. In inverse problems these spaces typically represent discretised functions as outlined in  Subsection~\ref{sec:discretisation}. 
The aim here is to formulate a natural continuum version of $\RecOp_{\theta} \colon \DataSpace \to \RecSpace$ that acts on functions or distributions that are not necessarily discretised. 

As outlined in the overview in Subsection~\ref{subsec:MicroLocalLPDOverview}, a key step lies in appropriately extending the residual convolutional networks in \eqref{eq:LPDiterates}.
To this end, we replace every discrete operation of the Learned Primal-Dual by a continuum analogue, i.e., an operator that takes as an input a distribution. In principle, only four types of operations happen in the definition of the Learned Primal-Dual. These are the convolutions, the application of a $\ReLU$, the application of a discretised Radon transform or the adjoint of the Fr\'echet derivative of the Radon transform, and taking sums of functions, either between channels in the convolution or due to residual connections. The discretised Radon transform as well as the back-projection, which is the adjoint of its Fr\'echet derivative, already stem from continuum counterparts which are the natural replacement. 
Furthermore, a sum of discrete images naturally corresponds to sums of continuum images. However, it is not immediately clear how to extend a discrete convolution to a continuum convolution. 
For example, replacing discrete convolutions by convolutions defined over $\Real^2$ begs the question how the continuum convolution kernel should be chosen. We will see in the subsection below, that there exists a natural replacement of the convolution operator, which however, necessitates that we work with distributions instead of $L^2$ functions. Because of this, also the application of the $\ReLU$ will need to be generalised to distributions. 
In the following two subsection we discuss the continuum counterparts of convolution and $\ReLU$. Thereafter, we present the continuum Learned Primal-Dual network.

\subsubsection{From discrete convolutions to differential operators}
To find an appropriate continuum counterpart to the discrete convolution step in \eqref{eq:discconvresent}, we interpret the discrete convolution as a discretisation of a differential operator. This approach is inspired by \cite{ruthotto2018pdes}.

Consider a continuum image $f\in \Lp^2(\Real^2)$ with discretisation $\vecsignal\in \Real^{n_1\times n_2}$. Concretely, let 
\[ 
\vecsignal = 
\begin{pmatrix} 
f_{11} & \dots  & f_{1 n_1}\\
\vdots & \ddots & \vdots \\
f_{n_2 1} & \dots  & f_{n_1 n_2}
\end{pmatrix}.
\]
Furthermore, let $\vecoperator{K}_{\boldsymbol{\theta}}$ be a $3\times 3$ (discretised)  convolutional operator parametrised by the filter:
\begin{equation}
\label{eq:S2E1}
\boldsymbol{\theta} = 
\begin{pmatrix} 
\theta_{11} & \theta_{12} & \theta_{13} \\
\theta_{21} & \theta_{22} & \theta_{23}\\
\theta_{31} & \theta_{32} & \theta_{33} 
\end{pmatrix},
\quad\text{where $\theta_{ij}\in \Real$.} 
\end{equation}
Note that $(\Delta_{ij})_{i,j = 1}^3\subset\Real^{3\times 3}$ such that the following forms a basis for $\Real^{3\times 3}$:
\begin{equation}
\label{eq:S2E4}
\begin{aligned}
\boldsymbol{\Delta}_{11}  
&= \begin{pmatrix} 
0 & 0 & 0 \\
0 & 1 & 0\\
0 & 0 & 0 
\end{pmatrix},
& \quad \boldsymbol{\Delta}_{12} 
&= \begin{pmatrix} 
0 & 1 & 0 \\
0 & 0 & 0\\
0 & -1 & 0 
\end{pmatrix},
& \quad 
\boldsymbol{\Delta}_{13}
&= \begin{pmatrix} 
0 & -1 & 0 \\
0 & 2 & 0\\
0 & -1 & 0 
\end{pmatrix},
\\[0.5em] 
\boldsymbol{\Delta}_{21} 
&= \begin{pmatrix} 
0 & 0 & 0 \\
1 & 0 & -1\\
0 & 0 & 0 
\end{pmatrix},
& \quad \boldsymbol{\Delta}_{22} 
&= \begin{pmatrix} 
1 & 0 & -1 \\
0 & 0 & 0\\
-1 & 0 & 1 
\end{pmatrix},
& \quad \boldsymbol{\Delta}_{23} 
&= \begin{pmatrix} 
1 & -2 & 1 \\
0 & 0 & 0\\
-1 & 2 & -1 
\end{pmatrix},
\\[0.5em] 
\boldsymbol{\Delta}_{31} 
&= \begin{pmatrix}
0 & 0 & 0 \\
1 & -2 & 1\\
0 & 0 & 0 
\end{pmatrix},
& \quad \boldsymbol{\Delta}_{32} 
&= \begin{pmatrix} 
1 & 0 & -1 \\
-2 & 0 & 2\\
1 & 0 & -1 
\end{pmatrix},
& \quad \boldsymbol{\Delta}_{33} 
&= \begin{pmatrix}
-1 & 2 & -1 \\
2 & -4 & 2\\
-1 & 2 & -1 
\end{pmatrix}.
\end{aligned}
\end{equation}
Hence, we can, for a given $h >0$, express $\boldsymbol{\theta}$ as
\begin{equation}
\label{eq:S2E3}
    \begin{split}
    \boldsymbol{\theta}
      &= \beta_{11}\boldsymbol{\Delta}_{11}
      +\frac{\beta_{12}}{2h}\boldsymbol{\Delta}_{12}
      +\frac{\beta_{21}}{2h}\boldsymbol{\Delta}_{21}
      +\frac{\beta_{22}}{4h^2}\boldsymbol{\Delta}_{22}
      +\frac{\beta_{13}}{h^2}\boldsymbol{\Delta}_{13}
      \\ & \phantom{=}
      +\frac{\beta_{31}}{h^2}\boldsymbol{\Delta}_{31}
      +\frac{\beta_{32}}{2h^3}\boldsymbol{\Delta}_{32}
      +\frac{\beta_{23}}{2h^3}\boldsymbol{\Delta}_{23}
      +\frac{\beta_{33}}{h^4}\boldsymbol{\Delta}_{33}.
\end{split}
\end{equation}
Note that the $3 \times 3$ matrices $\Delta_{ij}$ can be seen as the finite difference discretisations of the partial derivatives of $\signal$ if $h$ corresponds to the distance between the sampling density underlying the discretisation $\vecsignal$. For a smooth function $\signal$, we therefore observe that if the discretisation $h$ goes to zero, then  
\begin{align*}
    \boldsymbol{\theta} * \vecsignal (i,j) \to &  \left(\beta_{11} f 
    +\beta_{12}\partial_{2}f
    +\beta_{21}\partial_{1}f
    +\beta_{22}\partial_{1}\partial_{2}f
    +\beta_{13}\partial_{2}^2f \right.\\
    &\qquad +\left.\beta_{31}\partial_{1}^2f 
    +\beta_{23}\partial_{2}^2\partial_{1} f
    +\beta_{32}\partial_{1}^2\partial_{2}f
    +\beta_{33}\partial_{1}^2\partial_{2}^2f\right)((x_i,y_j)), 
\end{align*}
where $(x_i,y_j)_{i, j=1}^N$ are the discretisation points. For an open set $\Omega \subset \Real^2$, this yields the following operator $\operator{K}_{\boldsymbol{\theta}}$ defined on $\SchwartzFunc(\Omega)$:
\begin{equation}
\label{eq:S2E5}
\begin{aligned}
\mathcal{K}_{\boldsymbol{\theta}}(\signal) :=& \,
\beta_{11} f 
+\beta_{12}\partial_{2}f
+\beta_{21}\partial_{1}f
+\beta_{22}\partial_{1}\partial_{2}f
+\beta_{13}\partial_{2}^2f\\
&+\beta_{31}\partial_{1}^2f 
+\beta_{23}\partial_{2}^2\partial_{1} f
+\beta_{32}\partial_{1}^2\partial_{2}f
+\beta_{33}\partial_{1}^2\partial_{2}^2f. 
\end{aligned}
\end{equation}
By duality, we can extend $\operator{K}_{\boldsymbol{\theta}}$ to $\SchwartzFunc'(\Omega)$. 
Note also that $\operator{K}_{\boldsymbol{\theta}}$ is a linear second-order differential operator. In particular, it is a pseudodifferential operator with its symbol given by
\begin{equation}
\label{eq:S2E6}
\begin{aligned}
    p_{\boldsymbol{\theta}}(\xi) =& \,
    \beta_{11} 
    +\beta_{12}\xi_2
    +\beta_{21}\xi_1
    +\beta_{22}\xi_1\xi_2\\
    &+\beta_{13}\xi_2^2
    +\beta_{31}\xi_{1}^2
    +\beta_{23}\xi_{2}^2\xi_{1}
    +\beta_{32}\xi_{1}^2\xi_{2}
    +\beta_{33}\xi_1^2\xi_{2}^2 \text{ for } \xi \in \Omega. 
\end{aligned}
\end{equation}
The interpretation above of a discrete convolutional operator that takes non-smooth images as inputs necessitates to have a continuum definition that acts on distributions. Consequently, also all further operations will need to be applicable to distributions.

\subsubsection{Pointwise application of ReLU to distributions}\label{subsec:PointwiseReLU} 
The rectified linear unit ($\ReLU$) is an activation function used in many neural network architectures. 
It is defined as the positive part of its argument, i.e., $\ReLU \colon \Real \to \Real$ is given as $\ReLU(x) \coloneqq \max\{x,0\}$.
Our aim is to extend the $\ReLU$ to an operator that acts on tempered distributions on $\domain \subset \Real^n$, denoted as $\ReLUOp$.

We start by rewriting the ReLU function in terms of the Heaviside function $\Heav \colon \Real\to\Real$: 
\begin{align}\label{eq:MultiplicationHeavDefinition}
    \ReLU(x) =  \Heav(x) x,
    \quad\text{where}\quad
    \Heav(x) \coloneqq 
\begin{cases}
1, & \text{if $x > 0$,} \\
0, & \text{if $x\leq 0$.}
\end{cases}
\end{align}
The above can be used to extend ReLU in a straightforward manner to $f \in \Cont^{\infty}(\domain)$, by simply defining
\begin{equation}\label{eq:MultiplicationReLUDefinition}
  \ReLUOp(\signal)(x) \coloneqq \ReLU\bigl( f(x) \bigr) =  \Heav\bigl( f(x) \bigr) f(x) 
    = \begin{cases}
         f(x), & \text{if $f(x) > 0$,} \\
         0, & \text{if $f(x) \leq 0$.}
       \end{cases}
\end{equation}
We only know that $\ReLUOp \colon \SchwartzFunc(\domain) \to \Lp^{\infty}(\domain)$ and in fact $\ReLUOp(\signal)$ may not be smooth for $\signal \in \SchwartzFunc(\domain)$. Hence, $\ReLUOp$ does not map $\SchwartzFunc(\domain)$ to $\SchwartzFunc(\domain)$, i.e., we cannot use duality to define ReLU on distributions. Using the characterisation in \eqref{eq:MultiplicationReLUDefinition} to extend ReLU to distributions involves extending the Heaviside function to tempered distributions and also ensuring the subsequent multiplication is well-defined. 

We start by defining the Heaviside function of a distribution $\signal \in \SchwartzFunc'(\domain)$ as the characteristic function of its positive support, i.e., we define the Heaviside operator $\HeavOp \colon \SchwartzFunc'(\domain)\to \Lp^\infty(\domain)$ as
\begin{equation}\label{eq:HeavOp}
    \HeavOp(\signal) \coloneqq \mathds{1}_{\supp_+(\signal)}, \quad\text{for $f\in \SchwartzFunc'(\domain)$,}
\end{equation}
where $\supp_+(\signal) \subset \domain$ is the positive support of $\signal$ (Definition~\ref{def:S3D1}) and $\mathds{1}_{\supp_+(\signal)}$ denotes the characteristic function of $\supp_+(\signal)$. 
Before proceeding, we list desirable properties for an extension of ReLU to distributions. More precisely, $\ReLUOp \colon \SchwartzFunc'(\domain) \to \SchwartzFunc'(\domain)$ should preferably come with the following properties:
\begin{enumerate}
    \item $\esssupp \ReLUOp(\signal) \subset \supp_{+}(\signal)$ (the  essential support $\esssupp$ is defined in Definition~\ref{def:S3D1}),
    \item $\ReLUOp(\signal)(\phi) = f(\phi)$ for all test functions $\phi \in \SchwartzFunc(\domain)$ supported in $\supp_{+}(h)$,
    \item $\ReLUOp(\signal) = \HeavOp(\signal)\, f$ whenever $\signal \in \SchwartzFunc'(\domain)$.
\end{enumerate}

Having extended the Heaviside function to distributions as in \eqref{eq:HeavOp}, our main concern is to ensure that the multiplication between the distribution $\signal \in \SchwartzFunc'(\domain)$ and $\HeavOp(\signal) \in \Lp^{\infty}(\domain)$ is well-defined.
By Definition~\ref{def:S1D3} and Theorem~\ref{thm:S1T5}, this is indeed the case whenever $(x, -\lambda) \not \in \WF(\signal)$ for all $(x, \lambda) \in \WF\bigl(\HeavOp(\signal)\bigr)$, i.e., we can define $\ReLUOp(\signal)$ by \eqref{eq:MultiplicationReLUDefinition} for any $\signal \in \SchwartzFunc'(\domain)$ that satisfies this criteria.
However, the multiplication is not necessarily well-defined whenever there exists $(x,\lambda) \in \WF\bigl(\HeavOp(\signal)\bigr)$ such that $(x, -\lambda) \in \WF(\signal)$. 
Thus, all we know is that the multiplication of $\HeavOp(\signal)$ and $\signal$ is always defined if $\signal \in \Lp^{2}_{\mathrm{loc}}(\domain)$ (Remark~\ref{rem:whatever}).

One idea is therefore to locally dampen $\signal$ close to points where we cannot define the multiplication of $\signal$ with $\HeavOp(\signal)$. 
This leads to the following definition of ReLU on distributions. 
\begin{definition}\label{def:ReLUonDistributions}
Let $\Omega \subset \Real^2$ be open, $\kappa >0$, and $\phi_\kappa \in \SchwartzFunc(\Real^2)$ be a function that integrates to $1$, is positive and is supported on a compact subset of $B_\kappa(0)$. 
Then define
\begin{equation}\label{eq:ReLUDistr}
    \ReLUOp_{\kappa, \phi_\kappa}(\signal) \coloneqq \HeavOp(\signal) \signal^s, 
    \quad\text{for $\signal \in \SchwartzFunc'(\domain)$,}
\end{equation}
where $\signal^s \coloneqq (1-\theta_\kappa)f$. Here $\theta_\kappa \coloneqq \mathds{1}_{X} \ast \phi_\kappa$ with 
\begin{align*}
    X \coloneqq \Bigl\{ x \in \Real^2 \setminus \supp_{\Lp^2}(\signal) \colon 
    (x, \lambda) \in \WF(H(h)), (x, -\lambda) \in \WF(h) \text{ for a } \lambda \in \mathbb{S}^1 \Bigr\} + B_{\kappa}(0).
\end{align*}
In the above, $\supp_{\Lp^2}(\signal) \subset \domain$ denotes the $\Lp^2$-support of $\signal$ (Definition~\ref{def:L2supp}).
\end{definition}

We next show that Definition~\ref{def:ReLUonDistributions} can be used to extend the ReLU function to distributions.
\begin{proposition}
Consider $\ReLUOp_{\kappa, \phi_\kappa}$ in \eqref{eq:ReLUDistr} for some $\kappa >0$ and let $\phi_\kappa \in \SchwartzFunc(\Real^2)$ that integrates to $1$, is positive, and is supported on a compact subset of $B_\kappa(0)$.
Then $\ReLUOp_{\kappa, \phi_\kappa} \colon \SchwartzFunc'(\domain) \to \SchwartzFunc'(\domain)$ for an open domain $\Omega \subset \Real^2$.
\end{proposition}
\begin{proof}
We need to show that $\ReLUOp_{\kappa, \phi_\kappa}(\signal) \in \SchwartzFunc'(\domain)$, whenever $\signal \in \SchwartzFunc'(\domain)$.
To see this, note first that $1-\theta_\kappa$ is smooth and vanishes on a neighbourhood of every $x \in  \domain \setminus \supp_{\Lp^2}(\signal)$, where 
\[(x, \lambda) \in \WF\bigl(\HeavOp(\signal)\bigr) 
   \quad\text{and}\quad
   (x, -\lambda) \in \WF(\signal)
   \quad\text{for some $\lambda \in \mathbb{S}^1$.} 
\]
Hence, the product $(1-\theta_\kappa) f$ is well-defined and by Theorem~\ref{thm:S1T5}, there does not exist an $x \in \domain \setminus \supp_{\Lp^2}(\signal)$ such that 
\[(x, \lambda) \in \WF\bigl(\HeavOp(\signal)\bigr) 
   \quad\text{and}\quad
  (x, -\lambda) \in \WF((1-\theta_\kappa) f).
\]   
Theorem~\ref{thm:S1T5} and Remark~\ref{rem:whatever} now imply that $\ReLUOp_{\kappa, \phi_\kappa}(\signal) \in \SchwartzFunc'(\domain)$ whenever $\signal \in \SchwartzFunc'(\domain)$, which concludes the proof.
\end{proof}

\begin{remark}
The set $X$ in Definition~\ref{def:ReLUonDistributions} is a neighbourhood of the set on which the definition of $\ReLU(h)$ via the multiplication $\HeavOp(h)h$ is not well defined. To understand the nature of this set, we consider three examples:
\begin{enumerate}
\item $\signal\in \SchwartzFunc'(\domain)$. Then $X = \emptyset$, and hence $\ReLUOp_{\kappa, \phi_\kappa}(\signal) = \ReLU(\signal)$. In particular, if $\signal= \HeavOp(h)$ for some $h \in \SchwartzFunc'(\domain)$, then $\ReLUOp_{\kappa, \phi_\kappa}(\signal) = \ReLU(\signal) = \signal$.
\item $\signal= P( \mathds{1}_B)$ for some domain $B \subset \domain$ and $P$ is an elliptic linear differential operator of order at least one.
  Then $\esssupp(\signal) \subset \partial B$, so $\HeavOp(\signal) = 0$ which in turn implies that $X = \emptyset$ and 
  $\ReLU_{\kappa, \phi_\kappa}(\signal) = 0$.
\item $\signal = P( \mathds{1}_B + h)$ for some domain $B \subset \domain$ and $P$ is an elliptic linear differential operator.
  Assume furthermore that $h \in \Cont^\infty(\domain)$ is such that $P(h)$ is positive on $B$.
  Then $X = \partial B + B_\kappa(0)$, since $\signal$ is not a function at $\partial B$ and $\HeavOp(\signal) = \mathds{1}_B$. Thus $\WF\bigl(\Heav(\signal)\bigr) = \WF(\signal)$.    
\end{enumerate}
\end{remark}

We conclude by pointing out that $\ReLUOp_{\kappa, \phi_\kappa} \colon \SchwartzFunc'(\domain) \to \SchwartzFunc'(\domain)$ in Definition~\ref{def:ReLUonDistributions} fulfils almost all of the criteria stipulated earlier for an extension of ReLU to distributions. 
It satisfies the first and third criterium. Moreover, $\ReLUOp_{\kappa, \phi_\kappa}(\signal)(\phi) = f(\phi)$ holds for all $\phi \in \SchwartzFunc(\domain)$ with a support that has a distance of more than $2 \kappa$ from $\WF\bigl(\HeavOp(\signal)\bigr) \subset \partial \supp_{+}(\signal)$. 

\subsubsection{The continuum Learned Primal-Dual network}\label{subsec:thecontinuumOperator}
To define the continuum Learned Primal-Dual network, we start by introducing a continuum ResNet.
\begin{definition}[Continuum two-dimensional convolutional ResNet]
\label{def:ResNetContinuous}
Let $\Omega \subset \Real^2$ be open and let $N\in \Natural$, $j\in\{1,2,3,4\}$, where $n_4 = 1$ be the \emph{numbers of channels per layer}. Further, for $j = 1, \dots, 4$, let  $\boldsymbol{\theta}_j\coloneqq (\boldsymbol{\theta}_j^{l,k})^{n_{j-1},n_j}_{l=1,k=1}\subset (\Real^{3 \times 3})^{n_{j-1} \times n_{j-1}}$ be a \emph{set of coefficients}. Let $\kappa >0$ and let $\phi_\kappa \in \SchwartzFunc(\Omega)$ be a function that integrates to 1, is positive and is supported on a compact subset of $B_\kappa(0)$.

We define the \emph{continuum convolutional affine operator} $W_{\boldsymbol{\theta}_j}^c: (\SchwartzFunc'(\Omega))^{n_{j-1}} \to (\SchwartzFunc'(\Omega))^{n_j}$ as 
\begin{equation}
\label{eq:convresent}
    W_{\boldsymbol{\theta}_j}(\signal)_k = \sum_{l = 1}^{n_{j-1}} \mathcal{K}_{\boldsymbol{\theta}_j^{l, k}}(\signal) \quad \text{ for } k\in \{1,\ldots, n_j\}  \text{ and } f\in (\SchwartzFunc'(\Omega))^{n_{j-1}}.
\end{equation}
The \emph{continuum ResNet operator} $\ResNet_{\kappa, \phi_\kappa} \colon  (\SchwartzFunc'(\Omega))^{n_0} \to \SchwartzFunc'(\Omega)$ is then given by 
\[
    \ResNet_{\kappa, \phi_\kappa}(\signal_1, \dots, \signal_{{n_0}})= {f_1} +\mathcal{G}({f_1}, \dots , {f_{n_0}}) \quad \text{ for } {f_1},\dots, {f_{n_0}}\in \SchwartzFunc'(\Omega),
\]
where $\mathcal{G}: (\SchwartzFunc'(\Omega))^{n_0} \to \SchwartzFunc'(\Omega)$ is the operator
\[
    \mathcal{G}(f_1, \dots, f_{n_0}) = \bigl(W_{\boldsymbol{\theta}_4} \circ \mathrm{ReLU}_{\kappa, \phi_\kappa}
  \circ 
  \mathcal{W}_{\boldsymbol{\theta}_3} \circ \mathrm{ReLU}_{\kappa, \phi_\kappa}
  \circ 
  \mathcal{W}_{\boldsymbol{\theta}_2} \circ \mathrm{ReLU}_{\kappa, \phi_\kappa}
  \circ 
  \mathcal{W}_{\boldsymbol{\theta}_1}
  \bigr)({f_1}, \dots, {f_{n_0}})
\]
for ${f_1}, \dots, {f_{n_0}}\in \SchwartzFunc'(\Omega)$. 
\end{definition}
\begin{remark}
\begin{enumerate}
    \item Slightly deviating from Definition~\ref{def:ResNetDiscrete}, we do not include a bias term in the definition of the continuum ResNet above, since no such term will appear in our implementation in Subsection~\ref{sec:numresults}. 
    \item Note that besides the previously defined operators $\mathcal{K}_{\boldsymbol{\theta}_j^{l, k}}$ and $\ReLU_{\kappa, \phi_\kappa}$, only addition is applied in the continuum ResNet. Since the set of distributions is a linear space, we conclude that $\ResNet_{\kappa, \phi_\kappa}$ is a well-defined operator from $\SchwartzFunc'(\Omega)$ to $\SchwartzFunc'(\Omega)$.
\end{enumerate}
\end{remark}

Based on the definition above for the continuum ResNet, we can now define the continuum Learned Primal-Dual network as in \eqref{eq:LPDiterates} with operators in \eqref{eq:LPDResNets} given as continuum ResNets.
Hence, continuum Learned Primal-Dual network is a mapping 
\[
\RecOp_{\theta} \colon \SchwartzFunc'(\datadomain) \to 
 \SchwartzFunc'(\domain)
\quad\text{where}\quad
\RecOp_{\theta}(\data) \coloneqq \signal_N
\]
with $\signal_N \in \SchwartzFunc'(\domain)$ given by the $N$-step iterative scheme in Algorithm~\ref{alg:contLearnedPrimalDual} in which  $\Lambda_i$ and $\Gamma_i$ are continuum two-dimensional convolutional ResNets as in Definition~\ref{def:ResNetContinuous}.
\begin{center}
\begin{algorithm}[H]
\SetAlgoLined
\KwIn{$f_0\in \SchwartzFunc'(\Real^2)$, $h_0\in \SchwartzFunc'(\datadomain)$ and $\data\in\SchwartzFunc'(\datadomain)$.}
\KwOut{Primal solution $f_N\in \SchwartzFunc'(\Real^2)$ and dual solution $h_N\in \SchwartzFunc'(\datadomain)$.}
\For{$i=1, \ldots, N-1$}{
$h_{i+1}\longleftarrow \Gamma_i (h_i,\Radon(f_i),g)$;\\
 $f_{i+1}\longleftarrow \Lambda_i (f_i,[\partial\!\Radon(f_i)]^*(h_{i+1}))$;
}
\caption{Continuum Learned Primal-Dual network}
\label{alg:contLearnedPrimalDual}
\end{algorithm}
\end{center}

\subsection{Canonical relation for the continuum Learned Primal-Dual network}\label{subsec:CanonRelCLPD}
It is clear that we can describe the canonical relation for the continuum Learned Primal-Dual network, if we can identify the one of the continuum ResNets $\Lambda_i$ and $\Gamma_i$ for $i = 1, \dots, I$. In addition, we also need to combine these canonical relations with the canonical relations for the Radon transform, and the relations for the Fr\'echet derivative of the adjoint of the Radon transform.

\subsubsection{Differential operator}\label{subsec:convolution}
The canonical relation for a differential operator is typically very straight-forward to compute, if this operator is an (elliptic) pseudodifferential operator. In this respect, we recall the following result.

\begin{theorem}[Pseudolocal property, {\cite[Theorem~14]{krishnan2015microlocal}}]
\label{thm:S1T11}
A pseudodifferential operator $\mathcal{P}$ satisfies the pseudolocal property, i.e., 
\[
\singsupp(\mathcal{P}f)\subset\singsupp(\signal) 
\text{ and } 
\WF(\mathcal{P}f)\subset\WF(\signal)
\quad\text{for all $f\in \SchwartzFunc'(\Real^2)$.}
\]
If $\mathcal{P}$ is elliptic, then we have equality instead of inclusion, i.e.,  
\[
\singsupp (\mathcal{P}f) = \singsupp(\signal) 
\text{ and } 
\WF(\mathcal{P}f)=\WF(\signal)
\quad\text{for all $f\in \SchwartzFunc'(\Real^2)$.}
\]
\end{theorem}
Since $\mathcal{K}_{\boldsymbol{\theta}}$ defined in~\eqref{eq:S2E5} is a pseudodifferential operator, we obtain that
\begin{equation}
\label{eq:microcanonconv}
\WF(\mathcal{K}_{\boldsymbol{\theta}} f) \subset \WF(\signal).
\end{equation}
This means that $\mathcal{K}_{\boldsymbol{\theta}}$ does not introduce new singularities to $\signal$, and might even delete some of them; in the case that the coefficients $\beta_{ij}$ are such that
\[
    0 < |p_{\boldsymbol{\theta}}(\xi)| \text{ for all } \|\xi\| \neq 0, 
\]
the operator $\mathcal{K}_{\boldsymbol{\theta}}$ is elliptic and preserves the singularities, i.e., $\WF(\mathcal{K}_{\boldsymbol{\theta}} f) = \WF(\signal)$. Here $p_{\boldsymbol{\theta}}$ is the symbol defined in \eqref{eq:S2E6}.

\subsubsection{ReLU application}

Since for $h \in \SchwartzFunc'(\Real^2)$ the distribution $\ReLU_{\kappa, \phi_\kappa}(h)$ is defined in most parts of the domain as $H(h)h$, we can study its wavefront set using Theorem~\ref{thm:S1T5}. We now first study the wavefront set of $H(h)$. Afterward, we estimate the wavefront set of $\ReLU_{\kappa, \phi_\kappa}(h)$ in Subsection~\ref{subsec:ResNetMicroLocalS2S2}.

\subsubsection{The wavefront set of \texorpdfstring{$\HeavOp(\signal)$}{Hf}}
\label{subsec:ResNetMicroLocalS2S1}

For a function $g \in \Lp^2(\Omega)$, the wavefront set of $\HeavOp(g)$ is determined through the following factors: A point $x\in\Omega$, such that on a neighbourhood thereof $g$ is almost always positive will be mapped to something constant by the Heaviside function. 
Since constant functions are smooth, this operation erases the wavefront set associated with a neighbourhood of $x$. The same argument can be made on neighbourhoods where $g$ is almost everywhere negative. Points $x'\in \Omega$ in which $g$ vanishes have the potential to create new discontinuities, since the Heaviside function has a jump in $0$. 
If $g$ is smooth in $x'$ and also has non-vanishing gradient, then the implicit function theorem tells us the form of the discontinuity of $\HeavOp(g)$. 
We will see below in Proposition~\ref{prop:WavefrontSetViaImplicitFunctionTheorem} that the same argument can be made for tempered distributions. A crucial ingredient for that result will be the following estimation of the wavefront set of indicator functions. 

\begin{proposition}[{\cite[Proposition 20]{brouder2014smoothwf}}]\label{prop:WavefrontSetOfPiecewiseConstants}
Let $n\in \Natural$ and $B \subset \Real^n$ be a domain with smooth boundary. Then, $\WF(\mathds{1}_{B}) = \{(x,\lambda)\in \Real^n \times \mathbb{S}^1 \colon x \in \partial B , \lambda \text{ normal to } \partial B \text{ in } x\}$.
\end{proposition}

We can now state the following result describing the wavefront set of $\HeavOp(g)$.
\begin{proposition}\label{prop:WavefrontSetViaImplicitFunctionTheorem}
Let $g \in \SchwartzFunc'(\Omega)$ for $\Omega \subset \Real^2$ an open domain. Let further 
\begin{align*}
    R_g &\coloneqq \{x \in \Real^2 \colon x \in \partial(\supp_+(g)), x \not \in \singsupp(g), \nabla g(x)\neq 0\},\\
    C_g &\coloneqq \{x \in \Real^2 \colon x \in \partial(\supp_+(g)), x \not \in \singsupp(g), \nabla g(x) = 0\},\\
    S_g &\coloneqq \{x \in \Real^2 \colon x \in \partial(\supp_+(g)), x \in \singsupp(g)\}.
\end{align*}
Then, $(x, \lambda) \in \WF\bigl(\HeavOp(g)\bigr)$, if, for an $\alpha \neq 0$,
\begin{align}\label{eq:theconditionOntheLevelSets}
    x \in R_g \text{ and }\lambda = \pm \nabla_x(g)/\|\nabla_x(g)\|.
\end{align}
Moreover,  $(x, \lambda) \in \WF\bigl(\HeavOp(g)\bigr)$ only if \eqref{eq:theconditionOntheLevelSets} holds or $x \in C_g \cup S_g$.
\end{proposition}
\begin{proof}
We start with the "only if" part. The statement is clear if $\WF\bigl(\HeavOp(g)\bigr) = \emptyset$. Otherwise, let $(x, \lambda) \in \WF\bigl(\HeavOp(g)\bigr)$. 

Assume first that $x \in \partial(\supp_+(g))^c$. Then either $x \in \supp_{-,0}(g)$ or $x \in \supp_+(\signal)^\circ$. Since both $\supp_{-,0}(g)$ and $\supp_+(g)^\circ$ are open sets, we have that there exists an open neighbourhood $U$ of $x$ such that $U \subset \supp_{-,0}(g)$ or $U \subset \supp_+(g)^\circ$. As a result $\HeavOp(g)$ is constant on $U$. Therefore, $(x,\lambda)$ cannot be in $\WF\bigl(\HeavOp(g)\bigr)$, which produces a contradiction. 

Hence, we can assume that $(x, \lambda) \in \WF\bigl(\HeavOp(g)\bigr)$ and $x \in \partial(\supp_+(g))$. In addition, we assume that $x \not \in C_g \cup S_g$. Then, $x \not \in  \singsupp(g)$. Therefore, there exists a neighbourhood $U'$ of $x$, on which $g$ is smooth and $\nabla g$ does not vanish. 

We wish to show now that on $U'$ the set $\{g = 0\}$ is a smooth curve with normal $\nabla_x g$ at $x$. For this, we invoke a \emph{smooth version of the implicit function theorem} \cite[Theorem 2.1]{lang2012real}. In this form, the theorem considers a smooth function $\tilde{g}\colon \Omega \to \Real$ such that 
\[
    0 = \tilde{g}(x_1^*,x_2^*), \quad\text{for $(x_1^*, x_2^*) \in \Omega$.}
\]
Assuming that $\frac{\partial \tilde{g}}{\partial x_2} \neq 0$, then there exists a smooth $\kappa$ defined on a neighbourhood of $x_1^*$ such that locally, i.e., for $x_1$ in an open neighbourhood of $x_1^*$,
\[
\tilde{g}(x_1,\kappa(x_1)) = 0
\quad\text{and}\quad
\kappa'(x_1) = 
\frac{\partial \tilde{g}}{\partial x_1}(x_1)
\bigg/
\frac{\partial \tilde{g}}{\partial x_2}(x_1).
\]
Moreover, in an open neighbourhood of $x_1^*,x_2^*$ every $(x_1,x_2)$ such that $\tilde{g}(x_1,x_2) = 0$ is of the form $(x_1,\kappa(x_1))$. 
Applying the implicit function theorem to $g$ if $\frac{\partial {g}}{\partial x_2} \neq 0$ yields that $\eta_x = \nabla g(x) /\|\nabla g\|$ is a normal at the zero level set of $g$ at $x$. By swapping variables, the same argument can be made if $\frac{\partial g}{\partial x_1} \neq 0$. We obtain that locally $\HeavOp(g) = \mathds{1}_{B}$ with $\partial B$ being a smooth curve that has normal $\eta_x$ at $x$. By Proposition~\ref{prop:WavefrontSetOfPiecewiseConstants}, this implies that $(x,\lambda) \in \WF(\mathds{1}_{B})$ if $\lambda = \pm \eta_x$ for an $\alpha \neq 0$. This concludes the proof of the "only if" part. 

For the "if" part, we notice again that if $x \in R_g$, then $x \in \partial(\supp_+(g))$. Thus, $x \not \in S_g$ which implies that $x \not \in \singsupp(g)$. Therefore, and since $x \not \in C_g$, the implicit function theorem is applicable. The same argument as before yields \eqref{eq:theconditionOntheLevelSets}.
\end{proof}

\begin{remark}
We may ask ourselves whether or not the statement in Proposition~\ref{prop:WavefrontSetViaImplicitFunctionTheorem} is tight. To improve our intuition in this regard, we provide an example for each of the cases of  Proposition~\ref{prop:WavefrontSetViaImplicitFunctionTheorem} that may lead to a wavefront set.
\begin{enumerate}
    \item \emph{Creation of wavefront set according to \eqref{eq:theconditionOntheLevelSets}:} 
    
    Let $g(x) = 1-\|x\|^2$. The squared Euclidean norm is a smooth function. We easily observe that 
    \[
    \{g = 0 \} = \{x \colon \|x\| = 1\}
    \]
    is the unit circle. Moreover, $\HeavOp(g) = \mathds{1}_{B_1}$ is the indicator of the unit ball. It is not hard to see that the wavefront set of this function is $\{(x,x) \colon x \in \mathbb{S}^1\}$. Also 
    \[
    \nabla_x (1-\|x\|^2) = \binom{2 x_1}{2 x_2} = 2x.
    \]
    \item \emph{$x \in C_g$ and $x \in \singsupp \bigl(\HeavOp(g)\bigr)$:} 
    
    Let $g_1$ be a positive $\Cont^\infty$ function, supported on a set $D_1$ that contains $(0,0)$. Let $g_2$ be another such function, however, with $D_1\cap D_2 = \{(0,0)\}$. If $D_1\cup D_2$ is not an open neighbourhood of $(0,0)$, which is possible, then $H(g_1+g_2)$ is discontinuous at $(0,0)$ implying that $(0,0)$ is a singular point of $H(g_1+g_2)$. One concrete example, would be given by $D_1 = [-1,0]^2$, $D_2 = [0,1]^2$. In this case, $\partial(\supp_+(g_1+  g_2))$ is not given by a single curve in the neighbourhood of $(0,0)$. Note that, necessarily by the smoothness of $g_1,g_2$ it holds that $\nabla_x(g_1 + g_2) = 0$ for $x = (0,0)$.
    
    \item \emph{$x \in C_g$ and $x \not \in \singsupp \bigl(\HeavOp(g)\bigr)$:} 
        Let $g$ be a smooth compactly supported positive function. Then every $x \in \partial \supp(g)$ satisfies that $x \in C_g$ and $x \not \in \singsupp \bigl(\HeavOp(g)\bigr) = \singsupp (g)$.
        
    \item \emph{$x \in S_g$ and $x \in \singsupp \bigl(\HeavOp(g)\bigr)$:} 
    
    Let $g = \mathds{1}_{\Real^+ \times \Real}$, then $x \in S_g$ and $g = \HeavOp(g)$. 
    
    \item \emph{$x \in S_g$ and $x \not \in \singsupp \bigl(\HeavOp(g)\bigr)$:} 
    
    Let $\phi \colon \Real \to \Real$ be a $\Cont^\infty$ function with compact support on $\Real^+$. The function $g(x) = \phi(x_1)  + \mathds{1}_{\Real^-}(x_1) x_1^3$ is not smooth since it has a jump in its third derivative at the $x_1 = 0$ axis. At the same time $\{x_1 = 0\} = \partial(\supp_+(g))$. Finally, we observe that $\HeavOp(g)(x) = \phi(x)$ and hence the wavefront set of $\HeavOp(g)$ is empty.
        
\end{enumerate}
\end{remark}

Notice that Proposition~\ref{prop:WavefrontSetViaImplicitFunctionTheorem} stays short of a precise characterisation of the wavefront set of $\HeavOp(\signal)$. It implies that all singularities must be in one of the sets $R_g, C_g$, or $S_g$. However, there is a closed-form of the orientations of the singularities only if $x\in R_g$.

\subsubsection{Wavefront set of \texorpdfstring{$\ReLU(\signal)= \HeavOp(\signal)\signal$}{Reluf}} 
\label{subsec:ResNetMicroLocalS2S2}

In this subsection, we chose a fixed $\kappa >0$ and $\phi_\kappa \in \SchwartzFunc(\Real^2)$ that integrates to 1, is positive, and is supported on a compact subset of $B_\kappa(0)$. To reduce the computation of $\WF(\ReLU_{\kappa, \phi_\kappa}(\signal))$ to that of $\WF(\HeavOp(\signal))$, we will make use of the following version of the product theorem:

\begin{theorem}[{\cite[Theorem 13]{brouder2014smoothwf}}]
\label{thm:S2P4}
Let $u$ and $v$ be distributions in $\SchwartzFunc'(U)$ for an open domain $U$. Assume that for no point $(x,\lambda)$ in $\WF(u)$ we have $(x,-\lambda) \in \WF(v)$. Then, $uv\in \SchwartzFunc'(U)$ and 
\[
\WF(uv) = S_+ \cup S _u \cup S_v,
\]
where $S_+ \coloneqq \{(x, \lambda +  \mu  ) \colon (x,\lambda) \in \WF(u), (x,\mu) \in \WF(v) \}$, $S_u \coloneqq \{(x, \lambda ) \colon (x,\lambda) \in \WF(u), x \in \esssuppp(v)\}$, and $S_v \coloneqq \{(x, \lambda ) \colon (x,\lambda) \in \WF(v), x \in \esssuppp(u)\}$.

In particular, for $\data\in\SchwartzFunc'(\Real^2)$ and $f\in \Cont^\infty(\Real^2)$ where $\supp(\signal)$ is compact, we have that $\WF(fg) \subset \WF(\data)\cap (\supp(\signal)\times \Real^2)$.
\end{theorem}

Let $\Omega \subset \Real^2$, $\signal \in \SchwartzFunc'(\Omega)$, and let us assume that $\WF(\signal)$ is known. In addition, using the results of Subsection~\ref{subsec:ResNetMicroLocalS2S1} we have also access to $\WF(\HeavOp(\signal))$. 
We denote as in Definition~\ref{def:ReLUonDistributions}:
\begin{align}
    X_h &\coloneqq \{x \in \Real^2 \setminus \supp_{\Lp^2}(h) \colon (x, \lambda) \in \WF(H(h)), (x, -\lambda) \in \WF(h) \text{ for a } \lambda \in \mathbb{S}^1\} + B_{\kappa}(0)
\end{align}
and 
\begin{align}
    X_h^{3\kappa} &\coloneqq \{x \in \Real^2 \setminus \supp_{\Lp^2}(h) \colon (x, \lambda) \in \WF(H(h)), (x, -\lambda) \in \WF(h) \text{ for a } \lambda \in \mathbb{S}^1\} + B_{3\kappa}(0).
\end{align}
Note that per Definition~\ref{def:ReLUonDistributions} we have that $\theta_\kappa = 0$ on $(X_h^{3\kappa})^c$ and hence $h^s = h$ on $(X_h^{3\kappa})^c$. 


Now we have collected all necessary ingredients to be able to compute the wavefront set of $\ReLU_{\kappa, \phi_\kappa}(\signal)=\HeavOp(\signal)f$, which is the goal of the following result.

\begin{theorem}
\label{thm:microcanonrelu}
Let $\Omega \subset \Real^2$ be open and let $f\in \SchwartzFunc'(\Omega)$. In addition, let
\begin{align}
\mathcal{A}_f&\coloneqq \WF(\signal)\cap (\interior{\supp_+(\signal)}\times \mathbb{S}^1),\\
\mathcal{R}_f &\coloneqq \{(x,\lambda)\in R_f\times \mathbb{S}^1: \text{$(x,\lambda)$ follows~\eqref{eq:theconditionOntheLevelSets}}\},
\end{align}
where $R_f$ is defined as in Proposition~\ref{prop:WavefrontSetViaImplicitFunctionTheorem}. In addition, $\mathcal{CS}_f$ is given by
\begin{align}
\mathcal{CS}_f \coloneqq& \{(x,\xi)\in (S_f \cup C_f)\times \mathbb{S}^1 \colon (x,\xi)\in \WF(\ReLU_{\kappa, \phi_\kappa}(\signal))\}, \label{eq:FormulaForSpecialWavefrontSetComponent}    
\end{align}
where $C_f$ and $S_f$ are defined as in Proposition~\ref{prop:WavefrontSetViaImplicitFunctionTheorem}.
Then $\WF(\ReLU_{\kappa, \phi_\kappa}(\signal))$ is given by:
\begin{align}\label{eq:PreciseDescriptionOfWavefrontSet}
    \WF(\ReLU_{\kappa, \phi_\kappa}(\signal)) \cap (X_h^{3\kappa} \times \mathbb{S}^1)^c &= (\mathcal{A}_f\cup \mathcal{R}_f\cup \mathcal{CS}_f) \cap (X_h^{3\kappa} \times \mathbb{S}^1)^c,\\
    \WF(\ReLU_{\kappa, \phi_\kappa}(\signal)) \cap (X_h^{3\kappa} \times \mathbb{S}^1) &\subset \left(\mathcal{A}_f\cup \mathcal{R}_f \cup \mathcal{CS}_f\right) \cap (X_h^{3\kappa} \times \mathbb{S}^1),\label{eq:PreciseDescriptionOfWavefrontSetThepartThatWeThrowAway}
\end{align}

In particular, we have that 
\begin{equation}\label{eq:forFastAlgorithm}
\WF(\ReLU_{\kappa, \phi_\kappa}(\signal))\subset\mathcal{A}_f\cup \mathcal{R}_f\cup \left((C_f\cup S_f)\times \mathbb{S}^1\right).
\end{equation}
\end{theorem}
\begin{proof}
Since $\Real^2$ can be decomposed as
\[
\Real^2 = \interior{\supp_+(\signal)}\cup\partial(\supp_+(\signal))\cup\supp_{-,0}(\signal),
\]
we have that $\WF(\ReLU_{\kappa, \phi_\kappa}(\signal))$ can be decomposed as
\begin{align}\label{eq:TheDecompositionOfTheWavefrontSetForMainProp}
\WF(\ReLU_{\kappa, \phi_\kappa}(\signal))= \mathrm{A}_{\signal, \kappa}\cup\mathrm{B}_{\signal, \kappa}\cup\mathrm{D}_{\signal, \kappa},
\end{align}
where 
\begin{align*}
\mathrm{A}_{\signal, \kappa} &\coloneqq \WF(\ReLU_{\kappa, \phi_\kappa}(\signal))\cap (\interior{\supp_+(\signal)}\times\mathbb{S}^1),\\
\mathrm{B}_{\signal, \kappa} &\coloneqq \WF(\ReLU_{\kappa, \phi_\kappa}(\signal))\cap(\supp_{-,0}(\signal)\times\mathbb{S}^1),\\
\mathrm{D}_{\signal, \kappa} &\coloneqq \WF(\ReLU_{\kappa, \phi_\kappa}(\signal))\cap(\partial(\supp_+(\signal))\times\mathbb{S}^1).
\end{align*}
Notice in addition that $\mathrm{A}_{\signal, \kappa}$, $\mathrm{B}_{\signal, \kappa}$ and $\mathrm{D}_{\signal, \kappa}$ are disjoint. Now, since $\interior{\supp_+(\signal)}$ is open, we find for every $x\in\interior{\supp_+(\signal)}$ with $x \not \in X_h^{3\kappa}$ an open neighbourhood $U$ of $x$ such that $\ReLU_{\kappa, \phi_\kappa}(\signal)|_U=\HeavOp(\signal)|_U \signal^s|_U= \signal|_U$, since $\HeavOp(\signal)(x)=1$ for every $x\in U$ and $h^s = h$ on $(X_h^{3\kappa})^c$. Thus
\begin{align}
\mathrm{A}_{\signal, \kappa} \cap (X_h^{3\kappa} \times \mathbb{S}^1)^c= \WF(\signal)\cap (\interior{\supp_+(\signal)}\times\mathbb{S}^1) \cap (X_h^{3\kappa} \times \mathbb{S}^1)^c = \mathcal{A}_{\signal} \cap (X_h^{3\kappa} \times \mathbb{S}^1)^c, \label{eq:TheEstimateOfTheAfTerm1}
\end{align}
where $\mathcal{A}_f$ is as in the statement of the proposition. 
Moreover, for every $x\in\interior{\supp_+(\signal)}$ with $x \in X_h^{3\kappa}$ there is an open neighbourhood $U'$ of $x$ such that $\ReLU_{\kappa, \phi_\kappa}(\signal)|_{U'}=\HeavOp(\signal)|_{U'} \signal^s|_{U'} = \signal^s|_{U'}$. Therefore,
\begin{equation}
\label{eq:TheEstimateOfTheAfTerm2}
\begin{aligned}
    \mathrm{A}_{\signal, \kappa} \cap (X_h^{3\kappa} \times \mathbb{S}^1)
    &= \WF(\signal^s)\cap (\interior{\supp_+(\signal)}\times\mathbb{S}^1) \cap (X_h^{3\kappa} \times \mathbb{S}^1) 
    \\
    &\subset \WF(\signal)\cap (X_h^{3\kappa} \times \mathbb{S}^1) = \mathcal{A}_f \cap (X_h^{3\kappa} \times \mathbb{S}^1).
\end{aligned}
\end{equation}

On the other hand, since $\esssuppp(\HeavOp(\signal)) = \supp_+(\signal)$, by Theorem~\ref{thm:S2P4}, we can conclude that
\[
\supp_{-,0}(\signal)\subset (\singsupp(\HeavOp(\signal)\signal^s))^c= (\singsupp(\ReLU_{\kappa, \phi_\kappa}(\signal)))^c.
\]
Then, we have
\begin{align}
    \mathrm{B}_{\signal, \kappa}=\WF(\ReLU_{\kappa, \phi_\kappa}(\signal))\cap(\supp_{-,0}(\signal)\times\mathbb{S}^1) = \emptyset. \label{eq:TheBTermIsEmpty}
\end{align}
Let us now study the set $\mathrm{D}_{\signal, \kappa} \coloneqq \WF(\ReLU_{\kappa, \phi_\kappa}(\signal))\cap(\partial(\supp_+(\signal))\times\mathbb{S}^1)$. Following the notation of Proposition~\ref{prop:WavefrontSetViaImplicitFunctionTheorem}, we can decompose the set $\partial(\supp_+(\signal))$ as

\begin{align}\label{eq:LetsCallThisThingTHEDECOMPOSITION}
\partial(\supp_+(\signal))=R_f\cup C_f\cup S_f.
\end{align}
Using this decomposition, we can write $\mathrm{D}_{\signal, \kappa}$ as

\[
\mathrm{D}_{\signal, \kappa} = \mathrm{R}_{\signal, \kappa}\cup \mathrm{CS}_{\signal, \kappa},
\]
where 
\begin{align*}
    \mathrm{R}_{\signal, \kappa} &\coloneqq \WF(\ReLU_{\kappa, \phi_\kappa}(\signal))\cap(R_f\times\mathbb{S}^1),\\
    \mathrm{CS}_{\signal, \kappa} &\coloneqq \WF(\ReLU_{\kappa, \phi_\kappa}(\signal))\cap((C_f\cup S_f)\times\mathbb{S}^1).
\end{align*}
Next, we would like to show that 
\begin{align}\label{eq:LetsShowThis}
    \mathrm{R}_{\signal, \kappa} \cap (X_h^{3\kappa} \times \mathbb{S}^1)^c  &= \{(x,\lambda)\in R_f\times \mathbb{S}^1 \colon (x,\lambda) \text{ follows~\eqref{eq:theconditionOntheLevelSets}}\} \cap (X_h^{3\kappa} \times \mathbb{S}^1)^c= \mathcal{R}_f \cap (X_h^{3\kappa} \times \mathbb{S}^1)^c, \\
    \mathrm{R}_{\signal, \kappa} \cap (X_h^{3\kappa} \times \mathbb{S}^1)  &\subset \{(x,\lambda)\in R_f\times \mathbb{S}^1 \colon (x,\lambda) \text{ follows~\eqref{eq:theconditionOntheLevelSets}}\}\cap (X_h^{3\kappa} \times \mathbb{S}^1) = \mathcal{R}_f \cap (X_h^{3\kappa} \times \mathbb{S}^1).\label{eq:LetsShowThis2}
\end{align}
Let us start with \eqref{eq:LetsShowThis}. Consider first $(x,\lambda) \in \mathcal{R}_f$, $x \not \in X_h^{3\kappa}$. Then, $x \in R_f$ and thus $x \not \in \singsupp f$. In particular, $x \not \in \singsupp \signal^s$. Moreover, since $\nabla_x f \neq 0$, we conclude that $x \in \esssuppp(\signal)$ and therefore $x \in \esssuppp(\signal^s)$. 
Using Theorem~\ref{thm:S2P4}, we conclude that $(x, \lambda) \in \WF(\ReLU_{\kappa, \phi_\kappa}(\signal)) = \WF(\signal^s \HeavOp(\signal))$ if and only if $(x, \lambda) \in \WF(\HeavOp(\signal))$. Since $x \not \in C_f\cup S_f$, we conclude from  Proposition~\ref{prop:WavefrontSetViaImplicitFunctionTheorem} that $(x,\lambda)$ satisfies \eqref{eq:theconditionOntheLevelSets}. To show the converse embedding, assume that $(x,\lambda)$ is such that $x \in R_f$ and $(x,\lambda)$ satisfies \eqref{eq:theconditionOntheLevelSets}. By Proposition~\ref{prop:WavefrontSetViaImplicitFunctionTheorem}, we have that $(x,\lambda) \in \WF(\HeavOp(\signal))$. Furthermore, $x \not \in \singsupp f$ and $\nabla_x f \neq 0$, which implies that $x \in \esssuppp f$. We conclude by Theorem~\ref{thm:S2P4} that $(x,\lambda) \in \WF(\signal^s \HeavOp(\signal)) = \WF(\ReLU_{\kappa, \phi_\kappa}(\signal))$. This shows \eqref{eq:LetsShowThis}.

To show \eqref{eq:LetsShowThis2} it suffices to observe that $x \in R_f$ implies that $x \not \in \singsupp f$ and hence $x \not \in \singsupp (1-\theta_\kappa) f$. Therefore, we conclude that 
\begin{equation*}
\mathrm{R}_{\signal, \kappa}\cap (X_h^{3\kappa} \times \mathbb{S}^1) \subset \WF(\HeavOp(\signal))\cap(R_f\times\mathbb{S}^1) \cap (X_h^{3\kappa} \times \mathbb{S}^1)
= \mathcal{R}_f \cap (X_h^{3\kappa} \times \mathbb{S}^1),
\end{equation*}
where the last equality follows from Proposition~\ref{prop:WavefrontSetViaImplicitFunctionTheorem}. This yields \eqref{eq:LetsShowThis2}. 

The full result now follows by considering the decomposition \eqref{eq:TheDecompositionOfTheWavefrontSetForMainProp}. The part associated with $\mathrm{A}_{\signal, \kappa}$ is estimated via \eqref{eq:TheEstimateOfTheAfTerm1} and \eqref{eq:TheEstimateOfTheAfTerm2}. The part associated with $\mathrm{B}_{\signal, \kappa}$ vanishes due to \eqref{eq:TheBTermIsEmpty}. Finally, the part associated with $\mathrm{D}_{\kappa, \phi_\kappa}$ is estimated via the decomposition \eqref{eq:LetsCallThisThingTHEDECOMPOSITION}, where the $\mathrm{R}_{f,\kappa}$ part is estimated via \eqref{eq:LetsShowThis} and \eqref{eq:LetsShowThis2} and $\mathcal{CS}_{f} = \mathrm{CS}_{\signal, \kappa}$ holds per definition. 
\end{proof}

\begin{remark}
Theorem~\ref{thm:microcanonrelu} only yields an estimate for the wavefront set associated with the set $(X^\kappa_h)^c$. In the sequel, since we are using the continuum relations to infer certain properties of digital relations, we will assume that $\kappa$ is chosen very small and $X^\kappa_h$ is not seen by the discretisation. In other words, in practice, we compute the wavefront set only via \eqref{eq:PreciseDescriptionOfWavefrontSet}.
\end{remark}

Even on $(X^\kappa_h)^c$, Theorem~\ref{thm:microcanonrelu} does not entirely save us from computing $\WF(\ReLU_{\kappa, \phi_\kappa}(\signal))$, but restricts the necessity for such a computation to the cases where $x \in  S_f \cup C_f$. The set 
$\mathcal{CS}_f$ can also be further split up: For $x \in \Real^2 \setminus X^\kappa_h$, we denote by $\WF(\signal)_x$ the $x$-slice of the wavefront set of $\signal$ defined as $\Lambda \subset \Real^2$ such that $(x,\lambda) \in \{x\}\times \Lambda$ for all $(x,\lambda) \in \WF(\signal)$.

\begin{proposition}\label{prop:computationofCf}
Let $\Omega \in \Real^2$ be open, $f\in \SchwartzFunc'(\Omega)$ be a distribution and $\mathcal{CS}_f$ be as in Theorem~\ref{thm:microcanonrelu}. Then
\begin{equation}\label{eq:decompositionofCS}
\begin{split}
\mathcal{CS}_f &\bigcap (X^\kappa_h \times \mathbb{S}^1)^c
= \Bigl\{ (x, \lambda) : 
         x\in C_g \setminus X^\kappa_h 
         \text{ and } 
         (x,\lambda) \in \WF\bigl(\HeavOp(\signal)\bigr) 
       \Bigr\} 
\\
&\quad \bigcup \Bigl\{ (x, \lambda) :
      x\in S_g \setminus X^\kappa_h, 
      \WF(\signal)_x \cap -\WF\bigl(\HeavOp(\signal)\bigr)_x = \emptyset, 
      \lambda \in (\WF(\signal)_x + \WF\bigl(\HeavOp(\signal)\bigr)_x) \setminus \{0\} 
    \Bigr\}
\\
&\quad \bigcup \Bigl\{ (x, \lambda) : 
      x\in S_g \setminus X^\kappa_h, 
      \WF(\signal)_x \cap -\WF\bigl(\HeavOp(\signal)\bigr)_x \neq \emptyset, 
      (x,\lambda) \in \WF\bigl(\ReLU_{\kappa, \phi_\kappa}(\signal)\bigr)
    \Bigr\}. 
\end{split}
\end{equation}
\end{proposition}
\begin{proof}
 The result follows immediately from Theorem~\ref{thm:S2P4}.
\end{proof}

Theorem~\ref{thm:microcanonrelu} yields two ways to estimate the wavefront set of $\ReLU_{\kappa, \phi_\kappa}(\signal)$ on $(X^\kappa_h)^c$. First, it can be precisely computed by \eqref{eq:PreciseDescriptionOfWavefrontSet}. This, however, may require us to compute $\mathcal{CS}_f$ via Proposition~\ref{prop:computationofCf}. This computation could be performed according to \eqref{eq:decompositionofCS}, by using a method such as DeNSE to find $\WF(\ReLU_{\kappa, \phi_\kappa}(\signal))$ if required.

Second, the wavefront set on $(X^\kappa_h)^c$ can be estimated using \eqref{eq:forFastAlgorithm}. Since we expect that it is not problematic to overestimate the wavefront set slightly, we opt for the second option and cast this wavefront set extraction algorithm as Algorithm~\ref{alg:2}.

\begin{algorithm}[htb]
\SetAlgoLined
\KwIn{Distribution $\signal \in \SchwartzFunc'(\Omega)$, $\WF(\signal)$, $x \in \Omega$.}
\KwOut{Estimate $\WF(\ReLU_{\kappa, \phi_\kappa}(\signal))_x \subset \Omega$.}
 initialisation\;
 \If{$x \in\interior{\supp_+(\signal)}$}{
   \textbf{return}  $\Lambda_x = \WF(\signal)_x$\;
   }
  \If{$x \in R_f$}{
   \textbf{return}  $\Lambda_x = \{  \pm \nabla_x(\signal)/\|\nabla_x(\signal)\|\}$\;
   }
  \If{$x \in C_f\cup S_f$}{
   \textbf{return}  $\Lambda_x = \Omega$\;
   }  
 \caption{Wavefront set classifier of $\ReLU_{\kappa, \phi_\kappa}(\signal)$.}
 \label{alg:2}
\end{algorithm}

\subsubsection{Microlocal analysis of the residual layer and sum-taking}
\label{subsec:ResNetMicroLocalS3}

In the continuous setting, residual neural networks are operators $\mathcal{P} \colon \SchwartzFunc'(\Omega)\to \SchwartzFunc'(\Omega)$ of the form
\[
    \mathcal{P}(\signal)=f+\mathcal{C}(\signal),
\]
where $\mathcal{C} \colon \SchwartzFunc'(\Omega)\to \SchwartzFunc'(\Omega)$ is 
continuum two-dimensional convolutional ResNet according to Definition~\ref{def:ResNetContinuous}. In addition, in Subsection~\ref{subsec:thecontinuumOperator}, we allow summing over the channels of a convolutional block. 

Because of this, we would like to identify $\WF(f+g)$ for two distributions of which we know the wavefront set. The following two results yield a complete description thereof.

\begin{theorem}[{\cite[Page 93]{reed1975methods}}]
\label{thm:regularsum}
Let $\Omega \subset \Real^2$ be open, let $f,g\in \SchwartzFunc'(\Omega)$ and $(x;\lambda)$ be a regular directed point of $\signal$ and $g$, then $(x;\lambda)$ is a regular directed point of $f+g$. 

In particular, if $(x,\lambda) \in \WF(\signal)$ and $(x;\lambda)$ is a regular directed point of $g$, then $(x,\lambda) \in \WF(f+g)$.
\end{theorem}

\begin{corollary}
\label{cor:sumwfset}
Let $\Omega \subset \Real^2$ be open and let $f,g\in \SchwartzFunc'(\Omega)$, then $\WF(f+g)$ is given by
\begin{equation}
\label{eq:microcanonresidual}
\WF(f + g) = \mathcal{A}_f\cup \mathcal{A}_g\cup \mathcal{A}_{f+g},
\end{equation}
where 
\begin{alignat*}{2}
\mathcal{A}_f &\coloneqq \{(x;\lambda)\in \WF(\signal):x\not\in \WF(\data)\},
&\qquad
\mathcal{A}_{f+g} &\coloneqq ((\WF(\signal)\cap\WF(\data)))\cap \WF(f+g),
\\
\mathcal{A}_g &\coloneqq \{(x;\lambda)\in \WF(\data):x\not\in \WF(\signal)\}
 & &
\end{alignat*}
In particular, we have that 
\begin{align}\label{eq:superset}
\mathcal{A}_{f+g}\subset \WF(\signal)\cap\WF(\data).
\end{align}
\end{corollary}
\begin{proof}
The result follows immediately from Theorem~\ref{thm:regularsum}.
\end{proof}

In a similar fashion as we did for the wavefront set of $\ReLU(\signal)$, using Corollary~\ref{cor:sumwfset} we can find two ways to extract the wavefront set of $f+g$ via Corollary~\ref{cor:sumwfset}. We cast the one that yields a superset of the wavefront set of $f+g$ via \eqref{eq:superset} as Algorithm~\ref{alg:4}.

\begin{algorithm}[H]
\SetAlgoLined
\KwIn{Distribution $f,g \in \SchwartzFunc'(\Omega)$, $\WF(\signal),\WF(\data)$, $x \in \Omega$.}
\KwOut{Estimate $\WF(f+g)_x \subset \Omega$.}
 initialisation\;
 \If{$x \in \WF(\signal)\cap \WF(\data)^c$}{
   \textbf{return}  $\Lambda_x = \WF(\signal)_x$\;
   }
  \If{$x \in \WF(\signal)^c\cap \WF(\data)$}{
   \textbf{return}   $\Lambda_x = \WF(\data)_x$\;
   }
    \If{$x \in \WF(\signal)\cap \WF(\data)$}{
   \textbf{return}   $\Lambda_x = \WF(\data)_x \cup \WF(\signal)_x$;
   }
 \caption{Wavefront set classifier of $f+g$.}
 \label{alg:4}
\end{algorithm}

\subsubsection{Microlocal analysis of continuous Learned Primal-Dual networks}
\label{subsec:ResNetMicroLocalS5}

Let us first notice that the continuum residual neural network operator as stated in Definition~\ref{def:ResNetContinuous} has four basic components, the differential layers, summation over channels, application of the ReLU, and the residual connection. 
We have seen in Subsection~\ref{subsec:ResNetMicroLocalS3} that the effect on the wavefront set through summation over channels and the residual connection can be described as the output of Algorithm~\ref{alg:4}. In addition, the effect of the differential layers is described by \eqref{eq:microcanonconv} and the wavefront set after an application of the ReLU can be found through an application of Algorithm~\ref{alg:2}. Overall, there is an algorithm that produces for every continuum convolutional ResNet and every input function of which the wavefront set is known, an estimate of the wavefront set of the output. 
The microlocal behaviour of the full continuum Learned Primal-Dual network (Algorithm~\ref{alg:contLearnedPrimalDual}) can now be found by iteratively applying the canonical relation for the continuum residual neural network operator as well as the canonical relation for the Radon transform \eqref{eq:microcanonradon} and the adjoint of its Fr\'echet derivative (back-projection) in \eqref{eq:microcanonbackproj}. 

\section{The joint reconstruction and wavefront set extraction algorithm}
\label{sec:DigitalMicrolocal}

In this section, we present our numerical algorithm for the digital reconstruction problem associated with the inverse tomography problem. We first introduce the algorithm in the framework of statistical learning theory as an empirical risk minimisation problem over a special set of deep neural networks with a specific loss term. Then, we evaluate our approach on a test set and compare its results with various benchmarks. 

\begin{figure}[htb]
\centering
\includegraphics[width=\linewidth]{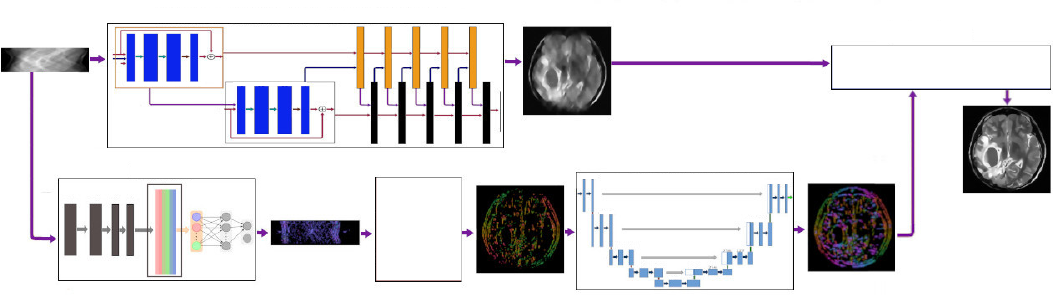}
\begin{tikzpicture}[remember picture,overlay]
  \node[anchor=north west, text width=0.35\linewidth, align=left, draw=none, fill=none] at (-6.3,5.15)
    {{\small Learned Primal-Dual (LPD) network}};
  \node[anchor=north west, text width=0.3\linewidth, align=left, draw=none, fill=none] at (-6.5,2.7)
    {{\small DeNSE}};
  \node[anchor=north west, text width=0.07\linewidth, align=left, draw=none, fill=none] at (-2.4,2.2)
    {{\scriptsize Canonical \par relation for LPD \par}};
  \node[anchor=north west, text width=0.3\linewidth, align=left, draw=none, fill=none] at (1,2.8)
    {{\small U-Net inpainting}};         
  \node[anchor=north west, text width=0.1\linewidth, align=left, draw=none, fill=none] at (5.25,4.75)
    {{\small Joint loss}};
  \node[anchor=north west, text width=0.1\linewidth, align=left, draw=none, fill=none] at (4.75,4.35)
    {{\scalebox{0.5}{%
         $\loss_{\mathrm{joint}}\bigl((\data_1,\dataother_1),(\data_1,\dataother_1)\bigr)$}
       \vskip-0.3\baselineskip
       \scalebox{0.5}{%
         \qquad$\coloneqq C \loss_{\mathrm{rec}}(\data_1,\data_2) 
         + (1-C)\loss_{\mathrm{imp}}(\dataother_1,\dataother_2)$}
     }};
\end{tikzpicture}
\caption{Outline of the joint reconstruction and wavefront set inpainting algorithm. The input is partial sinogram data. In the top row first a Learned Primal-Dual architecture is applied. In the bottom row we first apply DeNSE to extract the wavefront set, then we apply the canonical relation for the Learned Primal-Dual network of Subsection~\ref{subsec:ResNetMicroLocalS5}. To the output thereof we apply the U-Net for inpainting. This together with the output from the Learned Primal-Dual is then input into the joint loss function. }
\label{fig:task-adapt}
\end{figure}
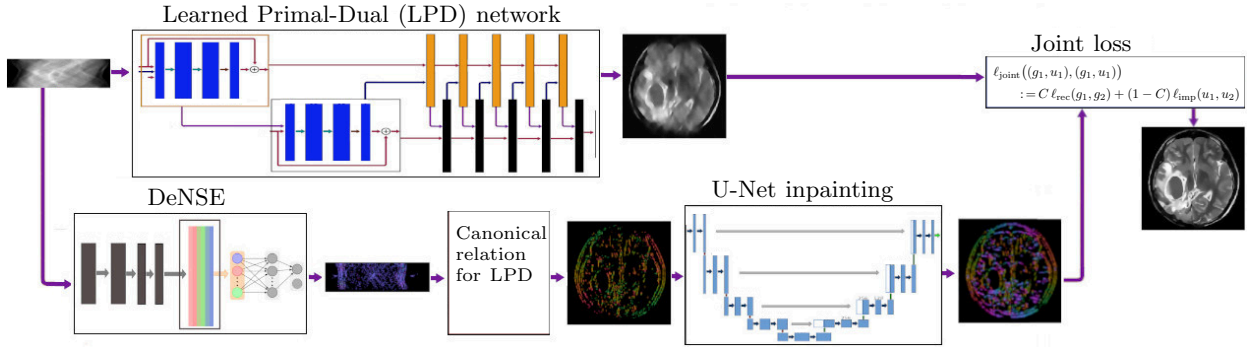

\subsection{Outline of our algorithm}
\label{subsec:algorithm}

The setting for our data driven approach is to have a supervised tomographic training dataset of the form $(\vecsignal_i, \vecdata_i)_{i=1}^N \in \Real^n \times \Real^m$. 
Here, $\vecsignal_i \in \Real^n$ is an array that is a discretisation of a function $\signal \in \RecSpace = \Lp^2(\Real^2)$ that represents the true image (signal).
Likewise, $\vecdata_i \in \Real^m$ is an array corresponding to a discretisation of continuum data, which here is a function $\data_i \in \DataSpace = \Lp^2(\datadomain)$ representing a noisy limited-angle sinogram with $\datadomain$ denoting the corresponding manifold of lines.
In addition, we also know the complete wavefront set for data $\data_i$ and signal $\vecsignal_i$. Subsection~\ref{subapp:TaskAdaptedS1} explores this statistical framework approach in more detail.

From the above, in absence of observation noise we have $\data_i = \Radon(\signal_i)$ where $\Radon \colon \RecSpace \to \DataSpace$ is the Radon transform (forward operator) restricted to $\datadomain$ (limited-angle data).
The elements $(\signal_i, \data_i) \in \RecSpace \times \DataSpace$ are assumed to be independent samples generated by some $(\RecSpace \times \DataSpace)$-valued random variable $(\stsignal, \stdata)$.
The reconstruction operator $\RecOp_{\theta} \colon \Real^m \to \Real^n$ is now a DNN given by the Learned Primal-Dual network \cite{adler2018lpd} and its parameter $\theta=\hat{\theta}$ is set by training this DNN to achieve two tasks simultaneously:
\begin{enumerate}
\item Applying $\RecOp_{\hat{\theta}} \colon \Real^m \to \Real^n$ (trained reconstruction operator) to an input sinogram $\vecdata \in \Real^m$ that is a discretisation of $\data \in \DataSpace$ drawn from $\stdata$, should produce $\vecsignal \in \Real^n$ that is a discretisation of a function $\signal \in \RecSpace$ such that $\Radon(\signal) = \data$.
\item The trained reconstruction operator $\RecOp_{\hat{\theta}} \colon \Real^m \to \Real^n$ should output a discretisation of the wavefront set of $\signal$, which we denote by $\vecWF(\signal)$, i.e., $\vecWF\bigl(\RecOp_{\hat{\theta}}(\vecdata)\bigr) = \vecWF(\signal)$
\end{enumerate}
While we are, in the end, only interested in an accurate reconstruction of $\signal$ (see Subsection~\ref{sec:DigRec}), we believe that requiring a neural network to perform both tasks at the same time constitutes a strong prior. Indeed, we will see below that this joint approach yields vastly superior reconstruction results to a neural network that only reconstructs $\signal$. As a key ingredient towards the solution of the aforementioned task, we can rely on a digital wavefront set extraction operator DeNSE \cite{andrade2019wfset} which allows us to extract the digital wavefront set of a digital image, this task is briefly introduced in Subsection~\ref{sec:WFExtract}. Furthermore, it is reasonable to construct an appropriate loss function $\Loss$ that measures the discrepancy between $\Phi(\vecdata)$ and $\vecsignal$ and between $\vecWF(\Phi(\vecdata))$ and $\vecWF(\signal)$. 

However, it turns out that for the training of the neural network it is computationally prohibitive to compute $\vecWF(\Phi(\vecdata))$ in every training iteration with the shearlet-based wavefront set extractor DeNSE \cite{andrade2019wfset}. Because of this, we introduced a heuristic based on the mapping properties of a continuum operator, the continuum Learned Primal-Dual network, to estimate $\vecWF(\Phi(\vecdata))$ in Subsection~\ref{subsec:CanonRelCLPD}. 


Through the canonical relation for the continuous Learned Primal-Dual operator, we obtain an estimate of a discretisation of the visible wavefront set $\WFvis(\Phi(\data))$ based on the weights of the neural network $\Phi$. 
To relate this estimate to $\vecWF(\signal)$, we apply wavefront set inpainting with a neural network $\Psi$ of U-Net architecture \cite{ronnenberger2015unet}, this is explained in detail in Subsection~\ref{sec:WFInpaint}. Overall, this leads to an operator $\operator{Q}$ that takes as an input $\vecdata$ and the neural networks $\Phi$ and $\Psi$ and outputs an estimate of $\vecsignal$ and $\vecWF(\signal)$. 
We therefore use the loss function
\begin{align}\label{eq:lossFunction}
    \Loss(\Phi, \Psi, \vecdata, \vecsignal, \vecWF(\signal)) \coloneqq 
    C \loss_1(\vecsignal ,\Phi(\vecdata)) + (1-C) \loss_2(\vecWF(\signal), \Psi(\operator{Q}(\vecdata,\Phi))), 
\end{align} 
where $C\in (0,1]$ and $\loss_1,\loss_2$ are appropriate distance measures on discretised images and discretised wavefront sets that will be discussed in Subsection~\ref{sec:discretisation} below. 
This strategy that jointly trains a reconstruction and a task (wavefront set inpainting) falls in the framework of \emph{task-adapted reconstruction} was introduced in~\cite{adler2018taskadapt}. 
Based on the loss function of \eqref{eq:lossFunction}, we now train the neural networks $\Phi$ and $\Psi$ to minimise the objective 
\begin{align}
    \frac{1}{N}\sum_{i=1}^N \Loss(\Phi, \Psi, \vecdata_i, \vecsignal_i, \vecWF(\signal_i)).
\end{align}
To find a minimiser, we use stochastic gradient descent-based optimisation over the neural network's weights as is standard in deep learning. Finally, Subsection~\ref{sec:jointreconinpaint} explores in detail the joint reconstruction in the digital setting.

\subsection{Statistical learning framework}
\label{subapp:TaskAdaptedS1}

We would like to frame the problem of reconstruction and joint wavefront set extraction as a statistical learning problem \cite{shalev2014understanding, Cucker02onthe}. For this, we first introduce some relevant notions:

\begin{definition}[Setting of statistical learning theory]
\label{def:statdec}
Let $\SampleSpace, \LabelSpace$ be two sets. We call $\SampleSpace$ the \emph{sample space}, and $\LabelSpace$ the \emph{label space}.  Further, let $\mathcal{D}$ be a distribution on $\SampleSpace \times \LabelSpace$, and let $\ell \colon \LabelSpace \times \LabelSpace \to \Real$ be a measurable function. We refer to $\ell$ as \emph{loss function}. Then the \emph{risk} $\mathrm{R}$ of a measurable function $h : X \to Y$ is defined as 
\[
	\mathrm{R}(h) \coloneqq \mathbb{E}_{(x,y) \sim \mathcal{D}}( \ell(h(x), y)). 
\]
For $m \in \Natural$, a set of samples $(x_i,y_i)_{i=1}^m \sim \mathcal{D}^m$ and for a hypothesis class $\mathcal{H}$, we define the \emph{empirical risk minimiser} $h^*$ as 
\[
	h^* \coloneqq \argmin_{h \in \mathcal{H}} \frac{1}{m} \sum_{i=1}^m  \ell(h(x_i), y_i)). 
\]
\end{definition}
Under certain assumptions on the complexity of the hypothesis set $\mathcal{H}$, it can be shown that for sufficiently large $m$, the function $h^*$ will also achieve a small risk \cite{shalev2014understanding, Cucker02onthe}. This means that the empirical risk minimiser $h^*$ resolves the unknown relation between input and output described by $\mathcal{D}$. Because of this, we treat the empirical risk minimisation problem in this work as the central problem of interest. 

Next we would like to phrase the problems of digital tomographic reconstruction, of digital visible wavefront set extraction, of full digital wavefront set extraction via inpainting, and of the joint reconstruction and digital wavefront set extraction and inpainting in the framework of statistical learning theory.

\subsection{Digital tomographic reconstruction}
\label{sec:DigRec}
A joint probability distribution on function spaces as described through the statistical formulation of the tomographic reconstruction problem in Definition~\ref{def:tomrecon} implies, via the discretisation procedure of Subsection~\ref{sec:discretisation}, a joint probability distribution $\mathcal{D}_{\mathrm{rec}}$ on discretised image-sinogram data. 
In the discrete problem, the sample space is $\Real^{I_s} \eqqcolon \SampleSpace$ for $I_s \subset [m_1] \times [m_2]$, where $[M] \coloneqq \{1, \dots, M\}$, with the label space being defined by $\Real^{n_1 \times n_2}  \eqqcolon \LabelSpace$. 
As loss function we choose $\ell_{\mathrm{rec}}$ to be the squared Euclidean distance between two elements in $\LabelSpace$, i.e., 
\begin{align}
 \ell_{\mathrm{rec}}(y_1,y_2) \coloneqq \sum_{i,j=1}^{n_1,n_2} |(y_1)_{i,j} - (y_2)_{i,j}|^2.
\end{align}
Moreover, we choose as a hypothesis set $\mathcal{H}_{\mathrm{rec}} \coloneqq (h_\vartheta)_{\vartheta \in \Theta_{\mathrm{res}}}$ the set of functions which perform the mapping $f_0 \mapsto f_I$ according to Algorithm~\ref{alg:contLearnedPrimalDual} parametrised via the weights of the discrete two-dimensional convolutional ResNets introduced in Definition~\ref{def:ResNetDiscrete}.

\subsection{Digital visible wavefront set extraction}
\label{sec:WFExtract}
Similarly to the previous Subsection~\ref{sec:DigRec}, the discretisation procedure of Subsection~\ref{sec:discretisation} implies a joint distribution between the digital wavefront set of a measured sinogram and the digital visible wavefront set of the image, precisely defined below: 

\begin{definition}
Let $I_s \subset [m_1] \times [m_2]$ be a digital grid in the sinogram domain and $I_d \subset [180]$. The \emph{digital visible wavefront set of the sinogram data} is given by $\SampleSpace_{\visible} \coloneqq \Real^{I_s} \times \{0, 1\}^{I_d}$, where $I_d$ represents the set of \emph{visible angles}. 
Moreover, let $J_s \subset [n_1] \times [n_2]$ be a digital grid in the image domain and $J_d \subset [180]$, determined through the canonical relation for the Radon transform. The \emph{digital visible wavefront set of the image data} is given by $\LabelSpace_{\visible} \coloneqq \Real^{J_s} \times \{0, 1\}^{J_d}$.
\end{definition}

The underlying loss function over $\LabelSpace_{\visible}$ is:
\[
\ell_{\visible}(y,y') \coloneqq 
  -\sum_{(i_1,i_2) \in J_s}\sum_{i_3 \in J_d} y_{i_1,i_2,i_3}\log \bigl( y_{i_1,i_2,i_3}' 
\bigr)
\quad\text{for $y, y' \in \LabelSpace_{\visible}$.}
\]
As a hypothesis set we choose $\mathcal{H}_{\visible} \coloneqq (\hat{h}_\vartheta)_{\vartheta \in \Theta_{\mathrm{res}}}$ as the set of maps resulting from the digital canonical relation described in Subsection~\ref{subsec:ResNetMicroLocalS5}. These maps are parametrised via the weights of the discrete two-dimensional convolutional ResNets introduced in Definition~\ref{def:ResNetDiscrete}. The important point to notice here is that \emph{$\mathcal{H}_{\visible}$ and $\mathcal{H}_{\mathrm{rec}}$ are parametrised by the same parameter set}. 

\subsection{Digital wavefront set extraction and inpainting}
\label{sec:WFInpaint}

To reconstruct the full digital wavefront set of an image from its sinogram data, we combine the visible wavefront set extraction of the previous subsection with a wavefront set inpainting step which will be performed using a U-Net.
In this case, the sample set is $\SampleSpace_{\mathrm{inp}} \coloneqq \Real^{I_s} \times \{0, 1\}^{I_d}$, where $I_s \subset [m_1] \times [m_2]$ and $I_d \subset [180]$, and label space has consequently to be chosen as $\LabelSpace_{\mathrm{inp}} \coloneqq \Real^{n_1 \times n_2} \times \{0, 1\}^{180}$. The loss function is then
\[
\ell_{\mathrm{inp}}(y,y') \coloneqq 
  -\sum_{i_1,i_2=1}^{n_1,n_2}\sum_{i_3 =1}^{180} y_{i_1,i_2,i_3}\log \bigl( y_{i_1,i_2,i_3}' 
\bigr)
\quad\text{ for } y, y' \in \LabelSpace_{\mathrm{inp}}. 
\]
As a hypothesis set $\mathcal{H}_{\mathrm{inp}}$, we use the set of maps resulting from the digital canonical relation described in Subsection~\ref{subsec:ResNetMicroLocalS5} composed with a U-Net. These maps are parametrised via the weights of the discrete two-dimensional convolutional ResNets introduced in Definition~\ref{def:ResNetDiscrete} as well as by the weights of the U-Net. Because of this, we can write $\mathcal{H}_{\mathrm{inp}} \coloneqq (\vecdualvar_{\vartheta, \theta})_{\vartheta \in \Theta_{\mathrm{res}}, \theta \in \Theta_{U}}$.

\subsection{Joint reconstruction and digital wavefront set extraction and inpainting}
\label{sec:jointreconinpaint}

We will now describe a joint approach in which the reconstruction problem as well as the wavefront set reconstruction and inpainting problem are solved simultaneously. This task certainly requires that the hypotheses classes underlying the risk minimisation problem are, at least partially, parametrised by the same parameter set. 

This joint approach leads to the sample space being chosen as $\SampleSpace_{\mathrm{joint}} \coloneqq \Real^{I_s} \times \Real^{I_s} \times \{0,1\}^{I_d} $ for $I_s \subset [m_1] \times [m_2]$ and $I_d \subset [180]$, and the label space being defined as $\LabelSpace_{\mathrm{inp}} \coloneqq \Real^{n_1 \times n_2} \times (\Real^{n_1 \times n_2} \times \{0, 1\}^{180})$. The loss function needs then to be chosen $\ell_{\mathrm{joint}}$ as a weighted average of $\ell_{\mathrm{rec}}$ and $\ell_{\mathrm{inp}}$ in the following sense: For a $\lambda \in (0,1]$ we set, for $g_1,g_2 \in \Real^{n_1\times n_2}, (h_1,h_2) \in \Real^{n_1\times n_2} \times \{0,1\}^{180}$
\begin{align}
\label{eq:jointloss}
	\ell_{\mathrm{joint}}((g_1,h_1), (g_2,h_2)) \coloneqq \lambda\ell_{\mathrm{rec}}(g_1,g_2) + (1-\lambda) \ell_{\mathrm{inp}}(h_1,h_2).
\end{align}
The task-adapted reconstruction architecture involving the joint loss function~\eqref{eq:jointloss} is depicted in Figure~\ref{fig:task-adapt}.

Finally, the hypothesis set for the joint approach is given by $\mathcal{H}_{\mathrm{joint}} \coloneqq ( (\dualvar_\vartheta, \vecdualvar_{\vartheta, \theta}))_{\vartheta \in \Theta_{\mathrm{res}}, \theta \in \Theta_{U}}$.

\section{Numerical results}
\label{sec:numresults}

We now provide numerical results for the joint reconstruction and wavefront set extraction algorithm of Section~\ref{sec:DigitalMicrolocal}. We first present the set-up, followed by a description of our results, and finally an interpretation of the outcome. 

\subsection{Set-up}
\label{subsubsec:setup}

The set-up follows our viewpoint taken in Subsection~\ref{sec:jointreconinpaint}, namely regarding the joint reconstruction problem from an empirical risk minimisation standpoint. Therefore, as a training set we use an artificial data set comprised of piecewise smooth functions, where the singularity curves are given by random splines of degree at most four and the smooth regions are polynomials of degree at most two. We call this data set the \emph{random cartoon-functions data set}. Some examples of functions in this data set were shown in Figure \ref{fig:dataSet}. Notice that for the images in this data set the digital wavefront set is known. We then study two types of partial data, namely  limited-angle data, where we assume that a wedge of $40^\circ$ is missing (a total angle interval of $80^\circ$), and sparse-view, where we only measure $40$ angles. This is followed by evaluating the trained task-adapted architecture (see Figure~\ref{fig:task-adapt}) on the OASIS dataset \cite{lamontagne2019oasis} formed by real brain scans (see \url{https://www.oasis-brains.org/}).

The results are then compared with classical approaches for the limited-angle and sparse-view tomography, including filtered back-projection \cite{krishnan2015microlocal}, Tikhonov \cite{tikhonov1995regul}, and total variation \cite{velikina2007TVCT}. In addition, we have also performed varitional regularisation using as regulariser the L2 and L1-norm of the shearlet coefficients of the reconstruction \cite{gitta2005shearlets, Kutyniok:2016:SFD:2888419.2740960}. We named these methods \emph{Shearlet-L2 sparse} and \emph{Shearlet-L1 sparse}. The deep-learning based benchmarks, namely the Learned Primal-Dual \cite{adler2018lpd} and the Phantom-Net architectures \cite{bubba2018learning}, were similarly trained using the random cartoon functions data set. The code to reproduce our experiments can be found in \url{http://shearlab.math.lmu.de/applications}.

\subsection{Results}

We now present some exemplary results for our algorithm in the case of limited-angle tomography ($80^\circ$ wedge) and sparse-view ($40$ measured angles). The visual results are depicted in Figures~\ref{fig:F10} and \ref{fig:F11}. The comparison with the benchmark approaches from Subsection~\ref{subsubsec:setup} is shown in Table~\ref{table:T3}.

In addition, we show in Figures~\ref{fig:F13} and \ref{fig:F12} reconstruction results of the limited-angle and sparse-view tomography for the benchmarks already mentioned in Subsection~\ref{subsubsec:setup}. In all cases, the joint approach outperforms all classical approaches. Table~\ref{table:T3} presents the performance measure of the benchmarks in terms of the average self similarity (SSIM) and peak signal-to-noise ratio (PSNR).

\begin{table}[H]

\flushleft
\hspace{1.4cm} \textbf{Sparse-view CT:} \hspace{4.7cm} \textbf{Limited-angle CT:}
 
 \

\centering

\begin{tabular}{l r r r r}
\textbf{Method} & \textbf{SSIM} & \textbf{PSNR}\\
\hline

FBP & 0.51 & 19.90 \\
\hline

Tikhonov & 0.73 & 24.77 \\
\hline

TV  & 0.88 & 26.59\\
\hline

Shearlet-L2 sparse  & 0.73 & 24.69\\
\hline

Shearlet-L1 sparse & 0.78 & 25.42\\
\hline

Learned Primal-Dual & 0.89 & 27.55\\
\hline

\emph{Joint approach}  & 0.92 & 28.46\\
\hline
\end{tabular}
\qquad \quad
\begin{tabular}{l r r r r}
\textbf{Method} & \textbf{SSIM} & \textbf{PSNR}\\
\hline

FBP & 0.44 & 14.53 \\
\hline

Tikhonov & 0.73 & 22.62 \\
\hline

TV  & 0.83 & 23.09\\
\hline

Shearlet-L2 sparse  & 0.70 & 22.20\\
\hline

Shearlet-L1 sparse & 0.76 & 22.29\\
\hline

PhantomNet & 0.87 & 25.50\\
\hline

Learned Primal-Dual & 0.86 & 25.55\\
\hline

\emph{Joint approach}  & 0.95 & 29.80\\
\hline
\end{tabular}
\caption{Realistic dataset performance for general benchmarks for sparse-view CT on the left and limited-angle CT on the right.}
\label{table:T3}
\end{table}

\begin{figure}[htb]
\centering
\begin{minipage}[t]{0.3\textwidth}
  \centering
  \vspace{0pt} 
  \includegraphics[width =\linewidth]{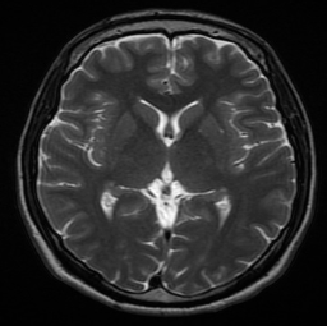}
  \\
  {\small Ground truth}
\end{minipage}
\hfill
\begin{minipage}[t]{0.3\textwidth}
  \vspace{0pt} 
  \includegraphics[width =\linewidth]{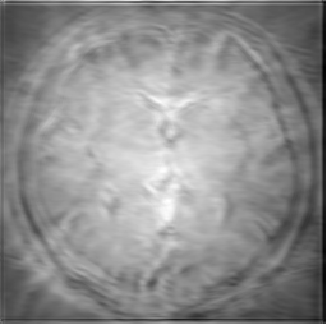}
  \\
  {\small Reconstruction using LPD without the wavefront set prior\\ (PSNR 24.90)}
\end{minipage}
\hfill
\begin{minipage}[t]{0.3\textwidth}
  \vspace{0pt} 
  \includegraphics[width =\linewidth]{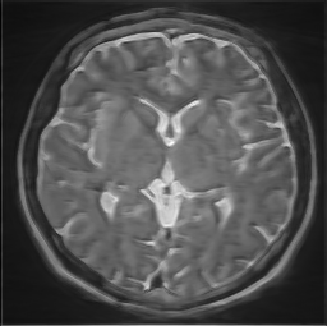}
  \\
  {\small Reconstruction using the joint approach introduced in this work \\ (PSNR 30.20)}
\end{minipage}  
\\[0.5em]
\begin{minipage}[t]{0.3\textwidth}
  \vspace{0pt} 
  \includegraphics[width =\linewidth]{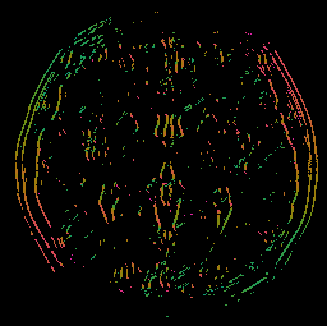}
  \\
  {\small Visible wavefront set of the ground truth extracted by DeNSE}
\end{minipage}
\hfill
\begin{minipage}[t]{0.3\textwidth}
  \vspace{0pt} 
  \includegraphics[width =\linewidth]{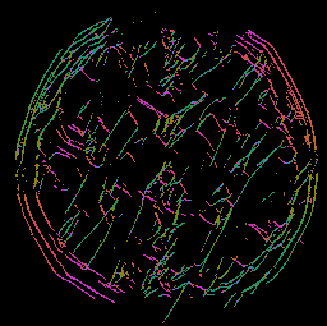}
  \\
  {\small Inpainted wavefront set using U-Net without simultaneous reconstruction}
\end{minipage}
\hfill
\begin{minipage}[t]{0.3\textwidth}
  \vspace{0pt} 
  \includegraphics[width =\linewidth]{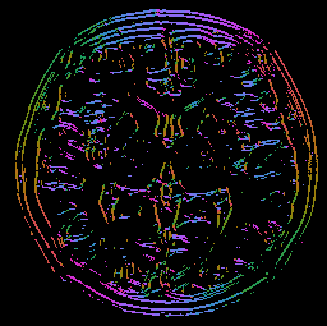}
  \\
  {\small Reconstructed wavefront set using our joint approach}
\end{minipage}  
\caption{Realistic data set results on the joint CT reconstruction and WFset inpainting for limited-angle case, missing wedge = $40^\circ$.}
\label{fig:F10}
\end{figure}

\begin{figure}[htb]
\begin{minipage}[t]{0.3\textwidth}
  \centering
  \vspace{0pt} 
  \includegraphics[width =\linewidth]{img/c8/RealData/brain-gt}
  \\
  {\small Ground truth}
\end{minipage}
\hfill
\begin{minipage}[t]{0.3\textwidth}
  \vspace{0pt} 
  \includegraphics[width =\linewidth]{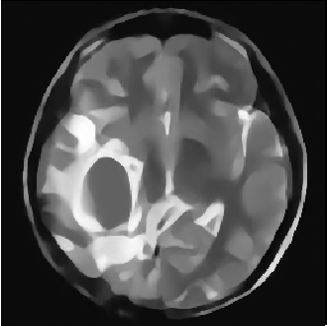}
  \\
  {\small Reconstruction using LPD without the wavefront set prior \\ (PSNR 26.91)}
\end{minipage}
\hfill
\begin{minipage}[t]{0.3\textwidth}
  \vspace{0pt} 
  \includegraphics[width =\linewidth]{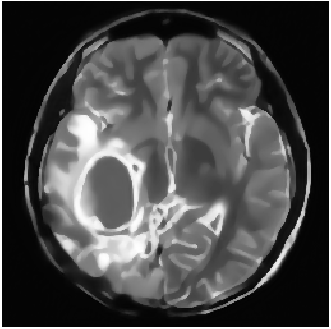}
  \\
  {\small Reconstruction using the joint approach introduced in this work\\ (PSNR 30.22)}
\end{minipage}
\\[0.5em]
\begin{minipage}[t]{0.3\textwidth}
  \vspace{0pt} 
  \includegraphics[width =\linewidth]{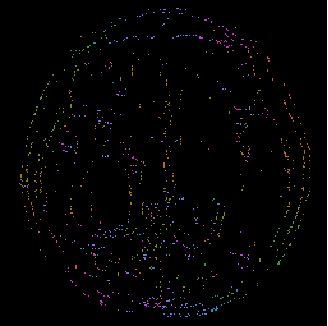}
  \\
  {\small Visible wavefront set of the ground truth extracted by DenSE}
\end{minipage}
\hfill
\begin{minipage}[t]{0.3\textwidth}
  \vspace{0pt} 
  \includegraphics[width =\linewidth]{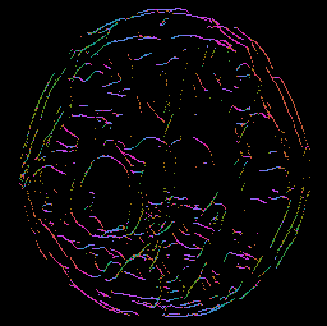}
  \\
  {\small Inpainted wavefront set using U-Net without simultaneous reconstruction}
\end{minipage}
\hfill
\begin{minipage}[t]{0.3\textwidth}
  \vspace{0pt} 
  \includegraphics[width =\linewidth]{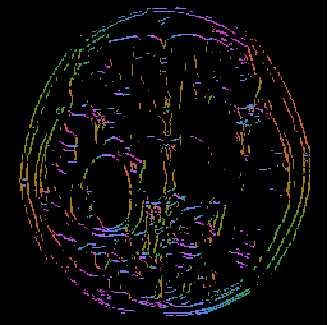}
  \\
  {\small Reconstructed wavefront set using our joint approach}
\end{minipage}
\caption{Realistic dataset results on the joint CT reconstruction and WFset inpainting for sparse-view case, number of angles = 40.}
\label{fig:F11}
\end{figure}

\clearpage

\subsection{Interpretation}

In Table \ref{table:T3}, we see that the joint approach significantly outperforms all competing approaches. The same observation can be made by observing Figures~\ref{fig:F12} and \ref{fig:F13} that demonstrate a strong improvement in the reconstruction accuracy in comparison with the other methods and in particular to the Learned Primal-Dual architecture. In this context, it is noteworthy that the performance gap in this application is much more pronounced than in the sparse-view application. Since the Learned Primal-Dual corresponds to one half of the pipeline of the joint approach, we observe that the additional wavefront set information is especially helpful in the limited-angle case. Since the gaps in the visible wavefront set are significantly smaller in the sparse-view set-up, we expect that, in this set-up, considerably more information on the wavefront set is already included in the observed data, whereas in the limited-angle set-up more prior knowledge needs to be invoked. It appears as if this necessary prior knowledge was very successfully identified in the training phase of the joint approach.

In Figures~\ref{fig:F10} and~\ref{fig:F11} we performed a type of ablation study by looking at the prowess of the inpainting algorithm alone. It transpires from those figures that the inpainting is less successful if it is not coupled with the reconstruction pipeline. This clearly shows that the strength of the joint approach does not solely lie in wavefront set identification and inpainting but in the interplay of inpainting and reconstruction jointly.

We wish to remark that \emph{our algorithm is not trained on real-world images} but instead on artificial phantoms from the random cartoon-functions data set. This shows that the joint prior that incorporates physical information is not based on memorisation. From the standpoint of applications, this could be a very desirable feature, as it prevents overfitting on specific biases in real-world data sets. 

In Figure~\ref{fig:F11}, we see that the learned primal dual, as well as the joint approach, appear to produce reconstructions that exhibit almost no texture-like areas but instead are piecewise smooth. This effect could potentially be an effect of the underlying data set. In this context, it may be worthwhile to expand the training data set by artificial images with more texture like features.


\section*{Acknowledgements}
GK acknowledges support from DFG-SFB/TR 109 Grant C09 and DFG-SPP 2298 Grant KU 1446/32-1. HA acknowledges support from DFG-SFB/TR 109 Grant C09.

\begin{figure}[H]
\begin{minipage}[t]{0.3\textwidth}
\centering
  \vspace{0pt} 
  \includegraphics[width =\linewidth]{img/c8/RealData/brain-gt}
  \\
  {\small Ground truth}
\end{minipage}
\hfill
\begin{minipage}[t]{0.3\textwidth}
\centering 
  \vspace{0pt} 
  \includegraphics[width =\linewidth]{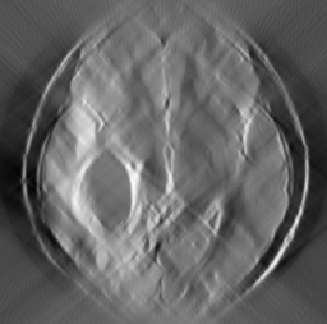}
  \\
  {\small Filtered Back-projection\\ (PSNR 14.40)}
\end{minipage}
\hfill
\begin{minipage}[t]{0.3\textwidth}
\centering 
  \vspace{0pt} 
  \includegraphics[width =\linewidth]{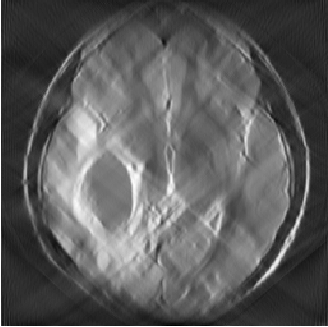}
  \\
  {\small Tikhonov regularisation\\ (PSNR 21.97)}
\end{minipage}
\\[0.5em]
\begin{minipage}[t]{0.3\textwidth}
\centering 
  \vspace{0pt} 
  \includegraphics[width =\linewidth]{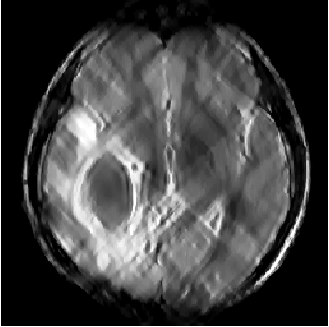}
  \\
  {\small Total variation regularisation\\ (PSNR 22.98)}
\end{minipage}
\hfill
\begin{minipage}[t]{0.3\textwidth}
\centering 
  \vspace{0pt} 
  \includegraphics[width =\linewidth]{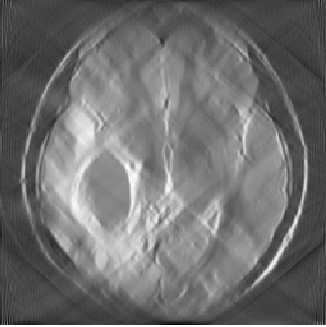}
  \\
  {\small Shearlet-L2 sparse regularisation\\ (PSNR 22.15)}
\end{minipage}
\hfill
\begin{minipage}[t]{0.3\textwidth}
\centering 
  \vspace{0pt} 
  \includegraphics[width =\linewidth]{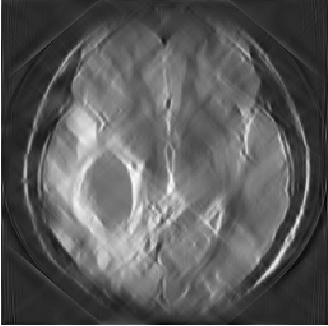}
  \\
  {\small Shearlet-L1 sparse regularisation\\ (PSNR 22.35)}
\end{minipage}
\\[0.5em]
\begin{minipage}[t]{0.3\textwidth}
\centering 
  \vspace{0pt} 
  \includegraphics[width =\linewidth]{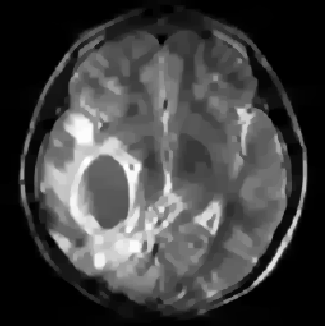}
  \\
  {\small PhantomNet\\ (PSNR 26.02)}
\end{minipage}
\hfill
\begin{minipage}[t]{0.3\textwidth}
\centering 
  \vspace{0pt} 
  \includegraphics[width =\linewidth]{img/c8/RealData/brain-recon-lpd_wedge40}
  \\
  {\small Learned primal-dual\\ (PSNR 25.02)}
\end{minipage}
\hfill
\begin{minipage}[t]{0.3\textwidth}
\centering 
  \vspace{0pt} 
  \includegraphics[width =\linewidth]{img/c8/RealData/brain-recon-joint_wedge40}
  \\
  {\small Joint approach\\ (PSNR 29.95)}
\end{minipage}
\caption{Realistic dataset results for general benchmarks for limited-angle CT, wedge = $40^\circ$.}
\label{fig:F13}
\end{figure}

\begin{figure}[H]
\centering
\begin{minipage}[t]{0.3\textwidth}
\centering 
  \vspace{0pt} 
  \includegraphics[width =\linewidth]{img/c8/RealData/brain-gt}
  \\
  {\small Ground truth}
\end{minipage}
\hfill
\begin{minipage}[t]{0.3\textwidth}
\centering
  \vspace{0pt} 
  \includegraphics[width =\linewidth]{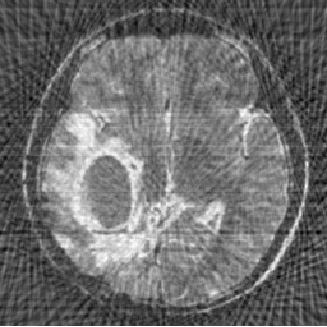}
  \\
  {\small Filtered back-projection\\ (PSNR 18.95)}
\end{minipage}
\hfill
\begin{minipage}[t]{0.3\textwidth}
\centering 
  \vspace{0pt} 
  \includegraphics[width =\linewidth]{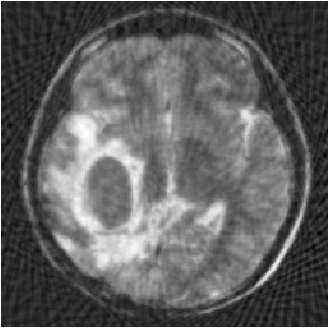}
  \\
  {\small Tikhonov regularisation\\ (PSNR 25.05)}
\end{minipage}
\\[0.5em]  
\begin{minipage}[t]{0.3\textwidth}
\centering 
  \vspace{0pt} 
  \includegraphics[width =\linewidth]{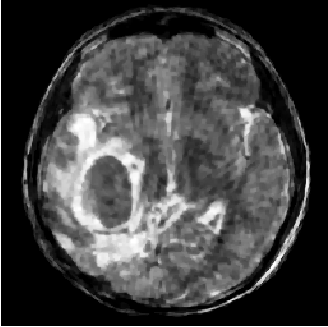}
  \\
  {\small Total variation regularisation\\ (PSNR 26.32)}
\end{minipage}
\hfill
\begin{minipage}[t]{0.3\textwidth}
\centering 
  \vspace{0pt} 
  \includegraphics[width =\linewidth]{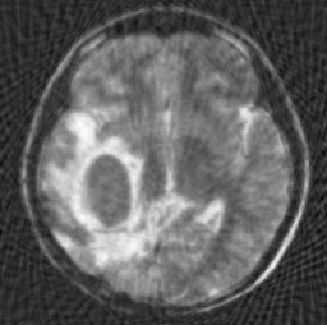}
  \\
  {\small Shearlet-L2 sparse regularisation\\ (PSNR 24.22)}
\end{minipage}
\hfill
\begin{minipage}[t]{0.3\textwidth}
\centering
  \vspace{0pt} 
  \includegraphics[width =\linewidth]{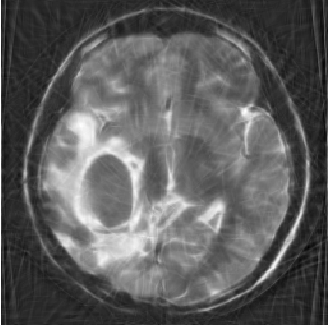}
  \\
  {\small Shearlet-L1 sparse regularisation\\ (PSNR 25.31)}
\end{minipage}
\\[0.5em]
\begin{minipage}[t]{0.3\textwidth}
\centering 
  \vspace{0pt} 
  \includegraphics[width =\linewidth]{img/c8/RealData/brain-recon-lpd_lowd40}
  \\
  {\small Learned Primal-Dual\\ (PSNR 26.91)}
\end{minipage}
\hfill
\begin{minipage}[t]{0.3\textwidth}
\centering 
  \vspace{0pt} 
  \includegraphics[width =\linewidth]{img/c8/RealData/brain-recon-joint_lowd40}
  \\
  {\small Joint approach\\ (PSNR 30.22)}
\end{minipage}
\hfill
\begin{minipage}[t]{0.3\textwidth}
  \vspace{0pt} 
  \,
\end{minipage}
\caption{Realistic data set results for general benchmarks for sparse-view CT using $40$ angles only.}
\label{fig:F12}
\end{figure}

\newpage

\appendix

\section{Notions and results from distribution theory}\label{app:DistrTheory}

Below, we collect some essential notions and results from distribution theory and the analysis of the wavefront set. These were used extensively in Subsection~\ref{subsec:ResNetMicroLocalS2S2}.

\begin{definition}[Essential support \cite{rudin1987anal}]
\label{def:S3D1bis}
Let $\domain \subset \Real^n$ be a domain. The \emph{essential support} of $f\in \SchwartzFunc'(\domain)$ is the defined as the set
\[
\esssupp(\signal) \coloneqq \Real^n \setminus \bigcup_{U \in \SetFam{U}} U,
\]
where $\SetFam{U} \coloneqq \bigl\{ U\subset\domain : U \text{ is open and } f \bigl\vert_{U}=0 \bigr\}$ 
with $f \bigl\vert_{U}$ denoting the restriction of $\signal$ to $U \subset \domain$.
\end{definition}
Note that if $f \in \Cont(\domain)$, then $\esssupp(\signal) = \supp(\signal)$, i.e., the essential support coincides with the usual support. 
We also introduce the notion of positive and negative supports of a distribution.
\begin{definition}[Positive support, \cite{rudin1987anal}]
\label{def:S3D1}
Let $\domain \subset \Real^n$ be a domain. The \emph{positive support} of $f\in \SchwartzFunc'(\domain)$ is the defined as the set
\[
\supp_{+}(\signal) \coloneqq \overline{ \bigcup_{U \in \SetFam{U}} U},
\]
where $\SetFam{U} \coloneqq
    \bigl\{ U \subset \domain :  f(\phi) > 0 \text{ for all } 
       \phi \in \SchwartzFunc(\domain) \setminus \{0\}, 
       \supp \phi \subset U, \phi \geq 0 
     \bigr\}$.
Finally, the \emph{negative support} of $\signal$ is defined as $\supp_{-}(\signal) \coloneqq (\overline{\supp_+(\signal)})^c$.
\end{definition}
From the definition, we see that  
\[ \supp_+(\signal) = \supp\bigl( | f | + f \bigr) = \overline{\{x\in \domain : f(x) > 0 \bigr\}}
    \quad\text{whenever $f \in \Cont^\infty(\domain)$.}
\]    

We now turn our attention to defining the product of two distributions. 
\begin{definition}[Product of distributions, {\cite[Definition~2]{brouder2014smoothwf}}]
\label{def:S1D3}
Let $\domain \subset \Real^n$ be a fixed domain and consider the distributions $u,v\in \SchwartzFunc'(\domain)$. 
We say that $w\in\SchwartzFunc'(\domain)$ is the \emph{product of $u$ and $v$} if an only if, for each $x\in\Real^n$, there exists a test function $\psi\in\SchwartzFunc(\domain)$, with $\psi=1$ on a neighbourhood of $x$, so that  
\begin{equation}
\label{eq:S1E2}
  \widehat{\psi^2w}(\xi) 
    = \bigl( \widehat{\psi u}\ast\widehat{\psi v} \bigr)(\xi) 
    = \int_{\Real^n} \widehat{\psi u}(\xi)\, \widehat{\psi v}(\xi-\eta)\,d\eta
\end{equation}
converges absolutely for each $\xi\in\Real^n$. In addition, under this conditions, we can define $w$ by \eqref{eq:S1E2}.
\end{definition}

\begin{theorem}[Product theorem/H\"ormander condition, {\cite[Section~3.2]{brouder2014smoothwf}}]
\label{thm:S1T5}
Let $\domain \subset \Real^n$ and $u,v\in \SchwartzFunc'(\domain)$. 
Assume there are no points $(x,\xi)\in \WF(u)$ such that $(x,-\xi)\in \WF(v)$.
Then, $w$ defined as in Definition~\ref{def:S1D3} is a well-defined unique distribution that is the product $u v$. 
Moreover, in this case we have
\[
    \WF(uv) \subset S_+\cup S_u \cup S_v,
\]
where 
\begin{align*}
S_+ &\coloneqq \Bigl\{ \bigl( x,(\xi+\omega)/\Vert \xi+\omega \Vert_2 \bigr) : (x,\xi) \in \WF(u)
\text{ and } (x,\omega)\in \WF(v) \Bigr\},
\\ 
S_u &\coloneqq \bigl\{ (x,\xi) : (x,\xi)\in\WF(u) 
  \text{ and } x\in\supp(v) \bigr\},
\\
S_v &\coloneqq \bigl\{ (x,\omega) : (x,\omega)\in\WF(v)
  \text{ and } x\in\supp(u)\bigr\}.
\end{align*}
\end{theorem}

\begin{remark}\label{rem:whatever}
If $u,v \in \Lp^2_{\mathrm{loc}}(\Real)$ and we define $uv(x) = u(x)v(x)$ almost everywhere, then the multiplication of $u$ and $v$ defined in \eqref{eq:S1E2} coincides with $uv$ almost everywhere. This holds even if there exist $(x, \lambda) \in \WF(u)$ such that $(x, -\lambda) \in \WF(v)$. To see this, let $x \in \Real^2$ and $\psi$ be as in Definition~\ref{def:S1D3}. Then 
\[
    \widehat{\psi^2(uv)}(\xi) = \int_{\Real^2} \widehat{\psi u}(\xi)\widehat{\psi v}(\nu-\xi)d\xi
\]
holds in an $L^2$ sense. Moreover, by Plancherel's identity, we have that $\widehat{\psi u}, \widehat{\psi v} \in \Lp^2(\Real^2)$, which yields with the Cauchy-Schwarz identity, that   
\[
    \int_{\Real^2} \bigl\vert \widehat{\psi u}(\xi)\widehat{\psi v}(\nu-\xi) \bigr\vert d\xi 
    \leq \Vert \widehat{\psi u} \Vert_{2} \, \bigl\Vert \widehat{\psi v}(\nu - \cdot) \bigr\Vert_{2} 
    = \Vert \widehat{\psi u} \Vert_{2} \, \Vert \widehat{\psi v}\Vert_{2} 
    < \infty.
\]
This yields absolute convergence in \eqref{eq:S1E2}.
\end{remark}

\begin{definition}\label{def:L2supp}
The $\Lp^2$-support of $\signal \in \SchwartzFunc'(\domain)$ is defined as the largest open set on $\domain$ where $\signal$ is given by an $\Lp^2$-function:
\[  \supp_{L^2}(h) \coloneqq 
    \bigcup \Bigl\{ U \subset \domain \text{ open } : 
      f\bigl\vert_{U} \in \SchwartzFunc'(U) 
    \Bigr\}.
\]
Thus, if $x \in \supp_{L^2}(h)$ then there is an open set $x \in U \subset \domain$ and $f_U \in \Lp^2(U)$ such that $f(\phi) = \int_{U} f_U(x) \phi(x) dx$ for all $\phi \in \SchwartzFunc(U)$.
\end{definition}

\bibliographystyle{abbrv}
\bibliography{references}
\end{document}